\documentclass[12pt]{article}

\usepackage[ruled,vlined]{algorithm2e}
\usepackage[margin=1in]{geometry}
\usepackage{amsfonts, amssymb, amsmath, amsthm, mathtools}
\usepackage{color}
\usepackage{expl3}
\usepackage{subcaption}
\usepackage{multirow}
\usepackage{float}
\usepackage{dirtytalk}
\usepackage[round]{natbib}
\usepackage[most]{tcolorbox}
\usepackage{authblk}
\usepackage{varioref}
\usepackage{hyperref}
\usepackage{cleveref}

\newcommand{\R}{\mathbb R}

\newtheorem{theorem}{Theorem}
\newtheorem{example}{Example}
\newtheorem{lemma}{Lemma}
\newtheorem{proposition}{Proposition}
\newtheorem{remark}{Remark}
\newtheorem{corollary}{Corollary}
\newtheorem{definition}{Definition}

\newtheorem{claim}{Claim}
\newtheorem{fact}{Fact}

\newcommand{\codom}{\mathrm{codom}}

\begin{document}

\title{Universal Representation of Permutation-Invariant Functions on Vectors and Tensors}

\author[]{Puoya Tabaghi \thanks{\url{ptabaghi@ucsd.edu}}}
\author[]{Yusu Wang \thanks{\url{yusuwang@ucsd.edu}}}

\affil[]{Hal{\i}c{\i}o\u{g}lu Data Science Institute, University of California - San Diego }
\date{\today}
\maketitle

\begin{abstract}
A main object of our study is multiset functions --- that is, permutation-invariant functions over inputs of varying sizes. Deep Sets, proposed by \cite{zaheer2017deep}, provides a \emph{universal representation} for continuous multiset functions on scalars via a sum-decomposable model. Restricting the domain of the functions to finite multisets of $D$-dimensional vectors, Deep Sets also provides a \emph{universal approximation} that requires a latent space dimension of $O(N^D)$ --- where $N$ is an upper bound on the size of input multisets. In this paper, we strengthen this result by proving that universal representation is guaranteed for continuous and discontinuous multiset functions though a latent space dimension of $O(N^D)$. We then introduce \emph{identifiable} multisets for which we can uniquely label their elements using an identifier function, namely, finite-precision vectors are identifiable. Using our analysis on identifiable multisets, we prove that a sum-decomposable model for general continuous multiset functions only requires a latent dimension of $2DN$. We further show that both encoder and decoder functions of the model are continuous --- our main contribution to the existing work which lack such a guarantee. Also this provides a significant improvement over the aforementioned $O(N^D)$ bound which was derived for universal representation of continuous and discontinuous multiset functions. We then extend our results and provide special sum-decomposition structures to universally represent permutation-invariant tensor functions on identifiable tensors. These families of sum-decomposition models enables us to design deep network architectures and deploy them on a variety of learning tasks on sequences, images, and graphs. 
\end{abstract}

\section{Introduction} \label{sec:introduction}
There is a wide gamut of machine learning problems aiming at identifying an optimal function on unordered collection of entities, namely, sets and multisets. Set or audience expansion tasks in image tagging, computational advertisement, and astrophysicists~\citep{ntampaka2016dynamical,ravanbakhsh2016estimating}, parsing objects in a scene~\citep{eslami2016attend,kosiorek2018sequential}, population statistics~\citep{poczos2013distribution}, inference on point clouds~\citep{qi2017pointnet,qi2017pointnet_plus}, min-cut and routing on a graph, reinforcement learning~\citep{sunehag2017value}, and modelling interactions between objects in a set~\citep{lee2019set} are examples of such problems. Popular machine learning models are designed for ordered algebraic objects, namely, vectors, matrices, and tensors. To adapt these standard models to operate on multisets, we must enforce various permutation invariance properties~\citep{oliva2013distribution,szabo2016learning,muandet2013domain,muandet2012learning,shawe1993symmetries}. 
To characterize a general class of multiset (or permutation-invariant) functions, several authors have proposed sum-decomposition models~\citep{ravanbakhsh2016deep,zaheer2017deep}. Notably, Deep Sets provides a universal representation for continuous multiset functions on \emph{scalars}. This model is a form of Janossy pooling which is easy to implement and parallelize~\citep{murphy2018janossy}. At its core, it maps elements of the input multiset $X$ individually via $\phi$ and then aggregates them to \emph{uniquely} encode the input multiset, that is, $\Phi(X) = \sum_{x \in X} \phi(x) \in \mathbb{R}^M$ is the unique encoding for $X$ or is an \emph{injective} map. Injectivity is the most important property of the encoder $\Phi$ as it operates an intermediate feature extraction step by uniquely mapping multisets to vectors. Then, to represent a multiset function $f(X)$, we map the resulting feature $\Phi(X)$ to $f(X)$, that is, $f(X) = \rho \circ \Phi(X)$ where $\rho$ is a decoder that belongs to a rich class of unconstrained functions. The existence of a continuous sum-decomposable model --- continuous encoder $\Phi$ and decoder $\rho$ --- is guaranteed only if the dimension of the model’s intermediate features ($M$) is sufficiently large. If we lower this dimension, \cite{wagstaff2022universal} prove that there exists no continuous decoder $\rho$ such that $\rho \circ \Phi$ that can even approximate some multiset functions better than a naive constant baseline. Regarding multiset functions on \emph{vectors}, the best available result is given by~\cite{zaheer2017deep}, which only provides a \emph{universal approximation} for continuous multiset functions through analyzing their finite-order Taylor approximation. As our first contribution, we provide a universal representation, through the sum-decomposable model, for continuous and discontinuous multiset functions on vectors which is a generalization to the existing universal approximation results. It is important to note that all universal representation results are stronger than their of universal approximation counterparts as the former results imply the latter ones.

Beyond permutation-invariant functions on scalars and vectors, SignNet and BasisNet~\citep{lim2022sign} are neural network architectures, among other works \citep{dwivedi2020generalization,dwivedi2020benchmarking,dwivedi2021graph,beaini2021directional,kreuzer2021rethinking,mialon2021graphit,kim2022pure}, that provide sign and orthonormal basis invariances as they are displayed by eigenspaces~\citep{eastment1982cross,rustamov2007laplace,bro2008resolving,ovsjanikov2008global}. Laplacian eigenvectors capture connectivity, clusters, subgraph frequencies, help derive graph positional encodings to generalize Transformers to graphs and improve performance of Graph Neural Networks (GNNs) \citep{dwivedi2020benchmarking,dwivedi2021graph}, and other useful properties of a graph \citep{von2007tutorial,cvetkovic1997eigenspaces}. Under certain conditions, these network structures can universally approximate any continuous function with the desired invariances. Both networks utilize Invariant Graph Networks (IGNs)~\citep{maron2018invariant} to build permutation invariance or equivariance property for functions on matrices. IGN treats graphs (with nodes and edges) as \emph{tensors}. Its architecture involves permutation-invariant and equivariant \emph{linear layers} for tensor input and output data. As the tensor order goes to $O(N^4)$, it achieves the universality for graphs of size $N$ \citep{azizian2020expressive,maron2019universality,keriven2019universal}. 

The type of injective multiset functions as introduced earlier are useful in studying the separation power of Message-Passing Neural Networks and its relation to the Weisfeiler-Leman (WL) graph isomorphism test~\citep{xu2018powerful}. They are also used in showing the equivalence of high-order GNNs to high-order WL tests \citep{morris2019weisfeiler,maron2019provably}, and results related to geometric GNNs and WL tests \citep{hordan2023complete,joshi2023expressive,pozdnyakov2022incompleteness}. \cite{amir2023neural} give a theoretical analysis on the required latent dimension for nonpolynomial encoders --- namely, sigmoid, hyperbolic tangent, sinusoid ---  to arrive at an injective multiset function.

\paragraph{Contributions.} In this paper we mainly focus on the study of multivariate multiset functions, that is, functions on multisets that contain at most $N$ vectors of dimension $D$. In the case of $D=1$, this reduces to multiset functions on scalars. Our main contributions are as follows.
    \begin{enumerate}
        \item We propose extended versions of the sum-decomposition models of multiset functions on vectors~\citep{zaheer2017deep}. Multiset functions encompass permutation-invariant functions since they are invariant to the specific ordering of the input elements. We adopt the term multiset function to emphasize the fact that the number of input elements can vary --- which is not the case for permutation-invariant functions. As our first contribution, in \Cref{sec:warmup}, we present the universal representation for continuous and discontinuous multiset functions --- over $D$-dimensional vectors ---- via a sum-decomposable model; see \Cref{thm:multi_decomposition}. The latent dimension of this model is ${N+D \choose D} - 1$ where $N$ is the upper bound on the size of input multisets. For universal representation of \emph{continuous} multiset functions, we show that both encoder and decoder functions (of the sum-decomposable model) are also continuous; see \Cref{thm:multi_decomposition_continous}. In the case of scalar domain $D= 1$, this latent dimension coincides with the one in ~\citep{wagstaff2019limitations,wagstaff2022universal}, that is, ${N+1 \choose 1} - 1 = N$. \Cref{thm:multi_decomposition,thm:multi_decomposition_continous} are novel contributions to the existing universal approximation results for continuous multiset functions~\citep{zaheer2017deep,maron2019provably,segol2019universal}. Universal approximation results rely on finite-order Taylor approximation of continuous multiset functions. This technique does not work for (1) universal representation and (2) discontinuous multiset functions. As discussed next, we significantly lower this bound for representation continuous multiset functions.
        \item In \Cref{sec:multiset_functions_identifiable}, we put forward the notion of \emph{identifiable} multisets. These are multisets whose distinct elements can be uniquely labeled via a continuous functional, for example, multisets containing finite-precision vectors are identifiable via a linear functional. We then show that on identifiable multisets of $D$-dimensional vectors, the latent dimension of the sum-decomposable representations can be lowered to $2DN$ --- from the original ${N+D \choose D} - 1$. More importantly, through subsequent analysis on identifiable multisets, we show that universal representation of continuous multiset functions, where both encoder and decoder functions are continuous, is possible via latent dimension of $2DN$; see \Cref{thm:continuous_extension_headache}. The techniques used to derive this results are centered at the notion of an identifier function. This is different from the previous lines of work using polynomial and nonpolynomial-based encoders \citep{zaheer2017deep,dym2022low}. While our result in \Cref{thm:multi_decomposition_continous} is suboptimal compared to this new result (\Cref{thm:continuous_extension_headache}), we still include~\Cref{sec:warmup} as it obtains a better result compared to the existing work based on polynomial-based encoders (common in approximation approaches), which is of independent interest. In summary, the main contribution of our results to existing literature are (1) the lowest latent dimension bound, and (2) the continuity guarantee for the decoder function. 
        \item We finally provide universal representation for continuous and discontinuous permutation-invariant tensor functions of an arbitrary order. We obtain a nested sum-decomposable representation on \emph{only} for what we call \emph{identifiable tensors} --- similar to identifiable multisets. Depending on the particular choice of the identifier function, we then provide different bounds on the latent dimensions for the representation. This is similar to an existing decomposition result on permutation-equivariant functions on matrices (tensors of order two)~\citep{fereydounian2022exact}. In contrast, we propose a modified encoder function that (1) provides a reduced latent dimension --- $2DN$ compared to ${D \choose 2} N$, (2) allows for generalization of the sum-decomposition representation to tensors of arbitrary order, (3) is guaranteed to be injective.
    \end{enumerate}
\paragraph{More on related work.}
The most notable work on universal representation of nonlinear multiset function concerns scalar-valued domains~\citep{wagstaff2019limitations,wagstaff2022universal}. Much of the existing result in the literature concerns universal approximations for permutation-invariant and -equivariant functions. Sum-decomposition of multiset functions on multidimensional entities have been solely approached through the universal approximation power of polynomial functions~\citep{zaheer2017deep,segol2019universal}. \cite{wagstaff2022universal} thoroughly investigates the theoretical distinction between universal representation and approximation of multiset functions on \emph{scalars}; but this has remained an open question for multiset function on multivariate elements. Invariant and equivariant \emph{linear} functions have been thoroughly studied in the literature~\citep{maron2018invariant,ravanbakhsh2020universal}. In comparison, our nonlinear model generalizes the permutation-invariant linear layers utilized in IGNs \citep{maron2018invariant}, which, for universal approximation on $N$ points, require $O(N^N)$-sized intermediate tensors~\citep{ravanbakhsh2020universal}. An important class of permutation-compatible (invariant or equivariant) \emph{nonlinear} functions is GNN --- the primary iterative-based models for learning information over graphs. There has been a large body of work aimed at understanding the expressive power of GNNs \cite{maron2019universality,maron2019provably,keriven2019universal,garg2020generalization,azizian2020expressive,bevilacqua2021equivariant}.  To provide insight to the capability of GNNs in representing graph functions, \cite{fereydounian2022exact} introduce an algebraic formulation --- akin to the sum-decomposition model for multiset functions --- to represent permutation-equivariant \emph{nonlinear} functions on matrices in terms of composition of simple encoder and decoder functions. One can connect the notion of permutation-compatible functions (on $2$-tensors) to our proposed algebraic form of permutation-invariant functions on $k$-tensors. Though, by focusing on identifiable tensors, we lowered the latent dimension required for representing $2$-tensors to $O(DN)$ --- from to $O(D^2 N)$ in \citep{fereydounian2022exact} --- and guarantee the \emph{injectivity} of the encoding function.

{\bf Organization.} In \Cref{sec:multiset_functions}, we review the existing sum-decomposition results for multiset functions on scalars. Then, in~\Cref{sec:warmup}, we provide our universal representation results for multivariate multiset functions. In \Cref{sec:multiset_functions_identifiable}, we introduce identifiable multisets and show how they can be used to derived a lowered latent dimension bound for continuous sum-decomposition of continuous multiset functions. Finally, focusing on permutation invariance, we put forth a nested sum-decomposition model to represent invariant functions over $k$-tensors in \Cref{sec:tensors}. Identifiablity for tensors is the main concept necessary to establish the aforementioned decomposition models. We delegate all proofs, supplementary results and discussions to the Appendix.

{\bf Notations.} We denote the nonnegative reals by $\mathbb{R}_{+} = \{ x \in \mathbb{R}: x \geq 0 \}$. For any $N \in \mathbb{N}$, we let $[N] = \{1, \ldots, N \}$. The function $f$ maps elements from its domain to elements in its codomain, that is to say, $f:\mathrm{dom}(f) \rightarrow \codom(f)$ where $\mathrm{codom}(f) = \{ f(x): x \in \mathrm{dom}(f) \}$. Example of domains are $\mathbb{R}$, $\mathbb{R}^D$, $\mathbb{N}$, and $\mathbb{Q}$. We denote the collection of subsets of a domain $\mathbb{D}$ as $2^{\mathbb{D}}$. Let $\mathbb{D}$ be a domain and $f: \mathbb{D} \rightarrow \mathrm{codom}(f)$. We then let $f(\mathbb{D}_1) \stackrel{\mathrm{def}}{=} \{ f(x): x \in \mathbb{D}_1 \} \subseteq \mathrm{codom}(f)$ where $\mathbb{D}_1 \subseteq \mathbb{D}$. A multiset is a pair $(X , m)$ where $X$ is a set of objects and $m$ is a map from $X$ to cardinals (representing the multiplicity of each element in $X$). We simplify this notation by identifying multisets by \say{multiset $X$} or using double curly brackets, namely,  $X = \{\{ 1, 1, 2\}\}$ has three elements but $X = \{ 1, 1 , 2 \} = \{ 1, 2\}$ has two elements. For any domain $\mathbb{D}$ and multiset $X$, $X \subseteq \mathbb{D}$ means that the underlying set for $X$ (repetitive elements removed) is a subset of $\mathbb{D}$, and $|X|$ is the size of the multiset (repetitive elements included). We denote multisets (and sets) with $X$ and tensors (and matrices) with $T$. For $ N \in \mathbb{N}$ and domain $\mathbb{D}$, we let $\mathbb{X}_{\mathbb{D},N} = \{\mbox{multiset} \ X \subseteq  \mathbb{D}: |X| = N \}$, $\mathbb{X}_{\mathbb{D},S} = \{ \mbox{multiset} \ X \subseteq  \mathbb{D}: |X| \in S \}$ where $|\cdot|$ returns the cardinality of its input set (or multiset) and $S \subseteq \mathbb{N}$, namely,  $\mathbb{X}_{\mathbb{D},[N]} = \{ \mbox{multiset} \ X \subseteq  \mathbb{D}: 1 \leq |X| \leq N \}$. 

\section{Review of the Sum-decomposable Model for Multiset Functions on Scalars}\label{sec:multiset_functions}
    
Standard machine learning algorithms operate on data arranged in canonical ways, namely vectors, matrices, and tensors. However, in statistic estimation, set expansion, outlier detection \citep{zaheer2017deep}, and problems involving point clouds or group of atoms form a molecule \citep{wagstaff2022universal}, we often want to learn maps defined on an unordered collection of entities, that is, a set or a multiset. Throughout this paper, we treat functions defined on sets and multisets differently. In this paper, we use a (multi)set function over $\mathbb{D}$ to refer to a function whose domain consists of sub(multi)sets of a domain $\mathbb{D}$. That is, a multiset function assigns a value for every possible submultiset of the domain $\mathbb{D}$. A multiset function $f$ must be: $(1)$ invariant to the ordering of its input elements (permutation invariance), and $(2)$ well-defined on multisets of different sizes.

In general, if one wishes to model a multiset function, it is not immediately clear how to enforce that the given function satisfies condition (1), namely, permutation invariance to the ordering of elements in input multisets. A powerful approach to tackle this problem is to first find a complete representation of multiset functions by specific composition of \emph{unconstrained functions} --- which we refer to as \emph{encoder and decoder functions}. Beside providing a characterization of multiset functions, such decomposition is crucial in the learning setting because, for example, these unconstrained functions can then be modeled (and learned) by neural networks; see for example the popular Deep Sets architecture \citep{zaheer2017deep}.  A specific form of this composition is called \emph{sum-decomposable} representation. The following provides such a result for set functions defined on a countable domain. 

\begin{theorem}[\citealt{zaheer2017deep}]
	Let $f : 2^{\mathbb{D}} \rightarrow \codom(f)$ where $\mathbb{D}$ is a countable domain. Then, 
    \begin{equation}\label{eq:sum_decomposition}
        \forall X \subseteq \mathbb{D}: f(X) = \rho \circ \Phi(X), \ \ \Phi(X) = \sum_{x \in X} \phi(x),
    \end{equation}
    where $\phi:\mathbb{D} \rightarrow \mathrm{codom}(\phi)\subset \mathbb{R}$, $\rho: \mathrm{codom}(\Phi) \rightarrow \codom(f)$, and $ \mathrm{codom}(\Phi) = [0,1] \subset \mathbb{R}$. \label{thm:deep_set}
\end{theorem}

\Cref{thm:deep_set} provides an algebraic construct for universal representation for set functions on countable sets. We use the term universal representation to distinguish it from the weaker universal approximation results in the literature.  This universal representation is obtained via the so-called sum-decomposable representation formally defined as follows: 
\begin{definition} \label{def:continuous_sum_decomposibility}
    A (multi)set function $f$ over $\mathbb{D}$ is sum-decomposable, or it has a sum-decomposable representation, if it can be written as $f(X) = \rho \circ \Phi(X)$ for any (multi)set $X \subseteq \mathbb{D}$, where $\Phi(X) = \sum_{x\in X} \phi(x)$. We refer to $\phi$, $\Phi$ and $\rho$ as the \emph{element-encoder}, \emph{(multi)set-encoder}, and \emph{decoder} functions, respectively. We also may call $\phi$ and $\Phi$ sometimes simply as encoder functions. Furthermore, suppose $\phi: \mathbb{D} \to \codom(\phi) \subseteq \mathbb{R}^M$, then we refer to $\mathbb{R}^M$ (which is the ambient space of $\codom(\phi)$) as the decomposition model's latent space, and say that $f$ is sum-decomposable via $\mathbb{R}^M$. The latent dimension of this sum-decomposition is $M$. A continuous multiset  function $f$ is continuously sum-decomposable if has a sum-decomposable representation where both the encoder and decoder functions, that is, $\phi$ (and thus $\Phi$) and $\rho$ are continuous in the entire ambient space of their respective domains.
\end{definition}

\Cref{thm:deep_set} states that a set function over a countable set is essentially sum-decomposable via $\mathbb{R}$ and the latent dimension is one.  Interestingly, it is shown in \citep{wagstaff2022universal} that set functions on an \emph{uncountable} domain $\mathbb{D}$ do not admit this sum-decomposable representation. Nevertheless, there is an extension of \Cref{thm:deep_set} to finite-sized multisets~\citep{wagstaff2022universal}. 

\begin{theorem}[\citealt{wagstaff2019limitations}]
    Let $N \in \mathbb{N}$, and $f : \mathbb{X}_{\mathbb{D}, [N]  } \rightarrow \R$ be a continuous multiset function where $\mathbb{D} = [0, 1]$. Then, it is continuously sum-decomposable (see \Cref{def:continuous_sum_decomposibility}) via $\R^{N}$ --- that is, the latent space is a subset of $\R^N$ --- and vice versa. \label{thm:continuously_sum_decomposable_1d}
\end{theorem}

Recall that $\mathbb{X}_{\mathbb{D}, [N]}$ is the collection of all multisets over $\mathbb{D}$ of cardinality at most $N$. Since $\mathbb{D}$ in the above theorem is $[0,1] \subset \mathbb{R}$, the result states that a continuous multiset function over scalars is continuously sum-decomposable via $\mathbb{R}^N$ where $N$ is the maximum cardinality of input multisets. The continuity of the decoder $\rho$ comes at the cost of increased latent space dimension; compare universal representation \Cref{thm:deep_set,thm:continuously_sum_decomposable_1d}. This latent dimension is tight in the worst case, that is, there does not exist a sum-decomposition via a latent space with dimension less than $N$~\citep{wagstaff2022universal}. Of course, in practice, for a specific multiset function at hand, there might exist a sum-decomposition with a much lower latent dimension. One might expect that the latent dimension can be reduced in the case of universal approximation. Interestingly, at least in the case of multiset functions over scalars, despite the reasonable intuition, universal approximation is not possible (for all multiset functions) if we lower the latent dimension from $N$ ~\citep{wagstaff2022universal}.  

\section{Warmup: Sum-decomposable Model for Multiset Functions on Vectors} \label{sec:warmup}
\Cref{thm:continuously_sum_decomposable_1d} concerns multiset functions operating on scalar-valued elements (that is, the input is a multiset with elements from $\mathbb{R}$). In practice we are often faced with applications on vector-valued multisets. For example, a multiset of $\leq N$ points in $\mathbb{R}^D$ can be represented as a multiset of cardinality $\leq N$ over $\mathbb{R}^D$; similarly, in the graph learning setting, we may have a set of $N$ nodes in a graph with $D$-dimensional node features. In what follows, we consider multiset functions over vectors in $\mathbb{R}^D$, that is, functions of the form $f: \mathbb{X}_{\mathbb{D}, [N]} \to \mathbb{R}$ where $\mathbb{D} \subset \mathbb{R}^D$. But for simplicity, we first consider functions over multisets of cardinally exactly $N$, that is, $f: \mathbb{X}_{\mathbb{D}, N} \to \mathbb{R}$. Our main result in this section is the following theorem: 
\begin{theorem}\label{thm:multi_decomposition_continous}
    A continuous multivariate multiset function $f:\mathbb{X}_{\mathbb{D},N} \rightarrow \codom(f) (\subseteq \R^n)$, over a multisets of $N$ elements in a compact set $\mathbb{D} \subseteq \mathbb{R}^D$, is continuously sum-decomposable via $\mathbb{R}^{{N+D \choose D}-1}$. That is, encoder $\phi$ is continuous over $\mathbb{D}$, and decoder $\rho$ is continuous over $\mathbb{R}^{{N+D \choose D}-1}$.
\end{theorem}
The above theorem states that a continuous multiset function over multisets of $N$ number of vectors from $\mathbb{D}\subset \mathbb{R}^D$ is continuously sum-decomposable via a latent dimension of ${N+D \choose D}-1$. In the special case of $D=1$, this recovers the previous result for multiset functions over scalars in \Cref{thm:continuously_sum_decomposable_1d}. In \Cref{sec:multiset_functions_identifiable}, we give a stronger result with a much lower latent dimension. We nevertheless include this result because (1) it is obtained via a similar proof technique to \Cref{thm:continuously_sum_decomposable_1d} by using polynomial-based encoders; and (2) this is a novel result that arrives at the same latent dimension as the one reported in~\citep{zaheer2017deep} for the universal approximation of continuous multiset functions. The detailed proofs are given in \Cref{sec:multi_decomposition,sec:multi_decomposition_continous}. We provide a high level description here. 
    
In the remainder of this section, we fix $\mathbb{D} \subset \mathbb{R}^D$ to be a compact subset of $\mathbb{R}^D$. Following  the proof technique in \citep{zaheer2017deep}, to show the existence of a sum-decomposition of $f = \rho \circ \Phi$, we want to construct a multiset encoder $\Phi$ that is \emph{injective} over $\mathbb{X}_{\mathbb{D}, N}$. Once we have an injective encoder $\Phi$, we can then define $\rho = f \circ \Phi^{-1}$ over all admissible inputs, that is, $\mathrm{codom}(\Phi)$. By construction, the encoder $\Phi$ is continuous. The key challenge is to show that $\rho = f \circ \Phi^{-1}$ is not only well-defined but also continuous over the latent space $\mathrm{codom}(\Phi)$.

To construct an injective multiset function $\Phi(X) = \sum_{x \in X} \phi(x)$, we use permutation-invariant polynomials as in  \citep{maron2019provably,segol2019universal}. We express these polynomials as follows: 
\begin{equation}\label{eq:symm_polynomial}
    \forall X \in \mathbb{X}_{\mathbb{R}^D,N} : \ p( X ) = \mathrm{poly}(e_1( X ), \cdots,e_K( X)),
\end{equation}
where $e_k( X )=\sum_{x \in X} \prod_{d=1}^D x_{d}^{k_d}$ is a power-sum multi-symmetric polynomial, $k_1 \ldots k_D$ is the $D$-digit representation of $k \in [K]$ in base $N+1$, $K = {N+D \choose D}-1$, and $\mathrm{poly}$ is a polynomial function~\citep{rydh2007minimal}.  
\begin{remark}
    It is known that one can universally approximate continuous multivariate multiset functions over a compact set with a multiset polynomial in equation \eqref{eq:symm_polynomial}. Since there are $K = {N+D \choose D}-1$ power-sum multi-symmetric polynomial basis$\big(e_{k}(X) \big)_{k \in [K]}$, we can design an encoder $\phi$ to provide a universally approximate sum-decomposable model for multivariate continuous multiset functions via $\mathbb{R}^{{N+D \choose D}-1}$; see Theorem 9 in \citep{zaheer2017deep}.
\end{remark}

In \Cref{sec:multi_decomposition}, we first state \Cref{thm:multi_decomposition} that guarantees a universal representation (not universal approximation as in the above Remark) of \emph{any} multivariate $(D > 1)$ multiset functions --- continuous or discontinuous --- via the sum-decomposable model through $\mathbb{R}^{{N+D \choose D}-1}$. The resulting decoder $\rho$ constructed this way may not be continuous. Nevertheless, this already is a novel contribution to the existing literature. From a technical standpoint, \Cref{thm:multi_decomposition} is valuable as it does not rely on approximating the multiset function $f$ using a finite-order polynomial; but rather, it aims at showing $\Phi$ is injective through analyzing the parameterized roots of a class of multivariate polynomials. Based on \Cref{thm:multi_decomposition}, in \Cref{sec:multi_decomposition_continous}, we show that \emph{if $f$ is a continuous multiset function}, then its decoder $\rho = f \circ \Phi^{-1}$ is  continuous in the ambient space of $\mathrm{codom}(\Phi)$, that is, $\mathbb{R}^{{N+D \choose D}-1}$. The key idea is to prove that (1) $\Phi^{-1}$ is a continuous function on $\mathrm{codom}(\Phi)$ and (2) $\mathrm{codom}(\Phi)$ is a compact subset of $\mathbb{R}^{{N+D \choose D}-1}$.  This completes the proof of \Cref{thm:multi_decomposition_continous}. 
    
We can further generalize the results in \Cref{thm:multi_decomposition_continous,thm:multi_decomposition} to multisets of varying sizes.
\begin{theorem}\label{thm:continuous_decomp_4}
    \Cref{thm:multi_decomposition,thm:multi_decomposition_continous} are valid for multivariate multiset functions of at most $N$ elements from a compact subset $\mathbb{D} \subset \mathbb{R}^D$, that is, $\mathbb{X}_{\mathbb{D},[N]}$.
\end{theorem}

As a direct result of the proof technique of \Cref{thm:continuous_decomp_4} (especially that the construction of the injective multiset encoder $\Phi$ is independent of the multiset function $f$ we try to represent), in \Cref{cor:two_multisets}, we show that for functions on product of \emph{different} multisets of $D$-dimensional vectors, we may use \emph{the same} encoder in its sum-decomposable model.

\begin{proposition}\label{cor:two_multisets}
    A (continuous) multiset function $f:\mathbb{X}_{\mathbb{D},[N_1]} \times \mathbb{X}_{\mathbb{D},[N_2]} \rightarrow \codom(f)$, where $\mathbb{D} $ is compact subset of $\mathbb{R}^{D}$, is (continuously) sum-decomposable via $\mathbb{R}^{{N+D \choose D}-1} \times \mathbb{R}^{{N+D \choose D}-1}$, that is, 
    \[
        \forall X \in \mathbb{X}_{\mathbb{D},[N_1]}, X^{\prime} \in \mathbb{X}_{\mathbb{D},[N_2]}: \ f(X, X^{\prime}) = \rho \big( \sum_{x \in X} \phi(x), \sum_{x^{\prime} \in X^{\prime}} \phi(x^{\prime}) \big),
    \]
    where continuous $\phi: \mathbb{R}^{D} \rightarrow \mathbb{R}^{{N+D \choose D}-1}, N = \max \{N_1, N_2 \}$ and (continuous) $\rho: \mathbb{R}^{{N+D \choose D}-1} \times \mathbb{R}^{{N+D \choose D}-1} \rightarrow \codom(\rho)$, and $\mathrm{codom}(f) \subset \mathrm{codom}(\rho)$.
\end{proposition}

\paragraph{Relation to the results of \cite{fereydounian2022exact}.}
We note that \cite{fereydounian2022exact} propose an encoder $\Phi$ that is injective over particular (multi)sets $\mathbb{X}^{s}_{N,D}$ (not all multisets) of $D$-dimensional vectors. The function $\Phi$ provides unique encodings for these (multi)sets in $\mathrm{codom}(\Phi) \subset \mathbb{R}^{{D \choose 2} N}$ --- where $\mathrm{codom}(\Phi) = \{ \Phi(X) : X \in \mathbb{X}^s_{N,D} \}$ and $N$ is the size of the input (multi)sets. This leads to a sum-decomposition for functions over (multi)sets in $\mathbb{ X}^{s}_{N,D}$. More importantly, for a continuous multiset function, a continuous sum-decomposition $f = \rho \circ \Phi$ is not guaranteed over all multisets; in particular, the continuity of $\rho = f \circ \Phi^{-1}$  is only guaranteed over $\mathrm{codom}(\Phi)$ --- an \emph{open} subset of $\R^{{D \choose 2} N}$. Therefore, it does not guarantee the existence of a continuous extension for $\rho$ to $\mathbb{R}^{{D \choose 2} N}$; see \Cref{sec:ferey} for a detailed discussion.

\section{Sum-decomposable Models on Identifiable Multisets}\label{sec:multiset_functions_identifiable}
Inspired by the theoretical difference between latent space dimensions for sum-decomposition representations of set and multiset functions --- refer to the result in \citep{fereydounian2022exact} --- we aim to reduce the dimension of the latent space. In this section, we introduce a way to achieve that by first restricting the domain to what we call \emph{identifiable multisets}, which we introduce below. We will then show that results over this set can be extended to the case with this restriction removed. 

\begin{definition} \label{def:identifiable_sets}
    Let $l: \mathbb{D} \rightarrow \R$ be a continuous function and $\mathbb{D}$ be a domain. We denote $\mathbb{X}^{l}_{\mathbb{D} , N} = \{ X \in \mathbb{X}_{\mathbb{D},N}:  \forall x,x^{\prime} \in X,   l(x) = l(x^{\prime}) \rightarrow x = x^{\prime} \}$, as the set of multisets of size $N$ that are identifiable via $l$, that is, $l$-identifiable.
\end{definition}
According to \Cref{def:identifiable_sets}, the continuous identifier function $l$ uniquely labels distinct elements of multisets in $\mathbb{X}^{l}_{\mathbb{D}, N}$. In \Cref{thm:perm_w_distinct_labels} and \Cref{prop:perm_w_distinct_labels} we provide improved bounds on latent dimensions given in \Cref{thm:multi_decomposition,thm:continuous_decomp_4} --- by restricting the domain of multiset functions to $l$-identifiable multisets.

\begin{theorem}\label{thm:perm_w_distinct_labels}
    Let $f:\mathbb{X}_{\mathbb{R}^D,N} \rightarrow \codom(f)$ be a multiset function and $\ell: \mathbb{R}^D \rightarrow \mathrm{codom}(l) \subseteq \mathbb{R}$ be  continuous. Then, there is a continuous function $\phi: \mathbb{R}^{D} \rightarrow \mathrm{codom}(\phi) \subset \mathbb{C}^{D \times N}$ such that 
    \[
        \forall X \in \mathbb{X}^{l}_{\mathbb{R}^D , N}: f(  X )  = \rho \big( \sum_{x \in X} \phi (x) \big) = \rho \circ \Phi(X),
    \]
    where $\rho: \Phi(\mathbb{X}^{l}_{\mathbb{R}^D, N}) \rightarrow \codom(f)$ and $\Phi(\mathbb{X}^{l}_{\mathbb{R}^D, N}) \stackrel{\mathrm{def}}{=} \{ \Phi(X): X \in \mathbb{X}^{l}_{\mathbb{R}^D, N} \}$.
\end{theorem}
\begin{proposition}\label{prop:perm_w_distinct_labels}
    \Cref{thm:perm_w_distinct_labels} is valid for multivariate multiset functions of at most $N$ elements from a compact subset of $\mathbb{R}^D$.
\end{proposition}

\begin{remark}
    \Cref{thm:perm_w_distinct_labels} asserts that sum-decomposition of \emph{arbitrary} (continuous or discontinuous) multiset functions is possible via latent dimension $O(ND)$ on inputs that are identifiable via a continuous identifier $l: \mathbb{R}^D \rightarrow \mathrm{codom}(l) \subseteq \mathbb{R}$. In comparison, the universal representation results in \Cref{thm:multi_decomposition_continous,thm:continuous_decomp_4} require the latent space dimension of $O(N^D)$, which even for a small number of features, becomes obsolete in practice. Furthermore, the bound in \Cref{thm:perm_w_distinct_labels} is an improvement over $O(ND^2)$ proposed in \citep{fereydounian2022exact}. Additionally, we propose a concrete characterization of the input domain in~\Cref{def:identifiable_sets} --- which works for \emph{any} continuous function $l$ which can be designed for the specific application. Since the set of identifiable multisets $\mathbb{X}^l_{ \mathbb{D}, N}$ (where $\mathbb{D}$ is a compact subset of $\mathbb{R}^D$) does not form a compact set, there is no guarantee that $\rho: \mathrm{codom}(\mathbb{X}^l_{ \mathbb{D}, N}) \rightarrow \mathrm{codom}(f)$ has a continuous extension to $\mathbb{C}^{D \times N}$ --- that is, if we use the multiset encoding function $\Phi$ (introduced in the proofs), for some multiset functions $f$ there may not exits a continuous $\rho: \mathbb{C}^{D \times N} \rightarrow \mathrm{codom}(\rho)$ that enables the sum-decomposition. However, we address this issue in \Cref{subsec:continuousdecoder}. We finally note that our specific multiset encoder $\Phi$ maps multisets to complex-valued matrices in $\mathbb{C}^{D \times N}$. Without causing any technical issues, this latent space can be viewed as $\mathbb{R}^{2D \times N}$.
\end{remark}

\begin{remark} \label{remark:others}
    The multiset encoding function $\Phi$ in \Cref{prop:perm_w_distinct_labels} is akin to separating invariants introduced in \citep{dym2022low}, that is, the quantity $\Phi(X)$ is invariant with respect to permutations --- as group actions. The subtle difference is that, multiset function are permutation-invariant but the converse is not true; since multiset functions may be allowed to have varying-sized inputs. Using separating invariants, \cite{dym2022low} claim that for randomized invariants of dimension $2DN+1$ (compare to ours which is $2DN$) almost all matrices in $\mathbb{R}^{D \times N}$ are identified up to the permutation of their columns. This results is based on applying linear projections on multidimensional elements to obtain scalars and then using a continuous separating (injective) map on them. Then they prove that the measure of matrices that can not be identified via the permutation-invariant encoding is zero. As a result, the sum-decomposition does \emph{not} hold for all matrices (akin to multisets in our paper) and there is no guarantee for the existence of continuous decoder $\rho$ (over the ambient space) for representing a continuous permutation-invariant function. On the other hand, \cite{amir2023neural} propose using a nonpolynomial element-encoder, that is, $\phi$ in our notation, to construct an injective multiset function $\Phi$. They arrive at a latent dimension of $2N(D+1)+1$. However, their construction of $\phi$ requires random selection of parameters and the injectivity only holds in the \emph{almost surely} sense. Therefore, it may not work for some parameters.
\end{remark}

\subsection{Towards a Continuous Decoder}\label{subsec:continuousdecoder} 

In \Cref{thm:perm_w_distinct_labels}, we prove how our notion of $\ell$-identifiable multisets admits a reduced latent dimension for sum-decomposition representation of multiset functions. The state-of-the-art approach that allow such reduced-dimensional representations rely on probabilistic arguments, that is, excluding multisets of measure zero from all valid multisets; see \Cref{remark:others}. These approaches do not yet lead to a continuous sum-decomposition (in particular, continuous decoder function $\rho$). In what follows, we use the $\ell$-identifable multisets, focus on allowing the representation on a \emph{dense subset} of multisets as \Cref{prop:rational_identifier} and \Cref{lem:phi_dense} below suggest, and ultimately find a continuous sum-decomposition as in \Cref{thm:continuous_extension_headache}. Proofs of all these results can be found in \Cref{sec:rational_identifier,sec:phi_dense,sec:continuous_extension_headache}.

\begin{proposition}\label{prop:rational_identifier}
    Let $\mathbb{X}_{\mathbb{Q}^D,N}$ be the set of all multisets of $N$ vectors from $\mathbb{Q}^D$ where $\mathbb{Q}$ denotes the set of rational numbers. Then, $\mathbb{X}_{\mathbb{Q}^D,N}$ is an $l$-identifiable subset of $\mathbb{X}_{\mathbb{R}^D,N}$.
\end{proposition}
\begin{lemma} \label{lem:phi_dense}
    Let $\mathbb{D} \subseteq \mathbb{R}^D$ be a compact set with continuous nonempty interior, $Q(\mathbb{D}) = \mathbb{D} \cap \mathbb{Q}^D$ be the set of all vectors with rational elements in $\mathbb{D}$. Then, $\mathbb{X}_{Q(\mathbb{D}),N}$ is a dense subset of $\mathbb{X}_{\mathbb{D},N}$. Similarly, $\Phi(\mathbb{X}_{Q(\mathbb{D}),N})$ is a dense subset of $\Phi(\mathbb{X}_{\mathbb{D},N})$ where $\Phi$ is the multiset encoder in \Cref{thm:perm_w_distinct_labels}.
\end{lemma}

\begin{remark}\label{rem:netlist}
To give an example for the utility of \Cref{thm:perm_w_distinct_labels}, consider circuits design applications where we want to learn a variety of electronic design tasks, for example, routed wire length prediction~\citep{xie2021net2}, circuit partitioning \citep{lu2020tp}, logic synthesis \citep{zhu2020exploring} and placement optimization~\citep{li2020customized}. We can represent circuit as geometric graphs with nodes placed on integer-valued vectors coordinates with multidimensional features, that is, properties of each circuit elements. As a result of \Cref{prop:rational_identifier}, we can uniquely identify each node with a continuous identifier. Since they are an important class of $\ell$-identifiable multisets, in \Cref{cor:sum_decomp_q}, we specialize \Cref{thm:perm_w_distinct_labels} to rational-valued multisets.
\end{remark}

\begin{corollary} \label{cor:sum_decomp_q}
    Let $f:\mathbb{X}_{\mathbb{R}^D,N} \rightarrow \codom(f)$ be a multiset function. Then, there is a continuous function $\phi: \mathbb{R}^{D} \rightarrow \mathrm{codom}(\phi) \subset \mathbb{C}^{D \times N}$ such that 
    \[
        \forall X \in \mathbb{X}_{\mathbb{Q}^D,N}: \ f(  X )  = \rho \big( \sum_{x \in X} \phi (x) \big) = \rho \circ \Phi(X),
    \]
    and $\rho: \Phi(\mathbb{X}_{\mathbb{Q}^D,N}) \rightarrow \codom(f)$.
\end{corollary}

\Cref{cor:sum_decomp_q} states that the sum-decomposable model is valid --- via latent dimension of $2 D N$ --- on a dense subset of multisets in $\mathbb{X}_{\mathbb{R}^D,N}$; see \Cref{lem:phi_dense}. The main drawbacks of this representation are as follows: (1) the measure of valid multisets $\mathbb{X}_{\mathbb{Q}^D,N}$ is zero and (2) there is no guarantee on the existence of a continuous extension of $\rho$ to $\mathbb{C}^{D \times N}$. It is important to note that we choose to focus on $\mathbb{X}_{\mathbb{Q}^D,N}$ despite the fact that it has a measure zero. We argue that one should not focus on the measure of valid multisets $\mathbb{X}_{\mathbb{Q}^D,N}$; but rather take advantage of the fact that valid multisets form a dense subset of all multisets, that is, $\mathbb{X}_{\mathbb{R}^D,N}$.  In \Cref{thm:continuous_extension_headache}, we leverage this fact and resolve both aforementioned issues by focusing on the sum-decomposable representation of \emph{continuous} multiset function.

\begin{theorem}\label{thm:continuous_extension_headache}
    Consider a compact set $\mathbb{D} \subset \mathbb{R}^D$ with nonempty interior. Let $f:\mathbb{X}_{\mathbb{D},N} \rightarrow \codom(f)$ be a \emph{continuous} multiset function and $\Phi: \mathbb{X}_{\mathbb{D},N} \rightarrow \codom(\Phi)$ be the function in \Cref{thm:perm_w_distinct_labels}. Then, there exists a \emph{continuous} function $\rho: \mathbb{C}^{D \times N} \rightarrow \mathrm{codom}(\rho) \subseteq f(\mathbb{X}_{\mathbb{D},N})$ such that 
    \[
        \forall X \in \mathbb{X}_{\mathbb{D},N}: f(  X )  = \rho \circ \Phi(X).
    \]
\end{theorem}

The major contribution of \Cref{thm:continuous_extension_headache} is the continuity of $\rho$ over the whole latent space. The detailed proof of this key theorem is in \Cref{sec:continuous_extension_headache}. At the high level, we begin with the result in \Cref{cor:sum_decomp_q}. There, we claim that there exits decoding function $\rho: \Phi( \mathbb{X}_{Q(\mathbb{D}),N}) \rightarrow \mathrm{codom}(\rho)$ such that the stated decomposition remains valid on rational-valued vectors in $ \mathbb{D} \subset \mathbb{C}^{D \times N}$. This result does not guarantee the continuity of $\rho$ in $\mathbb{C}^{D \times N}$. However, we leverage the facts that (1) $f$ is a continuous multiset function and (2) $ \Phi( \mathbb{X}_{Q(\mathbb{D}),N})$ is a \emph{dense} (noncompact) subset of $ \Phi( \mathbb{X}_{\mathbb{D},N})$ and prove that $\rho$ has a \emph{continuous extension} to $ \Phi( \mathbb{X}_{\mathbb{D},N})$ ---- a compact subset of $\mathbb{C}^{D \times N}$ --- and therefore has a continuous extension to $\mathbb{C}^{D \times N}$. The continuity guarantee of the decoder function $\rho$ is the major contribution of \Cref{thm:continuous_extension_headache} over existing results in \citep{dym2022low} and \citep{fereydounian2022exact}. 

\section{Permutation-Invariant Tensor Functions}\label{sec:tensors}
Data with underlying a hypergraph structure --- that is, nodes connected with weighted (hyper)edges --- are ubiquitous in many applications~\citep{chen2019reinforcement,ma2018constrained,wang2019kgat,yang2019auto}. Inspired by such data, we study functions defined on \emph{tensors} and adopt graph-theoretic notions to describe relevant concepts. The tensor setting is also used for the higher order graph neural network called IGN (Invariant graph network)~\citep{maron2018invariant}. 
\begin{definition}\label{def:tensor}
    Let $N, K \in \mathbb{N}$. We denote $\mathbb{T}_{N,K}$ as the set of $K$-th order $D$-dimensional tensors on $N$ entities, that is, $\mathbb{T}_{N,K} =  \mathbb{R}^{N^K \times D} $. 
\end{definition}
We can use tensors to represent (1) node features, (2) graph adjacency matrix (second order tensor), and (3) hypergraph hyperedges with multidimensional features. In \Cref{def:tensor_perm}, we introduce a tensor notation for permuting node entities. 
\begin{definition} \label{def:tensor_perm}
    Let $N \in \mathbb{N}$, $\Pi(N)$ be the set of permutations over $[N]$, and $\pi \in \Pi(N)$. Then, we let
    \begin{align*}
        T, T^{\prime} \in  \mathbb{T}_{N,K}: T^{\prime} &= \pi(T)  \Longleftrightarrow T^{\prime}_{n_1\ldots n_K} = T_{\pi(n_1)\ldots \pi(n_K)}~~~ \text{for all}~~ n_1,\ldots, n_K \in [N]. 
    \end{align*}
      Tensors $T, T^{\prime} \in  \mathbb{T}_{N,K}$ are congruent, denoted by $T \equiv T^{\prime}$, if there is $\pi \in \Pi(N)$ such that $T^{\prime} = \pi(T)$.
\end{definition}
Akin to multiset functions, tensor functions must exhibit the same permutation invariance property. Adopting the notation in \Cref{def:tensor_perm}, a permutation-invariant tensor function $f: \mathbb{T}_{N,K} \rightarrow \mathrm{codom}(f)$ is such that $f(T) = f(\pi(T))$ for all $T \in \mathbb{T}_{N,K} $ and permutation operator $\pi: [N] \rightarrow [N]$. This is a specific form of $G$-invariant functions~\citep{maron2019universality} where $G$ is the permutation group. It also is an extension of permutation-compatible functions formalized for $2$-tensors, that is, input graphs with node features and an adjacency second-order tensor (matrix)~\citep{fereydounian2022exact}. 

In what follows, we propose a sum-decomposable model to universally represent permutation-invariant tensor (of arbitrary order) functions. Our algebraic approach relies on identifying each node with a unique label. This is applicable when we have tensors accompanied with distinct node features or hypergraph structures that admit the unique labelling. Given \emph{any} identifier, in \Cref{def:tensor_idenfiable}, we formalize the set of all tensors that admit the required unique labelling.

\begin{definition} \label{def:tensor_idenfiable}
    Let $l: \mathbb{T}_{N,K} \rightarrow \R^{N \times M}$ be an identifier --- with $M$-dimensional labels --- such that 
    \[
        \forall T \in T \in \mathbb{T}_{N,M}: \  l(\pi(T))  = \pi l(T)
    \]
    We denote the set of tensors that are identifiable via $l$, that is, $l$-identifiable, as $\mathbb{T}^{l}_{N,K} \subset  \mathbb{T}_{N,K}$ such that $\forall T  \in \mathbb{T}^{l}_{N,K}$ the multiset $\{ \{ e_n^{\top}l(T) \in \mathbb{R}^M: n \in [N]\} \}$ consists of distinct elements where $e_n$ is the $n$-th standard basis of $\mathbb{R}^N$ and $n \in [N]$.
\end{definition}
In the first step of our approach,  given a tensor and an identifier, we first construct a \emph{set} that remains invariant with respect to the permutation of the node entities.
\begin{definition} \label{def:set_s_x_a}
Let $K,N \in \mathbb{N}$. For any $l$-identifiable tensor $T \in \mathbb{T}^{l}_{N,K}$, let $\alpha^K_{n_1 \ldots n_K}(T) = T_{n_1 \ldots n_K} \in \mathbb{R}^D$ for all $n_1, \ldots, n_K \in [N]$. Then, we define recursively that: 
\[
    \forall k \in K~\text{down to}~1, n_1, \ldots, n_{k-1} \in [N]: \ \alpha^{k-1}_{n_1 \ldots n_{k-1}}(T) = \{ \big( e_{n_{k}}^{\top}l(T), \alpha^{k}_{n_1 \ldots n_k}(T) \big): n_{k} \in [N]  \} .  
\]
We define the set $S(T) = \{ \big( e_{n_1}^{\top}l(T), \alpha^1_{n_1}(T) \big) : n_1 \in [N] \}$.
\end{definition}
\begin{proposition}\label{prop:hypersets}
    Let $K,N \in \mathbb{N}$ and $T, T^{\prime}  \in \mathbb{T}^{l}_{N,K}$. Then, we have $S(T) = S(T^{\prime})$ if and only if $T^{\prime} = \pi(T)$ for a permutation $\pi \in \Pi(N)$, that is, $T \equiv T^{\prime}$.
\end{proposition}
\Cref{prop:hypersets} establishes a bijection between identifiable tensors $\mathbb{T}^{l}_{N,K}$ --- upto a permutation factor--- and sets in $S(\mathbb{T}^{l}_{N,K})$. In \Cref{thm:tensor}, we give an algebraic characterization of (nonlinear) permutation-invariant tensor functions with distinct node features, that is, the sum-decompoable model is valid \emph{only} on identifiable tensors.
\begin{theorem}\label{thm:tensor}
    Let $K,N \in \mathbb{N}$. Let $f: \mathbb{T}_{N,K} \rightarrow \mathrm{codom}(f)$ be a permutation-invariant tensor function. Then we have 
    \[
        \forall T \in \mathbb{T}^{l}_{N,K} : f(T) = \rho \Big( \sum_{n_1 \in [N]} \phi_{1} (e_{n_1}^{\top}l(T), \beta^1_{n_1}(T) ) \Big)
    \]
    where $l: \mathbb{T}_{N,K} \rightarrow \mathrm{codom}(l)$ is an identifier function, $\beta^{K}_{n_1 \ldots n_K}(T) = T_{n_1 \ldots n_{K}} \in \mathbb{R}^D$ for all $n_1, \ldots, n_K \in [N]$, and
    \[
        \forall k\in [K], n_1,\ldots, n_{k-1} \in [N]: \ \beta^{k-1}_{n_1 \ldots n_{k-1}}(T) = \sum_{n_{k} \in [N]} \phi_{k}(e_{n_{k}}^{\top}l(T), \beta^{k}_{n_1 \ldots n_k}(T) ),
    \]
    where $\phi_k $ is continuous over its compact domain and its codomain resides in $\mathbb{R}^{D_k}$ ($k \in [K]$), and
    \begin{enumerate}
        \item $D_k = 2 ( M + D_{k+1}) N$ if  $\mathrm{codom}(l) \subset \mathbb{Q}^{N \times M}$
        \item  $D_k = {N+ D_{k+1} \choose N  } -1 $ if $\mathrm{codom}(l) \subset \mathbb{R}^{N \times M}$
    \end{enumerate}
    for all $k \in [K-1]$ and $D_K = D$. The function $\rho$ is defined on $\mathbb{D} \subset \mathbb{R}^{D_1}$ where
    \[
        \mathbb{D} = \{ \sum_{n_1 \in [N]} \phi_{1} (e_{n_1}^{\top}l(T), \beta^1_{n_1}(T) ) : T \in \mathbb{T}^{l}_{N,K} \},
    \]
    and it is not guaranteed to have a continuous extension to $\mathbb{R}^{D_1}$.
\end{theorem}

\section{Conclusion}
In this work, we provide several contributions regarding the universal representation theory of multiset functions and permutation-invariant tensor functions. We show that there exists a universal sum-decomposition model for multivariate multiset functions and provide the best available bound on the dimension of encoded multiset features. Our extensive analyses rely on the novel notion of $\ell$-identifiable multisets --- which allows us to uniquely label distinct elements of multisets. Our proposed decomposable model for permutation-invariant tensor functions generalizes the existing models for linear permutation invariant tensor functions used as the layers of IGNs. It is important to note that our universal representation (via sum-decomposables) is stronger than the concept of universal approximation. All these results lead to universal approximation results of multiset (or tensor) functions by sum-decomposables --- which suggest natural architectures for neural networks, similar to DeepSets. 

\bibliography{references}

\appendix

\newpage
\section{\Cref{thm:multi_decomposition} and Its Proof}\label{sec:multi_decomposition}   
    \begin{theorem}\label{thm:multi_decomposition}
    {\bf \emph{Any}} multivariate multiset function $f:\mathbb{X}_{\mathbb{R}^D,N} \rightarrow \codom(f)$ --- over a multisets of $N$ elements in $\mathbb{R}^D$ --- is sum-decomposable via $\mathbb{R}^{{N+D \choose D}-1}$, that is, 
    \[
        \forall X \in \mathbb{X}_{\mathbb{R}^D,N}: \ f(X) = \rho \circ \Phi( X ), \ \mbox{where} \ \Phi(X) \stackrel{\mathrm{def}}{=} \sum_{x \in X} \phi(x),
    \]
    where $\phi: \mathbb{R}^{D} \rightarrow \mathrm{codom}(\phi)\subseteq \mathbb{R}^{{N+D \choose D}-1}$ is a continuous function and $\rho: \mathrm{codom}(\Phi) \rightarrow \codom(f)$.
    \end{theorem}

Note that compared to Theorem \ref{thm:multi_decomposition_continous}, the function $f$ in Theorem \ref{thm:multi_decomposition} is not necessarily continuous, and the decoder function $\rho$ is not necessarily continuous as well. 

    \subsection{Proof}
    Let $N, D \in \mathbb{N}$. We want to prove that for any multivariate multiset function $f:\mathbb{X}_{\mathbb{R}^D,N} \rightarrow \codom(f)$, there exists a sum-decomposition via $\mathbb{R}^{{N+D \choose D}-1}$.

    {\bf Trivial case of $\bf{N=1}$.} We define functions $\phi$ and $\rho$ as follows:
    \[
            \forall x \in \mathbb{R}^D: \phi(x) = x, \ \mbox{and} \ \rho(x) = f(\{\{ x \}\}),
    \]
    where $\mathrm{dom}(\phi) = \mathbb{R}^D$. Since $\Phi( \{\{ x\} \}) = \phi(x)$, $\mathrm{codom}(\rho) = \mathrm{codom}(\Phi) = \mathrm{codom}(\phi) = \mathbb{R}^D =\mathbb{R}^{{1+D \choose D} - 1}$, $\mathrm{codom}(\rho) = \mathrm{codom}(f)$, and $f(\{\{x \}\}) = \rho \circ \Phi(\{ \{ x \} \})$, we arrive at the theorem's statement for $N=1$. 
    \begin{remark}
        In our notation, depending on the context, $x_n$ can mean either (1) the $n$-th coordinate (element) of vector $x$ (say in $\mathbb{R}^D$) or (2) a vector indexed by $n$, for example, $x_1, \ldots, x_N \in \mathbb{R}^D$. In the latter case, we do emphasize the domain of the vector a priori, that is, $x_n \in \mathbb{R}^D$.
    \end{remark}
    {\bf General case of $\bf{N \geq 2}$.} We break down our approach into two steps: 
    \begin{enumerate}
        \item We show that there exists a function $\phi:\mathbb{R}^{D} \rightarrow \mathrm{codom}(\phi) \subseteq \mathbb{R}^{{N+D \choose D}-1}$ such that $\Phi(X) = \sum_{x \in X} \phi(x)$ is an injective multiset function, that is, $\Phi^{-1}$ is well-defined on $\codom(\Phi)$.
        \item Let $\rho = f \circ \Phi^{-1}$. This immediately proves $f = \rho \circ \Phi(X) =  \rho \big( \sum_{x \in X} \phi(x) \big)$.
    \end{enumerate}
    This is an extension to the existing univariate result (that is, $D = 1$); refer Theorem 2 in \citep{zaheer2017deep}. In the one-dimensional case, \cite{zaheer2017deep} prove that $\Phi$ is an invertible function by showing that, given $\Phi(X)$, one can construct a univariate polynomial $p(t; \Phi(X))$ whose roots are $X$, that is, $\Phi^{-1} \circ \Phi(X) = \mathrm{roots} \circ p(t; \Phi(X)) =  X$ where $\mathrm{roots}$ returns the multiset of roots of a polynomial equation. Moreover, the appropriate choice for the basis function $\phi$ --- which makes this analysis tractable --- gives a bound for the latent dimension, that is, dimension of the ambient vector space containing $\mathrm{codom}(\Phi)$.
    
    In our approach, we arrive at the appropriate choice for $\phi$ constructing a {\bf multivariate} polynomial whose parameterized roots are {\bf related} to $X$. In what follows, we $(1)$ introduce the basis function $\phi$, $(2)$ construct an appropriate multivariate polynomial $p(t;z , \Phi(X))$ --- parameterized by both $t \in \mathbb{R}$ and $z \in \mathbb{R}^D$ --- and $(3)$ extract $X$ from its {\bf parameterized} roots. In step $(3)$, we introduce novel techniques for analyzing parameterized multisets --- akin to computing directional derivatives for multivariate functions. We summarize these steps in \Cref{fig:proof1}.

    \begin{figure}[t!]
        \includegraphics[width=\linewidth]{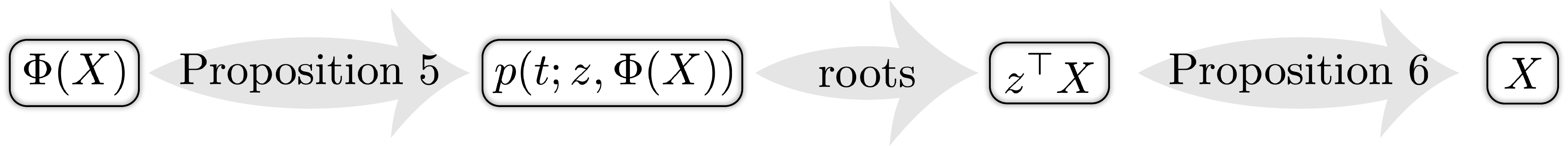}
        \caption{Proof sketch for the injectivity of $\Phi$.}
        \label{fig:proof1}
    \end{figure}
    
    The following definition introduces several frequently used functions in this proof.
    \begin{definition}\label{def:gap_sort_diam_unique}
        For any multiset of real scalars $X = \{ \{ x_n \in \R : n \in [N] \} \}$ where $N \geq 2$, we let
        \begin{align*}
            &\mathrm{gap}(X) = \min_{\substack{n, n^{\prime} \in [N] \\ x_n \neq x_n^{\prime} }} |x_n - x_{n^\prime}| , &&\mathrm{diam}(X) = \max_{\substack{n, n^{\prime} \in [N] \\ n \neq n^{\prime} }} |x_n - x_{n^\prime}|, \\
            &\mathrm{unique}(X) = \{x_n : n \in [N]  \},  &&\mathrm{sort} ( X) = \big( x_{\pi(n)} \big)_{n \in [N]} \in \R^N,
        \end{align*}
        where $\pi: [N] \rightarrow [N]$ is a permutation operator such that $x_{\pi(1)} \geq x_{\pi(2)} \geq \cdots \geq x_{\pi(N)}$. 
    \end{definition}
    \begin{remark}
        If $x_n, x_{n^{\prime}} \in X$ where $x_n = x_{n^{\prime}}$ for distinct $n, n^{\prime} \in [N]$, then the permutation operator $\pi$ in~\Cref{def:gap_sort_diam_unique} is not unique; but any such permutation $\pi$ results in the same sorted vector $( x_{\pi(n)})_{n \in [N]}$. Hence, $\mathrm{sort} (X)$ is well-defined for any multiset of real-valued scalars $X$.
    \end{remark}
    \begin{remark}
        Let $X$ be a multiset of real scalars. Then, $\mathrm{gap}(X)$ is well-defined only if the cardinality of $\mathrm{unique}(X) $ is strictly greater than one, that is, $|\mathrm{unique}(X)| > 1$. 
    \end{remark}
    We consider a class of multivariate polynomials paramterized with $t \in \R$ and $z \in \mathbb{R}^D$. In~\Cref{prop:construct_multinomial}, we introduce a function $\phi$ that enables us to construct each polynomial  --- in the aforementioned class --- using only $\Phi(X) = \sum_{x \in X} \phi(x)$. In other words, knowing $t$, $z$ and $\Phi(X)$, we can represent the polynomial  $\prod_{x \in X} (t-z^{\top}x)$. This allows us to write the polynomial $\prod_{x \in X} (t-z^{\top}x)$ as a function depending on variables $t, z$ and $\Phi(X)$, which we call $p(t; z, \Phi(X))$.
    \begin{proposition}\label{prop:construct_multinomial}
        Let $N, D \in \mathbb{N}$ and $\phi:\mathbb{R}^{D} \rightarrow \mathrm{codom}(\phi) \subseteq \mathbb{R}^{{N+D \choose D}-1}$ be the following continuous function:
        \[
            \forall x \in \mathbb{R}^D: \ \phi(x) = \big(  \prod_{d=1}^D x_{d}^{k_d} \big)_{k \in \mathcal{K}_{N}^D} \in \R^{{N+D\choose{D}}-1},
        \]
        where $k = (k_d)_{d \in [D]}$ is a $D$-tuple and $\mathcal{K}_N^D = \{ (k_d)_{d \in [D]}: k_1+\ldots+ k_D \in [N], k_1, \ldots, k_D \geq 0 \}$. Then, for all $X \in \mathbb{X}_{\mathbb{R}^D,N}$, $\Phi(X) = \sum_{x \in X} \phi(x)$ suffices to construct the following multivariate polynomial:
        \begin{equation} \label{eq:multinomial_polynomial_in_two_vs}
            \forall t \in \R, z \in \mathbb{R}^D: \  \prod_{x \in X} (t-z^{\top}x) = p\big(t; z, \Phi(X) \big).
        \end{equation}
    \end{proposition}
    To show that $\Phi$ is an invertible function, we want to argue that the multiset $X$ can be uniquely recovered form the multivariate polynomial in equation \eqref{eq:multinomial_polynomial_in_two_vs}, that is, $p\big(t; z, \Phi(X) \big)$. Let us proceed with the following definitions. 
    \begin{definition}\label{def:roots_separators}
        For any $z \in \mathbb{R}^D$, multiset $X$ of at least two $D$-dimensional vectors, and multivariate polynomial $p\big(t;z, \Phi(X) \big)$ in the equation \eqref{eq:multinomial_polynomial_in_two_vs}, we formalize the following functions:
        \begin{itemize}
            \item $\mathrm{roots} \circ p\big(t;z, \Phi(X) \big) =  \{ \{ t : p\big(t;z, \Phi(X) \big) = 0 \} \} =\{ \{ z^{\top} x : x \in X \} \} \stackrel{\mathrm{def}}{=} z^{\top} X$
            \item $\mathrm{separators} \circ  p\big(t; z, \Phi(X) \big) = \mathrm{argmax}_{z \in \mathbb{R}^D} | \mathrm{unique} \circ \mathrm{roots} \circ p\big(t; z, \Phi(X) |$ ,
        \end{itemize}
        where $|\cdot|$ returns the cardinality of its input set.
    \end{definition}    
    \begin{definition} \label{def:directional_derivative}
        Let $X$ be a multiset of at least two $D$-dimensional vectors. If exists, the directional derivative of $\mathrm{sort} \big( z^{\top} X \big)$ --- where $z^{\top} X =  \{ \{ z^{\top} x: x \in X \} \}$ --- at $z \in \mathbb{R}^D$ in the direction of unit norm $v \in \mathbb{R}^D$ is given as follows:
        \begin{equation}\label{eq:x_pi_d}
            \nabla_{v} \mathrm{sort} \big( z^{\top} X \big) =\lim_{\delta \rightarrow 0}  \frac{1}{\delta} \Big( \mathrm{sort} \big( (z + \delta v)^{\top} X \big)  - \mathrm{sort}\big( z^{\top} X \big) \Big)  .
        \end{equation}    
    \end{definition}
    In~\Cref{prop:zXtoX}, we show how to retrieve $X$ from $\mathrm{roots} \circ p\big(t;z, \Phi(X) \big)$, that is, the parameterized multiset $z^{\top}X$.
    \begin{proposition} \label{prop:zXtoX}
        For any $z \in \mathbb{R}^D$, multiset $X = \{ \{ x_n \in \mathbb{R}^D : n \in [N] \} \}$ where $N \geq 2$, and the multivariate polynomial $p\big(t;z, \Phi(X) \big)$ in the equation \eqref{eq:multinomial_polynomial_in_two_vs}, we have
        \[
            \mathrm{separators} \circ  p\big(t; z, \Phi(X) \big) \neq \emptyset.
        \]
        Moreover, for any $z^* \in \mathrm{separators} \circ  p\big(t; z, \Phi(X) \big)$, the directional derivative of $\mathrm{sort}\circ \mathrm{roots} \circ p\big(t;z, \Phi(X) \big)$ is well-defined and we have:
        \[
            \forall d \in [D] : \nabla_{e_{d}} \mathrm{sort}\circ \mathrm{root} \circ p\big(t; z, \Phi(X) |_{z = z^*} = (e_d^{\top} x_{\pi_{z^*}(n)})_{n \in [N]},
        \]
        where $e_d \in \mathbb{R}^D$ is the $d$-th standard basis vector for $\mathbb{R}^D$ ($d \in [D]$), and $\pi_{z^*}: [N] \rightarrow [N]$ is a permutation operator that sorts the elements ${z^*}^{\top}X$ --- see~\Cref{def:gap_sort_diam_unique}. 
    \end{proposition}
    In summary, given $\Phi(X)  \in \mathbb{R}^{{N+D \choose D}-1}$, we can construct a multivariate polynomial $p\big(t;z, \Phi(X)\big)$ with parameterized roots $z^{\top} X$; see~\Cref{prop:construct_multinomial}. Then, we can pick a fixed vector $z^* \in \mathrm{separators} \circ  p\big(t; z, \Phi(X) \big) \neq \emptyset$; see~\Cref{def:roots_separators} and~\Cref{prop:zXtoX}. We then prove the following result:
    \[
        \forall d \in [D] : \nabla_{e_{d}} \mathrm{sort}\circ \mathrm{root} \circ p\big(t; z, \Phi(X) |_{z = z^*} = (e_d^{\top} x_{\pi_{z^*}(n)})_{n \in [N]},
    \]
    where $\nabla_{e_{d}}$ computes the directional derivative (see \Cref{def:directional_derivative}) in the direction of $e_d$ --- the $d$-th standard basis of $\mathbb{R}^D$ --- for $d \in [D]$, and $x_n \in \mathbb{R}^D$ is an element of $X$ indexed by $n$. We retrieve $X$ as follows:
    \begin{align*}
        \{ \{ (e_d^{\top} x_{\pi_{z^*}(n)})_{d \in [D]} \in \mathbb{R}^D :n \in [N] \} \} &=\{ \{ (e_d^{\top} x_{n} )_{d \in [D]} \in \mathbb{R}^D :n \in [N] \} \}  \\
        &= X.
    \end{align*}
    This result does not depend on the specific choices for the permutation operator $\pi_{z^*}$ and 
     $z^* \in \mathrm{separators} \circ  p\big(t; z, \Phi(X) \big)$. Therefore, $\Phi$ is an invertible multiset function, that is,
    \begin{equation}\label{eq:phi_inverse}
        \Phi^{-1} \circ \Phi(X) = \{ \{ \begin{pmatrix}
            \big( \nabla_{e_{1}} \mathrm{sort}\circ \mathrm{root} \circ p\big(t; z, \Phi(X) |_{z = z^*}\big)_{n} \\
            \vdots \\
            \big( \nabla_{e_{D}} \mathrm{sort}\circ \mathrm{root} \circ p\big(t; z, \Phi(X) |_{z = z^*} \big)_n
        \end{pmatrix}  \in \mathbb{R}^D :n \in [N] \} \}, 
    \end{equation}
    where $z^* \in \mathrm{separators} \circ  p\big(t; z, \Phi(X) \big)$, and subscript $n$ denotes the $n$-the element of the $N$-dimensional vectors. The function $\Phi^{-1}$ is only well-defined on $\codom(\Phi)$; see equation \eqref{eq:phi_inverse}. \\
    Now we let $\rho: \mathrm{codom}(\Phi) \rightarrow \mathrm{codom}(f)$ where
    \[
        \forall y \in \mathrm{codom}(\Phi) \subseteq \mathbb{R}^{{N+D \choose D}-1} : \rho (y) = f \circ \Phi^{-1}(y).
    \]
    This proves the sum-decomposition representation claim of the theorem, that is, $f = \rho \circ \Phi$. In \Cref{sec:proof_construct_multinomial,sec:zXtoX} we provide proofs of \Cref{prop:construct_multinomial,prop:zXtoX}. In \Cref{section:example}, we provide two illustrative examples on computing $\Phi^{-1}$.
\subsection{Proof of \Cref{prop:construct_multinomial}}\label{sec:proof_construct_multinomial}
    We expand the expression in equation \eqref{eq:multinomial_polynomial_in_two_vs} as follows:
    \begin{equation}\label{eq:multinomial_polynomial_in_two_vs_coeff}
        \forall t \in \R,z \in \mathbb{R}^D: \prod_{x \in X} (t-z^{\top}x) = t^N + \sum_{n \in [N]} (-1)^n a_n(z; X) t^{N-n} 
    \end{equation}
    where each coefficient $a_n(z; X)$ is determined using the Newton-Girard formulae~\citep{seroul2012programming}, that is, 
    \begin{equation}\label{eq:girard}
    a_n(z; X) = \frac{1}{n} \det \begin{pmatrix}
            E_1(z; X) & 1 &0 & \cdots & 0 \\
            E_2(z; X) & E_1(z; X) & 1 &  \cdots & 0 \\
            \vdots & \vdots & \vdots  & \cdots & \vdots \\
            E_{n}(z; X) & E_{n-1}(z; X) & E_{n-2}(z; X) &  \cdots & E_{1}(z; X) \\
    \end{pmatrix}
    \end{equation}
    for all $n \in [N]$ and $z \in \mathbb{R}^D$, and $E_n(z;X)  =\sum_{x \in X} (z^{\top}x)^{n} $. Therefore, each coefficient $a_n(z; X)$ is a polynomial function of $\{ E_n(z;X) \}_{n = 1}^{N}$ --- 
    moments of the parameterized multiset $\{ \{ z^{\top} x : x \in X \} \} \stackrel{\mathrm{def}}{=} z^{\top} X$. \Cref{lem:psi_phi} lets us relate each moment to the elementary symmetric polynomials.
    \begin{lemma} \label{lem:psi_phi}
        For any $k_1, \cdots, k_D \in \mathbb{N} \cup \{0\}$ and $n \in \mathbb{N}$, let
            \[
            {n \choose k_1, \ldots, k_D}^{\mathrm{ind}} = 
                \begin{cases}
                    \frac{n!}{k_1! \cdots k_D!} \ &\mbox{if} \ k_1+\cdots+k_D = n \\
                    0  \ &\mbox{otherwise}.
                \end{cases}
            \]
            Let $x,z \in \mathbb{R}^D$ and $n \in [N]$. Then, we have $(z^{\top}x)^n =  \langle \psi(z,n), \phi(x) \rangle $ such that
        \begin{equation}\label{eq:phi}
            \psi(z,n) = \Big( {n \choose k_1, \ldots, k_D}^{\mathrm{ind}} \prod_{d=1}^D z_{d}^{k_d} \Big)_{k \in \mathcal{K}_{N}^D} ,~~~\phi(x) = \big(  \prod_{d=1}^D x_{d}^{k_d} \big)_{k \in \mathcal{K}_{N}^D} \in \R^{{N+D\choose{D}}-1},
        \end{equation}
        where $k = (k_d)_{d \in [D]}$ and $\mathcal{K}_N^D = \{ (k_d)_{d \in [D]}: k_1+\ldots+ k_D \in [N], k_1, \ldots, k_D \geq 0 \}$.
    \end{lemma}
    \begin{proof}
        Let $x,z \in \mathbb{R}^D$ and $n \in [N]$. Then, we have    
        \begin{align*}
            (z^{\top}x)^{n} = ( \sum_{d \in [D]} z_d x_d)^{n} = \sum_{k_1+\ldots+ k_D = n}  {n \choose k_1, \ldots, k_D}\prod_{d=1}^D z_{d}^{k_d}  \prod_{d=1}^D x_{d}^{k_d}  =  \langle  \psi( z,n ),\phi(x) \rangle 
        \end{align*}
        where $\phi(x)$ and $\psi(z,n)$ are given in equation \eqref{eq:phi}. Since the dimension of $\phi(x)$ --- the size of $\mathcal{K}_N^D$ --- is equivalent to the number of solutions to the following problem:
        \begin{equation} \label{eq:dimension_prob}
            k_1, \ldots, k_D \in \mathbb{N}\cup \{ 0 \}: 1 \leq \sum_{d = 1}^D k_d \leq N.
        \end{equation}
        We can transform the problem in equation \eqref{eq:dimension_prob} to the following form:
        \begin{equation} \label{eq:dimension_prob2}
            k_1, \ldots, k_D, k_{\circ} \in \mathbb{N}\cup \{ 0 \}, k_{\circ} \neq N: \sum_{d = 1}^D k_d + k_{\circ} = N.
        \end{equation}
        In the occupancy problem, we ask: \emph{how many ways can one distribute $N$ indistinguishable objects into $D+1$ distinguishable bins?} The number of nonnegative solutions are ${N+D\choose{D}}$; refer to \citep{feller1967introduction}, section 5. However, if $k_{\circ} = N$, then $k_1 = k_2 = \cdots = k_D = 0$ which is not allowed. If we exclude this case, we arrive at ${N+D\choose{D}}-1$ integer solutions for problems in equations \eqref{eq:dimension_prob} and \eqref{eq:dimension_prob2}
    \end{proof}
    Let us now prove the proposition's statement. Given $\Phi(X) = \sum_{x \in X} \phi(x)$, we compute
    \[
        \forall z \in \mathbb{R}^D, n \in [N]: E_n(z; X) = \sum_{x \in X} \langle z,x\rangle^n = \sum_{x \in X} \langle \psi(z,n), \phi(x) \rangle =  \langle \psi(z,n), \Phi(X) \rangle,
    \]
    that are, all parameterized moments required to construct $\prod_{x \in X} (t - x^{\top}z)$ --- refer to \Cref{lem:psi_phi}, and equation \eqref{eq:girard}. Therefore, we can uniquely identify the polynomial in equation \eqref{eq:multinomial_polynomial_in_two_vs} with only $\Phi(X)$.
    \subsection{Proof of \Cref{prop:zXtoX}}\label{sec:zXtoX}
    \begin{proposition}
        For any $z \in \mathbb{R}^D$, multiset $X = \{ \{ x_n \in \mathbb{R}^D : n \in [N] \} \}$ where $N \geq 2$, and the multivariate polynomial $p\big(t;z, \Phi(X) \big)$ in the equation \eqref{eq:multinomial_polynomial_in_two_vs}, we have
        \[
            \mathrm{separators} \circ  p\big(t; z, \Phi(X) \big) \neq \emptyset.
        \]
        Moreover, for any $z^* \in \mathrm{separators} \circ  p\big(t; z, \Phi(X) \big)$, the directional derivative of $\mathrm{sort}\circ \mathrm{roots} \circ p\big(t;z, \Phi(X) \big)$ is well-defined and we have:
        \[
            \forall d \in [D] : \nabla_{e_{d}} \mathrm{sort}\circ \mathrm{root} \circ p\big(t; z, \Phi(X) |_{z = z^*} = (e_d^{\top} x_{\pi_{z^*}(n)})_{n \in [N]},
        \]
        where $e_d \in \mathbb{R}^D$ is the $d$-th standard basis vector for $\mathbb{R}^D$ ($d \in [D]$), and $\pi_{z^*}: [N] \rightarrow [N]$ is a permutation operator that sorts the elements ${z^*}^{\top}X$ --- see~\Cref{def:gap_sort_diam_unique}. 
    \end{proposition}
    For any $z \in \mathbb{R}^D$ and multiset $X = \{ \{ x_n \in \mathbb{R}^D : n \in [N] \} \}$ where $N \geq 2$, we have
    \[
        \mathrm{sort} \circ \mathrm{root} \circ p(t;z,X) = ( z^{\top} x_{\pi_{z}(n)})_{n \in [N]} \in \R^N
    \]
    where $\pi_{z}: [N] \rightarrow [N]$ is a permutation operator such that $z^{\top} x_{\pi_{z}(1)} \geq {z}^{\top} x_{\pi_{z}(2)} \geq \cdots \geq {z}^{\top} x_{\pi_{z}(N)}$; see \Cref{def:gap_sort_diam_unique}. Given such an ordered list, we want to retrieve the multiset $X$. {\bf If} the order of the elements of $X$ after sorting remains unchanged for a perturbed parameter $z + \delta e_d$ --- where $e_d \in \mathbb{R}^D$ is the $d$-th standard basis for $\mathbb{R}^D$ and small enough $\delta \in \R$, that is, $x_{\pi_{z+\delta e_d}(n)} = x_{\pi_{z}(n)}$ for all $n \in [N]$ and $d \in [D]$ --- then we have the following equality:
    \begin{align*} 
    \frac{1}{\delta} \Big( \mathrm{sort}( (z + \delta e_d)^{\top} X) -  \mathrm{sort}( z^{\top} X) \Big) &= \frac{1}{\delta}\Big( (z+\delta e_d)^{\top} x_{\pi_{z+\delta e_d}(n)} -  z^{\top} x_{\pi_{z}(n)} \Big)_{n \in [N]} \\
        &\stackrel{\text(a)}{=} \frac{1}{\delta}\Big( {z}^{\top} x_{\pi_{z}(n)}  + \delta e_d^{\top} x_{\pi_{z}(n)}  -  {z}^{\top} x_{\pi_{z}(n)} \Big)_{n \in [N]}\\
        &= \frac{1}{\delta}( \delta e_d^{\top} x_{\pi_{z}(n)})_{n \in [N]} =   ( e_d^{\top} x_{\pi_{z}(n)})_{n \in [N]},
    \end{align*}
    where $\text{(a)}$ is due to our assumption $x_{\pi_{z+\delta e_d}(n)} = x_{\pi_{z}(n)}$ for all $n \in [N]$ and $d \in [D]$. If this property holds true, we can compute the following limit:
    \begin{align}
        \lim_{\delta \rightarrow} \frac{1}{\delta} \Big(\mathrm{sort} \circ \mathrm{root} \circ p(t;z + \delta e_d,X) &-\mathrm{sort} \circ \mathrm{root} \circ p(t;z,X) \Big) \label{eq:limit_obs} \\
        &= \lim_{\delta \rightarrow} \frac{1}{\delta} \Big(\mathrm{sort} ((z + \delta e_d)^{\top} X) -\mathrm{sort} ( z^{\top} X) \Big), \nonumber
    \end{align}
    to retrieve the $d$-component of the elements in $X$ up to a fixed but unknown permutation $\pi_{z}$ that does not depend on $e_d$ --- that is, $( e_d^{\top} x_{\pi_{z}(n)})_{n \in [N]}$ --- for all $d \in [D]$. The limit in equation \eqref{eq:limit_obs} is well-defined and returns $( e_d^{\top} x_{\pi_{z}(n)})_{n \in [N]}$ if there exists a vector $z \in \mathbb{R}^D$ such that it admits a solution for the following feasibility problem:
    \[
        \mbox{find} \ \delta^* >0 \ \mbox{such that} \ x_{\pi_{z+\delta e_d}(n)} = x_{\pi_{z}(n)}, \ \ \mbox{for all } n \in [N], d \in [D], \delta \leq \delta^*.
    \]
    As we shall see, any vector $z^* \in \mathrm{separators} \circ  p\big(t; z, \Phi(X) \big)$ admits a solution to the aforementioned problem. To prove this result, we first need to derive the following property for the separators.
    \begin{lemma}\label{lem:separator}
        For any $z \in \mathbb{R}^D$, multiset $X$ of at least two $D$-dimensional vectors, and the multivariate polynomial $p\big(t;z, \Phi(X) \big)$ in the equation \eqref{eq:multinomial_polynomial_in_two_vs}, we have $\mathrm{separators} \circ  p\big(t; z, \Phi(X) \big)$ is a nonempty subset of $\mathbb{R}^D$ and for all $z^* \in \mathrm{separators} \circ  p\big(t; z, \Phi(X) \big)$, we have
        \begin{align*}
            | \mathrm{unique} \circ \mathrm{roots} \circ p\big(t; z^*, \Phi(X) \big)  &=  \mathrm{max}_{z \in \mathbb{R}^D} | \mathrm{unique} \circ \mathrm{roots} \circ p\big(t; z, \Phi(X) \big) \\
            &= |\mathrm{unique}(X)|.
        \end{align*}
    \end{lemma}    
    \begin{proof}
        If $|\mathrm{unique}(X)| = 1$, then $\mathrm{separators} \circ  p\big(t; z, \Phi(X) \big) = \mathbb{R}^D$ and the statement is trivial. Therefore, in what follows, we assume $|\mathrm{unique}(X)| > 1$. \\
        Let $X$ be a multiset of (at least two distinct) $D$-dimensional vectors and $\mathrm{roots} \circ p\big(t; z, \Phi(X) = z^{\top} X$, for all $z \in \mathbb{R}^D$. If $x,x^{\prime} \in X$ where $x \neq x^{\prime}$, then we have $z^{\top} x = z^{\top} x^{\prime}$ for $z \in (x - x^{\prime})^{\perp} \subset \mathbb{R}^D$. Therefore, we have 
        \[
            \forall z \in \mathbb{R}^D: | \mathrm{unique}(z^{\top} X)| \leq |\mathrm{unique}(X)|.
        \]
        We can prove the claim if we show $| \mathrm{unique}(z^{\top} X)|$ achieves its upper bound $|\mathrm{unique}(X)|$ over a subset of $\mathbb{R}^D$ --- namely, $\mathrm{separators} \circ  p\big(t; z, \Phi(X) \big)$.  \\
        Let $P_{x,x^{\prime}} = (x-x^{\prime})^{\perp}$ for distinct $x, x^{\prime} \in \mathrm{unique}(X)$ --- that is, $x \neq x^{\prime}$. By construction, $P_{x,x^{\prime}}$ is a $(D-1)$-dimensional subspace since $x \neq x^{\prime}$. Since $\mathrm{unique}(X)$ contains only distinct elements, we have
        \[
            \forall z \in P_{x,x^{\prime}} \Longleftrightarrow \langle z, x-x^{\prime} \rangle = 0,
        \]
        for all distinct $x , x^{\prime} \in \mathrm{unique}(X) $. We now construct the following set:
        \[
            P_{X} = \bigcup_{ \stackrel{x , x^{\prime} \in \mathrm{unique}(X)}{ x \neq x^{\prime}}} P_{x,x^{\prime}},
        \]
        which is a finite union of $(D-1)$-dimensional subspaces. Therefore, $P_{X}$ can not be equal to $\mathbb{R}^D$, that is, $\mathbb{R}^D \setminus P_{X}$ is a nonempty set. For any $z^* \in \mathbb{R}^D \setminus P_{X}$, we have
        \begin{align*}
            \forall x, x^{\prime} \in X, x \neq x^{\prime}: \langle z^* , x-x^{\prime} \rangle &= {z^*}^{\top}x-{z^*}^{\top} x^{\prime} \neq  0 \\
            \forall x, x^{\prime} \in X, x = x^{\prime}: \langle z^* , x-x^{\prime} \rangle &= {z^*}^{\top}x-{z^*}^{\top} x^{\prime} =  0.
        \end{align*}
        Hence, we have $| \mathrm{unique}({z^*}^{\top} X)| = |\mathrm{unique}(X)|$ for all $z^* \in \mathrm{separators} \circ  p\big(t; z, \Phi(X) \big)$ where $\mathrm{separators} \circ  p\big(t; z, \Phi(X) \big) = \mathbb{R}^D \setminus P_{X}$  --- a nonempty subset of $\mathbb{R}^D$.
    \end{proof}
    As a result of \Cref{lem:separator}, we have
    \[
        \forall  z^* \in \mathrm{separators} \circ  p\big(t; z, \Phi(X) \big) : | \mathrm{unique}({z^*}^{\top} X)| = |\mathrm{unique}(X)|,
    \]
    that is, repeated (or distinct) elements in ${z^*}^{\top} X$ correspond to identical (or distinct) elements in $X$. We now want to show that the following directional derivative is well-defined:
    \[
        \nabla_{v} \mathrm{sort} \big( z^{\top} X \big)|_{z = z^*} =\lim_{\delta \rightarrow 0}  \frac{1}{\delta} \Big( \mathrm{sort} \big( (z^* + \delta v)^{\top} X \big)  - \mathrm{sort}\big( {z^*}^{\top} X \big) \Big)  
    \]
    for all $z^* \in \mathrm{separators} \circ  p\big(t; z, \Phi(X) \big)$ and unit norm vector $v \in \mathbb{R}^D$. \\
    We break down the rest of the proof in two cases. \\
    {\bf{Case 1:} $\mathbf{|\mathrm{\bf unique}(X)| > 1}$.} \\
    \textbf{Limiting behavior of $(z^* + \delta v)^{\top} X $ as $\delta \rightarrow 0$.} \\
    Let $x, x^{\prime}$ be two distinct elements in $\mathrm{unique}(X)$, that is, $\| x -x^{\prime} \|_2 > 0$. If $z^* \in \mathrm{separators} \circ  p\big(t; z, \Phi(X) \big)$, then we have $|{z^*}^{\top} x - {z^*}^{\top} x^{\prime}| =\varepsilon > 0$; see \Cref{lem:separator}. Let $z^*_v(\delta) = z^* + \delta v$ where $v \in \mathbb{R}^D$ is a unit norm vector and $\varepsilon = \mathrm{gap}( {z^*}^{\top} X) >0 $ --- which is well-defined since $|\mathrm{unique}(X)| > 1$. Then, we have
    \begin{align*}
        \forall \ \mbox{distinct} \ x , x^{\prime} \in \mathrm{unique}(X),\delta < \frac{\varepsilon}{2\mathrm{diam}(X)} : &\| z^*_v(\delta)^{\top}(x - x^{\prime}) \|_2 = \| (z^* + \delta v)^{\top}(x - x^{\prime}) \|_2 \\
        &\stackrel{\text{(a)}}{\geq} \| {z^*}^{\top}(x - x^{\prime}) \|_2 - \delta \|v^{\top}(x - x^{\prime}) \|_2 \\
        &\stackrel{\text{(b)}}{>} \varepsilon - \frac{\varepsilon}{2\mathrm{diam}(X)}\| x - x^{\prime} \|_2 \\
        &\stackrel{\text{(c)}}{\geq} \varepsilon - \frac{\varepsilon}{2}  = \frac{\varepsilon}{2}  > 0,
    \end{align*}
    where $\text{(a)}$ is due to the triangle inequality, $\text{(b)}$ is due to $|{z^*}^{\top} x - {z^*}^{\top} x^{\prime}| =\varepsilon$ and $\delta < \frac{\varepsilon}{2\mathrm{diam}(X)}$, and $\text{(c)}$ is due to $\| x - x^{\prime} \|_2 \leq \mathrm{diam}(X)$. Therefore, the vector $z^*_v(\delta)$ separates distinct elements of $X$ in $z^*_v(\delta)^{\top}X$ --- for all unit norm vectors $v \in \mathbb{R}^D$ and $\delta < \frac{\varepsilon}{2\mathrm{diam}(X)}$. On the other hand, if $x, x^{\prime}$ are two identical elements in $X$, then we have $z^*_v(\delta)^{\top}x = z^*_v(\delta)^{\top} x^{\prime}$ --- that is, the repeated elements in $X$ correspond to the repeated elements in $z^*_v(\delta)^{\top}X$. Therefore, we have $|\mathrm{unique}(z^*_v(\delta)^{\top}X)| = |\mathrm{unique}(X)|$, or equivalently $z^*_v(\delta)^{\top} \in \mathrm{separators} \circ p(t;z,X)$. \\
    \textbf{Directional derivative of $\mathrm{sort} \circ \mathrm{root} \circ p(t;z,X)$ at $z = z^*$.} \\
    Let $z^* \in \mathrm{separators} \circ  p\big(t; z, \Phi(X) \big) $. Then, we have
    \[
        \mathrm{sort} \circ \mathrm{root} \circ p(t;z^{*},X) = \big( {z^*}^{\top} x_{\pi_{z^{*}}(n)}\big)_{n\in [N]} \in \R^N,
    \]
    where $\pi_{z^*}: [N] \rightarrow [N]$ is a permutation operator such that ${z^*}^{\top} x_{\pi_{z^*}(1)} \geq {z^*}^{\top} x_{\pi_{z^*}(2)} \geq \cdots \geq {z^*}^{\top} x_{\pi_{z^*}(N)}$. The repeated elements in $X$ do not change the value of the output of the sort function as they correspond to the repeated elements in ${z^*}^{\top} X$. In other words, $\pi_{z^*}$ is not necessarily unique; but our results do not depend on the specific choice of the permutation operator. The minimum distance between distinct elements of ${z^*}^{\top} X$ is $\varepsilon = \mathrm{gap}( {z^*}^{\top} X) > 0$.
    If $z^*_v(\delta)$ is the perturbed version of $z^*$ in direction of $v$ such that $\| z^* - z^*_{v}(\delta) \|_2  = \delta < \frac{\varepsilon}{2\mathrm{diam}(X)}$, then $z^*_v(\delta) \in \mathrm{separators} \circ p(t;z,X)$ --- see our discussion in the previous paragraph.  \\
    \begin{claim}\label{claim:perturbed_permutation}
        The following equality holds true:
    \[
        \forall n \in [N]: x_{\pi_{z^*_v(\delta)}(n)}  =   x_{\pi_{z^*}(n)},
    \]
    for any unit norm vector $v \in \mathbb{R}^D$ and $\delta < \frac{\varepsilon}{2\mathrm{diam}(X)}$, and any permutation operator 
    $\pi_{z^*_v(\delta)}: [N] \rightarrow [N]$ such that $z^*_v(\delta)^{\top} x_{\pi_{z^*_v(\delta)}(1)} \geq z^*_v(\delta)^{\top} x_{\pi_{z^*_v(\delta)}(2)} \geq \cdots \geq z^*_v(\delta)^{\top} x_{\pi_{z^*_v(\delta)}(N)}$.
    \end{claim}
    \begin{proof}
        Consider $i,j \in [N]$ where $i > j$. If $x_{\pi_{z^*}(j)} = x_{\pi_{z^*}(i)}$, then we have $z^*_v(\delta)^{\top} x_{\pi_{z^*}(j)} \geq z^*_v(\delta)^{\top} x_{\pi_{z^*}(i)}$ --- as both terms are equal to each other. On the other hand, if  $x_{\pi_{z^*}(j)} \neq x_{\pi_{z^*}(i)}$, then we have
    \begin{align*}
        z^*_v(\delta)^{\top} x_{\pi_{z^*}(j)} - z^*_v(\delta)^{\top} x_{\pi_{z^*}(i)} &= (z^* + \delta v)^{\top} (x_{\pi_{z^*}(j)} -  x_{\pi_{z^*}(i)}) \\
        &\stackrel{\text{(a)}}{\geq} {z^*}^{\top}(x_{\pi_{z^*}(j)} -  x_{\pi_{z^*}(i)}) - \delta \| x_{\pi_{z^*}(j)} -  x_{\pi_{z^*}(i)} \|_2 \\
        &\stackrel{\text{(b)}}{\geq}  \varepsilon - \delta \mathrm{diam}(X) > \varepsilon - \frac{\varepsilon}{2} = \frac{\varepsilon}{2} >0,
    \end{align*}
    where $\text{(a)}$ is due to Cauchy–Schwarz inequality, and $\text{(b)}$ is due to $\| x_{\pi_{z^*}(j)} -  x_{\pi_{z^*}(i)} \|_2 \leq \mathrm{diam}(X)$ and $\delta < \frac{\varepsilon}{2\mathrm{diam}(X)}$. Therefore, the permutation $\pi_{z^*}$ also sorts the elements of $z^*_v(\delta)^{\top} X$, that is, $x_{\pi_{z^*_v(\delta)}(n)}  =   x_{\pi_{z^*}(n)}$, for all $n \in [N]$.
    \end{proof}
    Finally, for all $d \in [D]$, we have
    \begin{align*} \label{eq:approximation_sort}
     \nabla_{e_{d}} \mathrm{sort}\circ \mathrm{root} \circ p\big(t; z, \Phi(X) |_{z = z^*} &\stackrel{\text{(a)}}{=} \lim_{\delta \rightarrow 0}\frac{1}{\delta} \Big( \mathrm{sort}( (z^* + \delta e_d)^{\top} X) -  \mathrm{sort}( {z^*}^{\top} X) \Big) \\
    &\stackrel{\text{(b)}}{=} \lim_{\delta \rightarrow 0}\frac{1}{\delta}\Big( (z^*+\delta e_d)^{\top} x_{\pi_{z^*+\delta e_d}(n)} -  {z^*}^{\top} x_{\pi_{z^*}(n)} \Big)_{n \in [N]} \\
    &\stackrel{\text{(c)}}{=} \lim_{\delta \rightarrow 0}\frac{1}{\delta}\Big( (z^*+\delta e_d)^{\top} x_{\pi_{z^*}(n)} -  {z^*}^{\top} x_{\pi_{z^*}(n)} \Big)_{n \in [N]} \\
        &= \lim_{\delta \rightarrow 0} \frac{1}{\delta}( \delta e_d^{\top} x_{\pi_{z^*}(n)})_{n \in [N]} =   ( e_d^{\top} x_{\pi_{z^*}(n)})_{n \in [N]}.
    \end{align*}
    where  $\text{(a)}$ is due to the definition of the directional derivation, $\text{(b)}$ follows from the definition of permutation operator in $\mathrm{sort}$ function, and $\text{(c)}$ follows from Claim~\ref{claim:perturbed_permutation}. \\
    {\bf{Case 2:} $\mathbf{|\mathrm{\bf unique}(X)| = 1}$.} \\
This directional derivative is well-defined if $|\mathrm{unique}(X)| = 1$, that is,
    \begin{align*}
        \nabla_{v} \mathrm{sort} \big( z^{\top} X \big)|_{z = z^*} &=\lim_{\delta \rightarrow 0}  \frac{1}{\delta} \Big( \mathrm{sort} \big( (z^* + \delta v)^{\top} X \big)  - \mathrm{sort}\big( {z^*}^{\top} X \big) \Big)   \\
        &=\lim_{\delta \rightarrow 0}  \frac{1}{\delta} \Big( \big( (z^* + \delta v)^{\top} x_{\pi_1(n)} \big)_{n \in [N]}  - \big( {z^*}^{\top} x_{\pi_2(n)} \big)_{n\in [N]} \Big)  
        &=   v^{\top} x 1,
    \end{align*}
    where $\pi_1, \pi_2 : [N] \rightarrow [N]$ are two permutation operators, and $X = \{ \{ x_n : n \in [N] \} \}$ and $x_n = x$ for all $n \in [N]$, and $1 \in \R^N$ is the vector of all ones.  Therefore, for all $z^* \in \mathrm{separators} \circ  p\big(t; z, \Phi(X) \big) =\mathbb{R}^D$. And we have
    \begin{align*} 
     \forall d \in [D]: \nabla_{e_{d}} \mathrm{sort}\circ \mathrm{root} \circ p\big(t; z, \Phi(X) |_{z = z^*} =   e_d^{\top} x 1.
    \end{align*}
    This readily proves the proposition's statement.
    \subsection{Two Illustrative examples}\label{section:example}
    \begin{example}[Repeated Roots]\label{ex:multivarite_sum_decomposition_repeated_roots}
        Let $N = D = 2$, and $\Phi(X) = \begin{pmatrix}
                2 & 0 & 2 & 0 & 0
            \end{pmatrix}^{\top} \in \mathbb{R}^{{N+D \choose D} -1} = \R^5$ for a multiset $X$. The goal is to recover $X$.
        In the proof of \Cref{prop:construct_multinomial}, \Cref{lem:psi_phi} relates parameterized moments of the multivariate polynomial $p(t;z, \Phi(X))$ to  $\Phi(X)$ using the following functions:
        \[
            \forall z = (z_1, z_2)^{\top} \in \R^2 : \ \psi( z,1) = \begin{pmatrix}
                z_1 & z_2 & 0 & 0 & 0
            \end{pmatrix}^{\top}, \ \psi(z,2) = \begin{pmatrix}
                0 & 0 & z_1^2 & 2 z_1 z_2 & z_2^2
            \end{pmatrix}^{\top}.
        \]
        Since $E_n(z,X) = \langle \psi(z,n), \Phi(X) \rangle$, for $n \in [2]$, then we have $E_1(z,X) = 2z_1$ and $ E_2(z,X) = 2z_1^2$. We now can use Girard's formula (see equation \eqref{eq:girard}):
        \[
            a_1(z;X) = E_1(z,X) , a_2(z;X) = \frac{1}{2} \mathrm{det} \begin{pmatrix}
                E_1(z,X) & 1 \\
                E_2(z,X) & E_1(z,X),
            \end{pmatrix}.
        \]
        to compute the coefficients of the multivariate polynomial as $a_1(z;X) = 2 z_1$ and $a_2(z;X) = z_1^2$. We arrive at the following multivariate polynomial:
        \[
            p\big(t; z, \Phi(X)) = t^2 - a_1(z;X) t + a_2(z;X) =t^2 - 2 z_1 t + z_1^2 = (t-z_1)^2,
        \]
        and $\mathrm{roots} \circ p\big(t; z, \Phi(X) = z^{\top} X = \{ \{ z_1, z_1 \} \}$, for all $z \in \R^2$. Since  $\mathrm{unique}(z^{\top} X) = 1$ --- $\forall z \in \R^2$ ---- then we have $\mathrm{separators} \circ p\big(t; z, \Phi(X) = \R^2$. Let $z^* = (1,1)^{\top} \in \R^2$ be a separator vector. Therefore, we have $\mathrm{sort}(z^{\top}X)|_{z = z^*} = (1,1)^{\top}$. We also have
        \begin{align*}
            \mathrm{sort}(z+\delta e_1)^{\top}X|_{z = z^*} = (1+\delta,1+\delta)^{\top}, \ \ \mathrm{sort}(z+\delta e_2)^{\top}X|_{z = z^*} = (1,1)^{\top}.
        \end{align*}
        for all $\delta >0$. These quantities let us compute the directional derivatives in \Cref{prop:zXtoX} as follows:
        \[
            (e_1^{\top} x_{\pi_{z^*}}(n))_{n \in [2]} = (1,1)^{\top},  \ (e_2^{\top} x_{\pi_{z^*}}(n))_{n \in [2]} = (0,0)^{\top},
        \]
        see equation \eqref{eq:x_pi_d}. Finally, we arrive at $X = \Phi^{-1} \circ \Phi(X) = \{ \{ (1,0) , (1,0) \} \}$.
    \end{example} 
    \begin{example}[Unique Roots]\label{ex:multivarite_sum_decomposition_unique_roots}
        Let $N = D = 2$, and $\Phi(X) = \begin{pmatrix}
                -2 & 1 & 10 & -7 & 5
            \end{pmatrix}^{\top} \in \mathbb{R}^{{N+D \choose D} -1} = \R^5$ 
        for a multiset $X$. The goal is to recover $X$. Since $E_n(z,X) = \langle \psi(z,n), \Phi(X) \rangle$, for $n \in [2]$, then we have $E_1(z,X) = -2 z_1+z_2$ and $ E_2(z,X) = 10 z_1^2 -14 z_1 z_2+5 z_2^2$. We now can use Girard's formula; see  \Cref{lem:psi_phi} and equation \eqref{eq:girard}):
        \[
            a_1(z;X) = E_1(z,X) , a_2(z;X) = \frac{1}{2} \mathrm{det} \begin{pmatrix}
                E_1(z,X) & 1 \\
                E_2(z,X) & E_1(z,X),
            \end{pmatrix}.
        \]
        to compute the coefficients of the multivariate polynomial as $a_1(z;X) = -2 z_1+z_2$ and $a_2(z;X) = -3 z_1^2-2 z_2^2 + 5 z_1 z_2$. We then have the following multivariate polynomial:
        \[
            p\big(t; z, \Phi(X)  =t^2 +(2 z_1-z_2) t -3 z_1^2-2 z_2^2 + 5 z_1 z_2 .
        \]
        To compute the roots of $p\big(t; z, \Phi(X)$, we use the quadratic formula. The discriminant is given as follows:
        \[
            \Delta(z, X) = a_1(z;X)^2 - 4 a_2(z;X) = 16 z_1^2+ 9 z_2^2 - 24 z_1 z_2 = (4 z_1 - 3 z_2 )^2.
        \]
        The parametric roots are $r_1(z, X)= \frac{1}{2}(-a_1(z;X) + \sqrt{\Delta(z, X)}) = z_1 - z_2$ and $r_1(z, X)= \frac{1}{2}(-a_1(z;X) - \sqrt{\Delta(z, X)}) = -3 z_1 +2 z_2$, that is, $\mathrm{roots} \circ p\big(t; z, \Phi(X) = z^{\top} X= \{ \{ z_1 - z_2,-3 z_1 +2 z_2 \} \}$, for all $z \in \R^2$. Since  $\mathrm{unique}(z^{\top} X) = 2$ --- $\forall z \in \R^2 \setminus \{ z \in \R^2: z_1 - z_2 = -3 z_1 +2 z_2\} $ ---- then we have $\mathrm{separators} \circ p\big(t; z, \Phi(X) = \{ z \in \R^2: z_1  \neq \frac{3}{4} z_2 \}$. Let $z^* = (1,1)^{\top} \in \R^2$ be a separator vector. Therefore, we have $\mathrm{sort}(z^{\top}X)|_{z = z^*} = (0,-1)^{\top}$.
        \begin{align*}
            \mathrm{sort}(z+\delta e_1)^{\top}X|_{z = z^*} = (\delta,-1-3 \delta)^{\top}, \ \ \mathrm{sort}(z+\delta e_2)^{\top}X|_{z = z^*} = (-\delta,-1 + 2\delta)^{\top}.
        \end{align*}
        for all $0< \delta < \frac{1}{3}$. Now we can compute the directional derivatives in \Cref{prop:zXtoX} as follows:
        \[
            (e_1^{\top} x_{\pi_{z^*}}(n))_{n \in [2]} = (1,-3)^{\top},  \ (e_2^{\top} x_{\pi_{z^*}}(n))_{n \in [2]} = (-1,2)^{\top},
        \]
        see equation \eqref{eq:x_pi_d}. Finally, we arrive at $X = \Phi^{-1} \circ \Phi(X) = \{ \{ (1,-1) , (-3,2) \} \}$.
    \end{example}

\newpage

\section{Proof of \Cref{thm:multi_decomposition_continous}} \label{sec:multi_decomposition_continous}
The function $f:\mathbb{X}_{\mathbb{D},N} \rightarrow \codom(f)$ is continuous over its domain, that is, $\rho \circ \Phi$ is continuous over $\mathbb{X}_{\mathbb{D},N}$, and we have $\rho  = f \circ \Phi^{-1}$; see \Cref{thm:multi_decomposition} and its proof for the definition of $\Phi$ and its inverse. Before proceeding with the proof, let us introduce the following set.
\begin{definition}\label{def:phi_x_n_b}
    For any mutiset function $\Phi$, we let $\Phi(\mathbb{X}_{\mathbb{D},N} ) \stackrel{\mathrm{def}}{=} \{ \Phi(X):  X \in \mathbb{X}_{\mathbb{D},N}\}$.
\end{definition}
With the notation in \Cref{def:phi_x_n_b}, $\rho= f \circ \Phi^{-1}$ is a map from $\Phi(\mathbb{X}_{\mathbb{D},N} ) $ to $\codom(f)$. Using \Cref{lem:continuous_Phi,lem:compact_set_of_sets} and Fact \ref{fact:cont_compact}, we show first that $\Phi(\mathbb{X}_{\mathbb{D},N} )$ is a compact set.
\begin{lemma}\label{lem:continuous_Phi}
    $\Phi:\mathbb{X}_{\mathbb{D},N} \rightarrow \Phi(\mathbb{X}_{\mathbb{D},N} )$ is a continuous and injective function.
\end{lemma}

\begin{lemma} \label{lem:compact_set_of_sets}
    $\mathbb{X}_{\mathbb{D},N}$ is a compact set.
\end{lemma}
\begin{fact}(\citealt{pugh2002real})\label{fact:cont_compact}
    The image of a compact set under continuous map is a compact set. 
\end{fact}
In \Cref{prop:phi_inv_continuous}, we prove that $\Phi^{-1}$ is a continuous function over the compact set $\Phi(\mathbb{X}_{\mathbb{D},N} )$. 
\begin{proposition}\label{prop:phi_inv_continuous}
    The function $\Phi^{-1}$ is continuous on the compact set $\Phi(\mathbb{X}_{\mathbb{D},N} )$. 
\end{proposition}
As a direct result of \Cref{prop:phi_inv_continuous}, $\rho  = f \circ \Phi^{-1}$ is a continuous function on the compact subset $\Phi(\mathbb{X}_{\mathbb{D},N}) \subset \mathbb{R}^{ {N+ D\choose D}-1}$. 
\begin{fact}
\label{fact:rho_hat}
    Since $\Phi(\mathbb{X}_{\mathbb{D},N} )$ is a compact subset of $\mathbb{R}^{ {N+ D\choose D}-1}$, the continuous function $\rho: \Phi(\mathbb{X}_{\mathbb{D},N}) \rightarrow \mathrm{codom}(f)$ has a continuous extension to $\mathbb{R}^{ {N+ D\choose D}-1}$, that is, there exists a continuous function $\rho_{\mathrm{e}}:\mathbb{R}^{ {N+ D\choose D}-1} \rightarrow 
    \mathrm{codom}(\rho_{\mathrm{e}})$ where
    \[
        \forall u \in \Phi(\mathbb{X}_{\mathbb{D},N} ): \rho_{\mathrm{e}}(u) = \rho(u),
    \]
    and $\mathrm{codom}(f) \subseteq \mathrm{codom}(\rho_{\mathrm{e}})$. To see the continuous extension theorem, refer to  \citep{deimling2010nonlinear}.
\end{fact}
From Fact~\ref{fact:rho_hat}, there exists a continuous function $\rho_{\mathrm{e}}:\mathbb{R}^{ {N+ D\choose D}-1} \rightarrow 
    \mathrm{codom}(\rho_{\mathrm{e}})$ where $ f(X) = \rho_{\mathrm{e}} \circ \Phi(X)$ for all $X \in \mathbb{X}_{\mathbb{D},N}$. Finally, if we rename $\rho_{\mathrm{e}}$ to $\rho$, we arrive at the theorem's statement.

\subsection{Proof of \Cref{lem:continuous_Phi}}
    As a direct result of \Cref{thm:multi_decomposition}, $\Phi$ is an injective function as it is invertible over its domain.  The continuity of $\Phi$ follows from the continuity of $\phi$ --- see \Cref{lem:continuous_Phi2}.
    \begin{lemma}\label{lem:continuous_Phi2}
        Let $\phi: \mathbb{D} \rightarrow \mathrm{codom}(\phi) \subset \mathbb{R}^K$ be a continuous function on metric space $(\mathbb{D},d)$ and $\Phi:\mathbb{X}_{\mathbb{D},N} \rightarrow \mathrm{codom}(\Phi) \subset \mathbb{R}^K$, $\Phi(X) = \sum_{x \in X} \phi(x)$ for $K, N \in \mathbb{N}$. Then, $\Phi$ is a continuous multiset function on $\mathbb{X}_{\mathbb{D},N}$. The same result is also valid on domain $\mathbb{X}_{\mathbb{D},[N]} $.
    \end{lemma}
    \begin{proof}
        We use the following the notion of distance between multisets with elements in $\mathbb{D}$:
        \begin{equation}\label{eq:dm}
            d_{M}(X, X^{\prime}) = \begin{cases}
            \min_{\pi \in \Pi(N_\circ)} \sqrt{ \sum_{n \in [N_\circ]} d( x_{n} , x^{\prime}_{\pi(n)} )^2} \ &\mbox{if} \ |X| = |X^{\prime}| = N_\circ \\
            \infty  \ &\mbox{otherwise},
            \end{cases}
        \end{equation}
        where $N_\circ \in [N]$, $|\cdot|$ returns the cardinality of its input multiset, $\Pi(N_\circ)$ is the set of permutation operators on $[N_\circ]$, $X = \{ \{ x_{n} : n \in [|X|] \} \} $, and $X^{\prime} = \{ \{ x^{\prime}_{n}: n \in [|X^{\prime}|] \} \}$. 
        
        Following the definition of continuity, for any $\varepsilon >0 $, we want to find a $\delta(\varepsilon)$ such that if $d_{M}(X,X^{\prime}) < \delta(\varepsilon)$, then $\| \Phi (X)-\Phi(X^{\prime})\|_2 < \varepsilon$. 
        
        For any $\delta >0$ and $X \in \mathbb{X}_{\mathbb{D},[N]}$, let $X^{\prime}  \in \mathbb{X}_{\mathbb{D},[N]}$ be such that $d_{M}(X,X^{\prime}) < \delta$, that is, both multisets have the same size of $|X| = |X^{\prime}| = N_\circ \in [N]$ and there is a permutation operator $\pi: [N_\circ] \rightarrow [N_\circ]$ such that $d_{M}(X,X^{\prime}) = \sqrt{ \sum_{n \in [N_\circ]} d( x_{n} , x^{\prime}_{\pi(n)} )^2} < \delta$. It suffices to show the following:
        \begin{align*}
            \| \Phi (X)-\Phi (X^{\prime})\|_2  &= \| \sum_{x \in X} \phi ( x ) - \sum_{x^{\prime} \in X^{\prime}} \phi ( x^{\prime} )\|_2 \stackrel{\text{(a)}}{\leq} \sum_{n \in [N_\circ]} \|  \phi ( x_n ) - \phi ( x^{\prime}_{\pi(n)} )\|_2 \\
            &\stackrel{\text{(b)}}{\leq}  \sum_{n \in [N_\circ]} \max_{v \in \mathbb{R}^D: \|v \|_2 < \delta } \|  \phi ( x_n ) - \phi (x_n + v )\|_2 < \varepsilon,
        \end{align*}
        where $\text{(a)}$ is due to the triangle inequality, $\text{(b)}$ is due to $\| x_{n} - x^{\prime}_{\pi(n)}\|_2 \leq \delta$, for all $ n \in [N_\circ]$. \\
        It suffices to show that for any $\varepsilon >0 $, there exits a $\delta(\varepsilon)$ such that
        \[
            \forall n \in [N],v \in \mathbb{R}^D, \|v \|_2 < \delta(\varepsilon):  \|  \phi ( x_n ) - \phi (x_n + v )\|_2 < N^{-1} \varepsilon < N_{\circ}^{-1} \varepsilon.
        \]
        Since $\phi$ is a continuous function, there exists a $\delta_{\phi}(x_n, N^{-1} \varepsilon) > 0$ such that 
        \[
            \forall v \in \mathbb{R}^D, \|v \|_2 < \delta_\circ(x_n, N^{-1} \varepsilon) :  \|  \phi ( x_n ) - \phi (x_n + v )\|_2 < N^{-1} \varepsilon.
        \]
        If we let $\delta(\varepsilon) = \min_{n \in [N_\circ]} \delta_{\phi}(x_n, N^{-1} \varepsilon) > 0$, then we have $\| \Phi (X)-\Phi (X^{\prime})\|_2 \leq \varepsilon$. Therefore, $\Phi$ is a continuous function. The same result is also valid on domain $\mathbb{X}_{\mathbb{D},N} $.
    \end{proof}

\subsection{Proof of \Cref{lem:compact_set_of_sets}}
    Let $\mathrm{OC}(S)$ be the set of all open covers of a topological space $S$. 
    \begin{fact}(\citealt{engelking1989general})
        A topological space $S$ is compact if any open cover of $S$ has a finite subcover.
    \end{fact}
    \begin{definition}
    We define the following maps between subsets of $\mathbb{X}_{\mathbb{D},N}$ and $\mathbb{D} \subseteq \mathbb{R}^D$.
    \begin{itemize}
        \item Let $\mathbb{U} \subseteq \mathbb{D}^{N}$ and $T=[x_1, \ldots, x_N] \in \mathbb{U}$. Then, we let $\mathrm{set}( T ) \stackrel{\mathrm{def}}{=}  \{ \{ x_n:n \in [N]\}\} \in \mathbb{X}_{\mathbb{D},N}$ and $\mathrm{set}( \mathbb{U} ) \stackrel{\mathrm{def}}{=} \{\mathrm{set}( T ) : T \in \mathbb{U} \} \subseteq \mathbb{X}_{\mathbb{D},N}$
        \item Let $\mathbb{V}  \subseteq \mathbb{X}_{\mathbb{D},N}$ and $X= \{ \{x_n: n \in [N] \}\} \in \mathbb{V}$. Then, we let $\mathrm{mat}( X ) \stackrel{\mathrm{def}}{=}  \{  [x_{\pi(1)}, \ldots, x_{\pi(N)}]:  \pi \in \Pi(N) \}  \subseteq \mathbb{R}^D$ and $\mathrm{mat}( \mathbb{V} ) \stackrel{\mathrm{def}}{=} \bigcup_{X \in \mathbb{V}} \mathrm{mat}( X )  \subseteq \mathbb{R}^D$,
    \end{itemize}
    where $\Pi(N)$ is the set of permutation operators $\pi: [N] \rightarrow [N]$ for $N \in \mathbb{N}$.
    \end{definition}    
    Given a matrix, the function $\mathrm{set}$ maps it to a multiset. In contrast, the function $\mathrm{mat}$ creates all possible matrices by rearranging elements of its input multiset.
    \begin{claim}\label{claim:1}
        If $\{ \mathbb{V}_{\lambda}: \lambda \in \Lambda \} \in \mathrm{OC}(\mathbb{X}_{\mathbb{D},N})$, then $\{ \mathrm{mat}(\mathbb{V}_{\lambda}): \lambda \in \Lambda \} \in \mathrm{OC}(\mathbb{D}^{N})$.
    \end{claim}
    \begin{claim}\label{claim:2}
        If $\{ \mathbb{U}_{\lambda}: \lambda \in \Lambda \} \in \mathrm{OC}(\mathbb{D}^{N})$, then $\{ \mathrm{set}( \mathbb{U}_{\lambda}): \lambda \in \Lambda \} \in \mathrm{OC}(\mathbb{X}_{\mathbb{D},N})$
    \end{claim}
    Let $\{ \mathbb{V}_\lambda : \lambda \in \Lambda \}$ be an open cover for $\mathbb{X}_{\mathbb{D},N}$. From Claim \ref{claim:1}, $\{\mathrm{mat}( \mathbb{V}_\lambda) : \lambda \in \Lambda \}$ is an open cover for $\mathbb{D}^N$ --- a closed and bounded subset of $\mathbb{R}^D$. Therefore, there is a finite subsequence $\{\mathrm{mat}( \mathbb{V}_{\lambda_k}) : k \in [K] \}$ that forms an open cover for $\mathbb{D}^N$. 
    From Claim \ref{claim:2}, $\{\mathrm{set}\circ \mathrm{mat}( \mathbb{V}_{\lambda_k}) : k \in [K] \} = \{ \mathbb{V}_{\lambda_k}: k \in [K] \}$, is a finite open cover for $\mathbb{X}_{\mathbb{D},N}$. Therefore, $\mathbb{X}_{\mathbb{D},N}$ is a compact set. \\
    {\bf Proof of Claim \ref{claim:1}} To prove $\{ \mathrm{mat}( \mathbb{V}_{\lambda}): \lambda \in \Lambda \}$ is an open cover for $ \mathbb{D}^{N}$, we first show that for all $T \in \mathbb{D}^{N} \subseteq \R^{N \times D}$, we have $T \in \mathrm{mat}( \mathbb{V}_{\lambda})$ for a $\lambda \in \Lambda$.  \\
    Let $T = [x_1, \ldots, x_N] \in \mathbb{D}^{N}$. Then, we have $\mathrm{set}(T) = \{ \{ x_n : n \in [N] \} \} \in \mathbb{V}_{\lambda} \subseteq \mathbb{X}_{\mathbb{D},N}$ for a $\lambda \in \Lambda$. Since the following holds true:
    \[
        \forall \pi \in \Pi(N):  [x_{\pi(1)}, \ldots, x_{\pi(N)}] \in \mathrm{mat}( \mathbb{V}_{\lambda} ),
    \]
    then, we have $T \in  \mathrm{mat}( \mathbb{V}_{\lambda})$. Therefore, $\{ \mathrm{mat}( \mathbb{V}_{\lambda}): \lambda \in \Lambda \}$ forms a cover for $\mathbb{D}^{N}$.  
    
    Next, we prove that $\mathrm{mat}( \mathbb{V}_{\lambda})$ is an open set. Let $T = [x_1, \ldots, x_N] \in \mathrm{mat}( \mathbb{V}_{\lambda})$, $\varepsilon >0$, and $\mathcal{N}(T,\varepsilon) = \{ T^{\prime} \in \R^{N \times D}: \| T - T^{\prime} \|_F \leq \varepsilon \}$. We want to show that for small enough $\varepsilon >0$, $\mathcal{N}(T,\varepsilon) \subseteq \mathrm{mat}( \mathbb{V}_{\lambda})$. 
    
    For all $T^{\prime} = [x^{\prime}_1, \ldots, x^{\prime}_N] \in \mathcal{N}(T,\varepsilon)$, we have
        \begin{align*}
            d_M(X,X^{\prime}) = \min_{\pi \in \Pi(N)} \sqrt{\sum_{n \in [N]} \| x_n - x^{\prime}_{\pi(n)} \|^2_2} \leq \| T - T^{\prime} \|_F \leq \varepsilon, \ \mbox{where} \ X^{\prime} = \{\{ x^{\prime}_n : n \in [N] \} \}.
        \end{align*}
        Since $\mathbb{V}_{\lambda}$ is an open set, for any $X \in \mathbb{V}_{\lambda}$, there exists $\varepsilon >0$ such that $X^{\prime} \in V_{\lambda}$ where $d_M(X, X^{\prime}) < \varepsilon$. Therefore, we have $T^{\prime} \in  \mathrm{mat}( \mathbb{V}_{\lambda})$. Since this is the case for all $T^{\prime} \in \mathcal{N}(T,\varepsilon)$, we have $\mathcal{N}(T,\varepsilon) \subseteq \mathrm{mat}( \mathbb{V}_{\lambda})$, that is, $\mathrm{mat}( \mathbb{V}_{\lambda})$ is an open set.

    {\bf Proof of Claim \ref{claim:2}} To prove $\{ \mathrm{set}( \mathbb{U}_{\lambda}): \lambda \in \Lambda \}$ is an open cover for $\mathbb{X}_{\mathbb{D},N}$, we first show that for all $X \in \mathbb{X}_{\mathbb{D},N}$, we have $X \in \mathrm{set}( \mathbb{U}_{\lambda})$ for a $\lambda \in \Lambda$.  
    
    Let $X = \{ \{ x_n : n \in [N] \} \}  \in \mathbb{X}_{\mathbb{D},N}$. Since $T_{\pi} = [x_{\pi(1)}, \ldots, x_{\pi(N)}] \in \mathbb{D}^{N}$ --- for all $\pi \in \Pi(N)$ --- we have $T_{\pi} \in \mathbb{U}_{\lambda}$ for a $\lambda \in \Lambda$. Therefore, we have $\mathrm{set}(T_\pi) = \{ \{ x_{\pi(n)}: n \in [N] \} \} = X \in \mathrm{set}(\mathbb{U}_{\lambda}) $. This proves that $\{\mathrm{set}(\mathbb{U}_{\lambda}): \lambda \in \Lambda \} $ is a cover for $\mathbb{X}_{\mathbb{D},N}$.  
    
    We now prove that $\mathrm{set}(\mathbb{U}_{\lambda})$ is an open set. Let $X =\{ \{ x_n: n \in [N] \} \} \in \mathrm{set}(\mathbb{U}_{\lambda})$, $\varepsilon >0$,  $\mathcal{N}(X,\varepsilon) = \{ X^{\prime} \in \mathbb{X}_{\mathbb{R}^D,N}: d_M(X,X^{\prime}) \leq \varepsilon \}$, and $T = [x_1, \ldots, x_N]$. We want to show that for small enough $\varepsilon >0$, $\mathcal{N}(X,\varepsilon) \subseteq \mathrm{set}( \mathbb{U}_{\lambda})$. 
    
    For all $X^{\prime} =\{ \{ x^{\prime}_n: n \in [N] \} \} \in \mathcal{N}(X,\varepsilon)$, we have
    \[
        \| T- T^{\prime}_{\pi}\|_F = d_M(X,X^{\prime}) \leq \varepsilon \ \mbox{where} \ T^{\prime}_{\pi} = [ x^{\prime}_{\pi (1)}, \ldots, x^{\prime}_{\pi(N)}  ],
    \]
    for a permutation operator $\pi : [N ] \rightarrow [N]$ that best match elements of $X$ and $X^{\prime}$. Since $\mathbb{U}_{\lambda}$ is an open subset, there exists  $\varepsilon >0$ such that $T_{\pi} \in \mathbb{U}_{\lambda}$. Therefore, we have $X^{\prime} = \mathrm{set}(T^{\prime}_{\pi}) \in \mathrm{set}( \mathbb{U}_\lambda)$. Since this is the case for all $X^{\prime} \in \mathcal{N}(X,\varepsilon)$, we have $\mathcal{N}(X,\varepsilon) \subseteq \mathrm{set}( \mathbb{U}_{\lambda})$, that is, $\mathrm{set}( \mathbb{U}_{\lambda})$ is an open set.
\subsection{Proof of \Cref{prop:phi_inv_continuous}}
By definition of continuity, we want to show that, for any $\varepsilon >0$ and $X \in \mathbb{X}_{\mathbb{D},N}$, there exists $\delta_f(\varepsilon) >0$ such that
\begin{align*}
    \forall X^{\prime} \in \mathbb{X}_{\mathbb{D},N}, \| \Phi(X) - \Phi(X^{\prime})\|_2 < \delta_f(\varepsilon)&: d_{M}( \Phi^{-1} \circ \Phi(X) - \Phi^{-1} \circ \Phi(X^{\prime}) ) < \varepsilon \\
    &: d_{M}(X  , X^{\prime} ) < \varepsilon,
\end{align*}
where $d_M$ is given in equation \eqref{eq:dm}. We use the result in \Cref{lem:param_root_cont} to establish the continuity of $\Phi^{-1}$ over $\Phi(\mathbb{X}_{\mathbb{D},N})$.
\begin{lemma} \label{lem:param_root_cont}
    Let $X \in \mathbb{X}_{\mathbb{D},N}$. The parameterized multiset that consists of the root of the polynomial $p\big(t; z, \Phi(X))$ in equation \eqref{eq:multinomial_polynomial_in_two_vs} (that is, $z^{\top}X$) varies continuously with $\Phi(X)$. More precisely, for all $\varepsilon >0$, there exists $\delta(\varepsilon )>0$ such that
    \begin{align*}
    \forall X^{\prime} \in \mathbb{X}_{\mathbb{D},N}, \| \Phi(X) - \Phi(X^{\prime} )\|_2 < \delta(\varepsilon) : \max_{z \in \mathbb{R}^D: \| z \|_2 = 1} d_M(z^{\top}X, z^{\top}X^{\prime}) < \varepsilon.
\end{align*}
\end{lemma}
Let $\varepsilon > 0$, $X = \{ \{ x_n : n \in [N] \} \}$ and $X^{\prime} = \{ \{ x^{]\prime}_n : n \in [N] \} \} \in \mathbb{X}_{\mathbb{D},N}$. From \Cref{lem:param_root_cont}, if $\| \Phi(X) - \Phi(X^{\prime})\|_2 < \delta(\varepsilon)$, then we have
\[
   \forall z \in \mathbb{R}^D, \| z \|_2 = 1: d_M(z^{\top}X, z^{\top}X^{\prime}) = \sqrt{ \sum_{n \in [N]} |z^{\top}x_n - z^{\top} x^{\prime}_{\pi^*(n)}|^2 } < \varepsilon
\]
for a permutation operator $\pi^*: [N] \rightarrow [N]$. Then, we have
\[
    \forall z \in \mathbb{R}^D, \| z \|_2 = 1, n \in [N]:  |z^{\top}x_n - z^{\top} x^{\prime}_{\pi^*(n)}|^2  < \varepsilon.
\]
If $x_n -x^{\prime}_{\pi^*(n)} \neq 0$ and $z = \| x_n - x^{\prime}_{\pi^*(n)}\|_2^{-1}(x_n -x^{\prime}_{\pi^*(n)}) $, then we have arrive at $ \|  x_n - x^{\prime}_{\pi^*(n)} \|_2  < \varepsilon$, where $n \in [N]$. If $x_n -x^{\prime}_{\pi^*(n)} = 0$, then $\|  x_n - x^{\prime}_{\pi^*(n)} \|_2  < \varepsilon$ is trivially the case. Therefore, we have
\[
    d_{M}(X, X^{\prime}) = \min_{\pi \in \Pi(N)}\sqrt{\sum_{n \in [N]} \| x_n - x^{\prime}_{\pi(n)} \|_2^2} \leq \sqrt{N} \varepsilon,
\]
where $\Pi(N)$ is the set of permutation operators on $[N]$. Finally, we establish the continuity of $\Phi^{-1}$ on $\Phi(\mathbb{X}_{\mathbb{D},N})$ by letting $\delta_f(\varepsilon) = \delta(\frac{\varepsilon}{\sqrt{N}})$, that is,
\begin{align*}
    \forall X^{\prime} \in \mathbb{X}_{\mathbb{D},N}, \| \Phi(X) - \Phi(X^{\prime})\|_2 < \delta_f(\varepsilon)  &: \max_{z \in \mathbb{R}^D: \| z \|_2 = 1} d_M(z^{\top}X, z^{\top}X^{\prime}) < \frac{\varepsilon}{\sqrt{N}} \\
    &: d_{M}(X, X^{\prime}) < \varepsilon.
\end{align*}
{\bf Proof of~\Cref{lem:param_root_cont}.}
    We construct the the polynomial $p\big(t; z, \Phi(X))$ in equation \eqref{eq:multinomial_polynomial_in_two_vs}, that is,
    \begin{align*}
        \forall t \in \R,z \in \mathbb{R}^D: p\big(t; z, \Phi(X)) = t^N + \sum_{n \in [N]} (-1)^n a_n(z; X) t^{N-n}  
    \end{align*}
    by first computing the following parameterized moments:
    \[
        \forall n \in [N], z \in \mathbb{R}^D: E_n(z,X) = \langle \psi(z,n), \Phi(X) \rangle.
    \]
    \begin{fact}\label{fact:1}
        For a fixed $z \in \mathbb{R}^D$ and $n \in [N]$, $E_n(z,X)$ is a linear function of $\Phi(X)$. Furthermore, $E_n(z,X)$ is a continuous functions of $(z, \Phi(X))$.
    \end{fact}
    The coefficients of $p\big(t; z, \Phi(X))$ are polynomial functions of the moments $\big( E_n(z,X) \big)_{n \in [N]}$; see the Newton-Girard equation \eqref{eq:girard}. 
    \begin{fact}\label{fact:2}
        The coefficients of the polynomial $p\big(t; z, \Phi(X))$ in equation \eqref{eq:multinomial_polynomial_in_two_vs} vary continuously with the moments $\big( E_n(z,X) \big)_{n \in [N]}$.
    \end{fact}
    Therefore, the coefficients of the polynomial $p\big(t; z, \Phi(X))$ in equation \eqref{eq:multinomial_polynomial_in_two_vs} vary continuously with $(z,\Phi(X))$; see Facts \ref{fact:1} and \ref{fact:2}.    
    \begin{theorem}\citep{curgus2006roots}
        The function $f : \mathbb{C}^N \rightarrow \mathbb{C}^N$, which associates every $a = (a_n)_{n \in [N]} \in \mathbb{C}^N$ to the multiset of roots,  $f(a) \in \mathbb{C}^N$, of the monic polynomial formed using a as the coefficient, i.e., $t^N + a_1 t^{N-1} +\cdots +(-1)^{N-1}a_{N-1} x + (-1)^{N} a_N$, is a homeomorphism. \label{thm:root_homeomorphism}
    \end{theorem}
    From \Cref{thm:root_homeomorphism}, Facts \ref{fact:1}, and \ref{fact:2}, the parameterized root multiset of $p\big(t; z, \Phi(X))$ (that is, $z^{\top} X$) vary continuously with $(z, \Phi(X))$. Therefore, for all $X \in \mathbb{X}_{\mathbb{D},N}$, $z \in \mathbb{R}^D$ and $ \varepsilon >0$, there exists $\delta(\varepsilon,z) > 0$ such 
    \[
        \forall X^{\prime}\in \mathbb{X}_{\mathbb{D},N}, \| \Phi(X) - \Phi(X^{\prime} )\|_2  < \delta(\varepsilon,z)  : d_{M}(z^{\top} X, z^{\top} X^{\prime}) < \varepsilon.
    \]
    We may fix the norm of the vector $z$ to one, since by definition of $d_M$ in equation \eqref{eq:dm}, we have
    \[
        \forall \alpha \in \mathbb{R}: d_{M}( (\alpha z)^{\top} X, (\alpha z)^{\top}X^{\prime}) = |\alpha| d_{M}(z^{\top} X, z^{\top}X^{\prime}.
    \]
    After this normalization, for all $\varepsilon >0$, we have
    \[
        \forall X^{\prime}\in \mathbb{X}_{\mathbb{D},N}, z\in \mathbb{R}^D, \| z \|_2 = 1, \| \Phi(X) - \Phi(X^{\prime} )\|_2  < \delta(\varepsilon,z)  : d_{M}(z^{\top} X, z^{\top}X^{\prime}) < \varepsilon.
    \]
    Let $z^* \in \mathrm{argmax}_{z\in \mathbb{R}^D: \| z \|_2 = 1} d_{M}(z^{\top} X, z^{\top} X^{\prime})$. Then, we have
    \[
        \forall X^{\prime} \in \mathbb{X}_{\mathbb{D},N}, \| \Phi(X) - \Phi(X^{\prime} )\|_2  < \delta(\varepsilon,z^*)  :  \mathrm{max}_{z\in \mathbb{R}^D: \| z \|_2 = 1} d_{M}(z^{\top} X, z^{\top} X^{\prime}) < \varepsilon,
    \]
    which proves the statement if $z^*$ exists. Therefore, we need to prove the existence of $z^*$.  \\
    The set $\{ z\in \mathbb{R}^D: \| z \|_2 = 1 \}$ is compact. If we prove that $d_{M}(z^{\top} X, z^{\top} X^{\prime})$ is a continuous function of $z$, then by the extreme-value theorem \citep{stein2010complex}, $z^*$ does exist. To this end, we show that $d^2_{M}(z^{\top} X, z^{\top} X^{\prime})$ (and hence $d_{M}(z^{\top} X, z^{\top} X^{\prime})$) is continuous. We use the following first-order perturbation analysis:
    \begin{align*}
        d^2_{M}((z+\mathrm{dz})^{\top} X, (z+\mathrm{dz})^{\top} X^{\prime}) = \sum_{n \in [N]} |(z+\mathrm{dz})^{\top}x_n - (z+\mathrm{dz})^{\top}x^{\prime}_{\pi_{z+\mathrm{dz}}(n)} |^2  
    \end{align*}
    where $\pi_{z+\mathrm{dz}}:[N] \rightarrow [N]$ is a permutation operator that best matches elements of perturbed multisets $(z+\mathrm{dz})^{\top} X$ and $ (z+\mathrm{dz})^{\top} X^{\prime}$. Let $X^{\prime \prime} = \{ \{ x_n - x^{\prime}_{{\pi_z}(n)}: n \in [N] \} \}$. As we discussed in the proof of \Cref{prop:zXtoX}, if $\| \mathrm{dz} \|_2 < \frac{\mathrm{gap}(z^{\top}X^{\prime \prime})}{\mathrm{diam}(\mathbb{D})}$ --- $\mathrm{gap}(z^{\top}X^{\prime \prime}) \neq 0$ since $X \neq X^{\prime}$ ---- then $x^{\prime}_{{\pi_z}(n) }= x^{\prime}_{{\pi_{z+\mathrm{dz}}}(n)}$ for all $n \in [N]$. Therefore, we have
    \begin{align*}
        d^2_{M}((z+\mathrm{dz})^{\top} X, (z+\mathrm{dz})^{\top} X^{\prime}) = d^2_{M}(z^{\top} X, z^{\top} X^{\prime}) + O(\| \mathrm{dz}\|_2^2),
    \end{align*}
    that is, $d_{M}(z^{\top} X, z^{\top} X^{\prime})$ is a continuous function of $z$. This concludes the proof. 

\newpage

\section{Proof of \Cref{thm:continuous_decomp_4}}
\subsection{Extension of \Cref{thm:multi_decomposition}}
    Let $\mathbb{D}$ be a compact subset of $\mathbb{R}^D$, that is, compact $\mathbb{D} \neq \mathbb{R}^D$. The encoding function $\Phi(X) = \sum_{x \in X} \phi(x)$ --- where $\phi: \mathbb{D} \rightarrow \mathrm{codom}(\phi)$ --- is an injective map over multisets with exactly $N$ elements, that is, $\Phi^{-1} \circ \Phi(X) = X$ where $X \in \mathbb{X}_{\mathbb{D},N}$. To extend the result to multisets of variable sizes, we follow the proof sketch for the one-dimensional case \citep{wagstaff2019limitations}. Let $x_{\circ} \in \mathbb{R}^{D} \setminus \mathbb{D}$. Then, we define $\phi^{\prime}(x) = \phi(x) - \phi(x_{\circ})$. For a multiset $X \in \mathbb{X}_{\mathbb{D},[N]}$ with $|X| \leq N$ elements, we have
    \begin{align*}
        \forall X \in \mathbb{X}_{\mathbb{D},[N]}: \Phi^{\prime}(X) &= \sum_{x \in X} \phi^{\prime}(x) = \sum_{x \in X}  \phi(x) - |X| \phi(x_{\circ} ) \\
        &= \Phi( X \cup \{ \{ \underbrace{x_{\circ} , \ldots, x_{\circ} }_{N-|X|} \} \} ) - N \phi(x_{\circ} ) \\
        &= \Phi( X \cup \{ \{ \underbrace{x_{\circ} , \ldots, x_{\circ} }_{N-|X|} \} \} )+\mathrm{const}
    \end{align*}
    where $\mathrm{const} = -N \phi(x_\circ) $. Since $\Phi$ is injective over $\mathbb{X}_{\mathbb{D},N}$, $\Phi^{\prime}$ is an injective map. That is to say,
    \begin{align*}
         \forall X \in \mathbb{X}_{\mathbb{D},[N]}: \Big( \Phi^{-1} \circ (\Phi^{\prime} (X)  - \mathrm{const} \big) \Big) \cap \mathbb{D}  &= \Big( \Phi^{-1} \circ  \Phi ( X \cup \{ \{ \underbrace{x_{\circ} , \ldots, x_{\circ} }_{N-|X|} \} \} )  \Big) \cap \mathbb{D}   \\
         &= ( X \cup \{ \{ \underbrace{x_{\circ} , \ldots, x_{\circ} }_{N-|X|} \} \} )  \cap \mathbb{D} \\
         &= X.
    \end{align*}
    Therefore, we have ${\Phi^{\prime}}^{-1}(U) = \Phi^{-1} \big( U  - \mathrm{const} \big) \cap \mathbb{D}$ for all $U \in \Phi^{\prime}(\mathbb{X}_{\mathbb{D},[N]}) = \{ \Phi^{\prime}(X): X \in \mathbb{X}_{\mathbb{D},[N]} \}$. If we define $\rho = f \circ (\Phi^{\prime})^{-1}$ where $\mathrm{dom}(\rho) = \Phi^{\prime}(\mathbb{X}_{\mathbb{D},[N]})$, then we have $f(X) = \rho \circ \Phi^{\prime}(X)$ for all $X \in  \mathbb{X}_{\mathbb{D},[N]}$. We arrive at the theorem's exact statement by renaming $\Phi^{\prime}$ to $\Phi$.
\subsection{Extension of \Cref{thm:multi_decomposition_continous}}
    Let $\mathbb{D}$ be a compact subset of $\mathbb{R}^D$, that is, compact $\mathbb{D} \neq \mathbb{R}^D$. In \Cref{lem:compact_set_of_sets}, we prove that $\mathbb{X}_{\mathbb{D},n}$ is a compact set, for all $n \in \mathbb{N}$. Since $\mathbb{X}_{\mathbb{D},[N]}$ is a finite union of compact sets, that is, $\mathbb{X}_{\mathbb{D},[N]} = \bigcup_{n = 1}^{N} \mathbb{X}_{\mathbb{D},n}$, itself is a compact set~\citep{sutherland2009introduction}. Since $\Phi^{\prime}$ is a continuous map (see \Cref{lem:continuous_Phi2}), $\Phi^{\prime}(\mathbb{X}_{\mathbb{D},[N]})$ is also a compact set \citep{pugh2002real}. 
    
    Now let us show that ${\Phi^{\prime}}^{-1}$ is a continuous map over compact set $\mathrm{codom}(\Phi^{\prime})=\Phi^{\prime}(\mathbb{X}_{\mathbb{D},[N]})$. We have to show that for all $\varepsilon >0$ and all $X, X^{\prime} \in \mathbb{X}_{\mathbb{D},[N]}$ such that $\| \Phi^{\prime}(X) - \Phi^{\prime}(X^{\prime})\|_2 < \delta(\varepsilon)$ we have $d_M( {\Phi^{\prime}}^{-1} \circ \Phi^{\prime}(X) ,  {\Phi^{\prime}}^{-1} \circ \Phi^{\prime}(X^{\prime}) ) < \varepsilon$ where $\delta(\varepsilon) >0$ and $d_M$ is the matching distance between multisets, that is,
    \[
        d_M(X, X^{\prime}) = \begin{cases}
        \min_{ \text{bijection} \ \pi : [N_{\circ}] \rightarrow [N_{\circ}] } \sqrt{\sum_{n \in [N_{\circ}]} \| x_n  - x^{\prime}_{\pi(n)}  \|_2^2} & \mbox{if} \ \ |X| = | X^{\prime} | = N_{\circ} \\
        \infty & \mbox{if}  \ \ |X| \neq | X^{\prime}|,
        \end{cases} 
    \]
    where $X = \{ \{ x_n: n \in [N_{\circ}] \} \}$, $X^{\prime} = \{ \{ x^{\prime}_n: n \in [N_{\circ}] \} \}$, $N_{\circ} \in [N]$. On the other hand, we have ${\Phi^{\prime}}^{-1}(U) = \Phi^{-1} \big( U  - \mathrm{const} \big) \cap \mathbb{D}$ for all $U \in \Phi^{\prime}(\mathbb{X}_{\mathbb{D},[N]})$ where $\Phi^{-1}$ is a continuous function; see \Cref{prop:phi_inv_continuous}.
    
    Consider the continuous function $\Psi(U) = \Phi^{-1} \big( U  - \mathrm{const} \big)$ where $U \in \Phi^{\prime}(\mathbb{X}_{\mathbb{D},[N]})$. By definition of continuity, for all $\varepsilon >0$ and all $X, X^{\prime} \in \mathbb{X}_{\mathbb{D},[N]}$ such that $\| \Phi^{\prime}(X) - \Phi^{\prime}(X^{\prime})\|_2 < \delta(\varepsilon)$ we have $d_M( \Psi \circ \Phi^{\prime}(X) , \Psi \circ \Phi^{\prime}(X^{\prime}) ) < \varepsilon$ where $\delta(\varepsilon) >0$. 
    Since we have,
    \begin{align*}
        \Psi \circ \Phi^{\prime}(X) &= X \cup \{ \{ \underbrace{x_{\circ} , \ldots, x_{\circ} }_{N-|X|} \} \}  \\
        \Psi \circ \Phi^{\prime}(X^{\prime}) &=  X^{\prime} \cup \{ \{ \underbrace{x_{\circ} , \ldots, x_{\circ} }_{N-|X^{\prime}|} \} \} ,
    \end{align*}
    we can simplify $d_M( \Psi \circ \Phi^{\prime}(X) , \Psi \circ \Phi^{\prime}(X^{\prime}) ) < \varepsilon$ as follows:
    \[
        d_M( X \cup \{ \{ \underbrace{x_{\circ} , \ldots, x_{\circ} }_{N-|X|} \} \} , X^{\prime} \cup \{ \{ \underbrace{x_{\circ} , \ldots, x_{\circ} }_{N-|X^{\prime}|} \} \} ) < \varepsilon.
    \]
    If $X$ and $X^{\prime}$ have different number of elements in $\mathbb{D}$, then we have $\varepsilon >  \inf_{x \in \mathbb{D}} \|  x-  x_\circ \|_2$.  Let $\varepsilon_\circ >0 $ be such that $\varepsilon_\circ < \inf_{x \in \mathbb{D}} \|  x-  x_\circ \|_2$. If we pick $0< \varepsilon < \varepsilon_\circ$, then $X$ and $X^{\prime}$  have the same number of elements in $\mathbb{D}$ and 
    \begin{align*}
        d_M( (\Phi^{\prime})^{-1} \circ \Phi^{\prime}(X)  , (\Phi^{\prime})^{-1} \circ \Phi^{\prime}(X^{\prime})) &= d_M( \Psi \circ \Phi^{\prime}(X) \cap \mathbb{D}  , \Psi \circ \Phi^{\prime}(X^{\prime})\cap \mathbb{D} ) \\
        &= d_M( \Psi \circ \Phi^{\prime}(X)  , \Psi \circ \Phi^{\prime}(X^{\prime}) ) < \varepsilon
    \end{align*}
     That is, $(\Phi^{\prime})^{-1}$ is a continuous function over $\Phi^{\prime}(\mathbb{X}_{\mathbb{D},[N]})$. Therefore, $\rho = f \circ (\Phi^\prime)^{-1}$ is a continuous function on compact set $\Phi^{\prime}(\mathbb{X}_{\mathbb{D},[N]}) \subset \mathbb{R}^{ {N+D \choose D}-1}$, and it has a continuous extension to $\mathbb{R}^{ {N+D \choose D}-1}$; refer to Fact \ref{fact:rho_hat}. We arrive at the theorem's statement by renaming $\Phi^{\prime}$ to $\Phi$.

 \newpage   
\section{Proof of \Cref{cor:two_multisets}}
    
Let $\Phi^{\prime}: \mathbb{X}_{\mathbb{D},[N]} \rightarrow \mathrm{codom}(\Phi^{\prime})$ where $N = \max \{N_1, N_2 \}$, $\Phi^{\prime}(X) = \sum_{x \in X} \phi^{\prime}(x)$, and $\phi^{\prime}$ is given in the proof of \Cref{thm:continuous_decomp_4}. The function $\Phi^{\prime}$ is injective on $\mathbb{X}_{\mathbb{D},[N]}$ and $(\Phi^{\prime})^{-1}$ is continuous on compact set $\Phi^{\prime}(\mathbb{X}_{\mathbb{D},[N]})$. Since $\mathbb{X}_{\mathbb{D},[N_1]}$ and $\mathbb{X}_{\mathbb{D},[N_2]}$ are compact subsets of $\mathbb{X}_{\mathbb{D},[N]} \subseteq \mathbb{R}^{{N+D \choose D}-1}$, the function $\Phi^{\prime}$ is injective on $\Phi^{\prime}(\mathbb{X}_{\mathbb{D},[N_1]}) $ and $\Phi^{\prime}(\mathbb{X}_{\mathbb{D},[N_2]}) $  
and $(\Phi^{\prime})^{-1}$ is continuous on both domains. We define the following function:
\[
    \forall U_1 \in \Phi^{\prime}(\mathbb{X}_{\mathbb{D},[N_1]}) , U_2 \in \Phi^{\prime}(\mathbb{X}_{\mathbb{D},[N_2]}) :
    \rho(U_1, U_2) = f \big(  (\Phi^{\prime})^{-1}(U_1), (\Phi^{\prime})^{-1}(U_2) \big).
\]

If $f$ is a continuous multiset function, $\rho$ (defined above) is a continuous function on its compact domain $\Phi^{\prime}(\mathbb{X}_{\mathbb{D},[N_1]}) \times \Phi^{\prime}(\mathbb{X}_{\mathbb{D},[N_2]})$ as it is the composition of continuous functions. Therefore, it has a continuous extension to $\mathbb{R}^{{N+D \choose D}-1} \times \mathbb{R}^{{N+D \choose D}-1}$; refer to Fact \ref{fact:rho_hat}.

\newpage
\section{Proof of \Cref{thm:perm_w_distinct_labels}}
    We define $\phi: \mathbb{R}^{D} \rightarrow \mathrm{codom}(\phi) \subset \mathbb{C}^{D\times N}$ as follows: 
    \begin{equation} \label{eq:phi_l}
         \forall x \in \mathbb{R}^{D}: \phi ( x ) = \begin{pmatrix}
            r(x) & r(x)^{\odot 2} & \cdots & r(x)^{\odot N}
        \end{pmatrix} \in \mathbb{C}^{D \times N},
    \end{equation}
    where $r(x)=  x+ 1 l(x)j$, $1 \in \mathbb{R}^{D}$ is a vector of all ones, $l:\mathbb{R}^D \rightarrow \R$ is a continuous function, $j=\sqrt{-1}$, and $\odot $ computes elementwise exponents. 
    \begin{fact} \label{fact:phi_contu}
        The function $\phi$ is continuous.
    \end{fact}
    \begin{lemma} \label{lem:Phi_inj}
    Let $\phi$ be the function defined in equation \eqref{eq:phi_l}. Then, the function $\Phi ( X ) = \sum_{x \in X } \phi ( x )$ is injective on $\mathbb{X}^{l}_{\mathbb{R}^D,N}$.
    \end{lemma}
    Let $\Phi( \mathbb{X}^{l}_{\mathbb{R}^D,N} ) \stackrel{\mathrm{def}}{=} \{ \Phi(X) : X \in \mathbb{X}^{l}_{\mathbb{R}^D,N} \}$. From \Cref{lem:Phi_inj}, there exists an inverse function $\Phi^{-1}: \Phi( \mathbb{X}^{l}_{\mathbb{R}^D,N} )  \rightarrow \mathbb{X}^{l}_{\mathbb{R}^D,N}$, that is, $\Phi^{-1} \circ \Phi(X) = X$ for all $X \in \mathbb{X}^{l}_{\mathbb{R}^D,N}$. We construct $\rho: \Phi(\mathbb{X}^{l}_{\mathbb{R}^D,N}) \rightarrow \codom(f)$ as $\rho = f \circ \Phi^{-1}$. This completes the proof as follows:
    \[
        \forall X \in \mathbb{X}^{l}_{\mathbb{R}^D,N}: \rho \circ \Phi(X) = f \circ \Phi^{-1} \circ \Phi(X) = f(X).
    \]
   
    \subsection{Proof of \Cref{lem:Phi_inj}}
    From equation \eqref{eq:phi_l}, we have
    \begin{equation}\label{eq:Phi_l}
        \forall X \in \mathbb{X}^{l}_{\mathbb{R}^D,N}: \ \Phi(X) = \sum_{x \in X} \begin{pmatrix} 
            r(x) & r(x)^{\odot 2}& \cdots & r(x)^{\odot N}
        \end{pmatrix}  \in \mathbb{C}^{D \times N}.
    \end{equation}
    \begin{definition}\label{def:deep_phi}
        Let $\Phi_{\mathrm{deep}}^{-1}$ be the continuous function introduced in Deep Sets paper \citep{zaheer2017deep}, viz., $\Phi_{\mathrm{deep}}^{-1} \circ \Phi_{\mathrm{deep}} = X $ where $\Phi_{\mathrm{deep}}(X) = ( \sum_{x \in X} x, \ldots, \sum_{x \in X} x^N)$, where $X \in \mathbb{X}_{\mathbb{C},N}$ is a multiset of $N$ scalars in $\mathbb{C}$. With slight abuse of notation, we generalize this definition to the following row-wise function:
        \[
            \forall X_1,\ldots, X_D \in \mathbb{X}_{\mathbb{C},N}: \ \Phi_{\mathrm{deep}}^{-1}( \begin{pmatrix}
                    \Phi_{\mathrm{deep}}(X_1) \\
                    \vdots \\
                    \Phi_{\mathrm{deep}}(X_D)
            \end{pmatrix}) = \begin{pmatrix}
                    \Phi_{\mathrm{deep}}^{-1} \circ \Phi_{\mathrm{deep}}(X_1) \\
                    \vdots \\
                    \Phi_{\mathrm{deep}}^{-1}\circ  \Phi_{\mathrm{deep}}(X_D)
            \end{pmatrix} = \begin{pmatrix}
                    X_1 \\
                    \vdots \\
                    X_D
            \end{pmatrix}
        \]
    \end{definition}
    \begin{definition}\label{def:sort}
        Let $X = \{ \{ x_n \in \mathbb{C}: n \in [N] \} \} \in \mathbb{X}_{\mathbb{C},N}$ be a multiset of $N$ complex-valued elements. We then define the function $\mathrm{sort}$ as follows:
        \[
        \mathrm{sort}( X ) = \big( \mathrm{Re}( x_{\pi(n)} ) \big)_{n \in [N]} \in \R^N,
        \]
        where $\pi: [N] \rightarrow [N]$ is any permutation operator such that $\mathrm{Im}(x_{\pi(1)}) \leq \cdots \leq \mathrm{Im}( x_{\pi(N)})$.
    \end{definition}
    \begin{definition}\label{def:sort_vec}
        Let $X_1,\ldots, X_D \in \mathbb{X}_{\mathbb{C},N}$ be multisets of $N$ complex-valued elements. We then define the function $\mathrm{sortvec}$ as follows:
            \[
            \mathrm{sortvec}( \begin{pmatrix}
                    X_1 \\
                    \vdots \\
                    X_D
            \end{pmatrix} ) = \{\{ \begin{pmatrix}
                    e_{n}^{\top}\mathrm{sort}( X_1 ) \\ \vdots \\ e_{n}^{\top}\mathrm{sort}( X_D )
                    \end{pmatrix} \in \mathbb{R}^D: n \in [N] \} \} \in \mathbb{X}_{\mathbb{R}^D,N},
        \]
        where $e_n \in \R^N$ is the $n$-th standard basis vector for $\R^N$.
    \end{definition}
    \begin{remark}
        Permutation operators in \Cref{def:sort,def:sort_vec} may not be unique. This happens if the input multiset has at least two elements with nonunique imaginary parts. If this is the case, functions $\mathrm{sort}$ and $\mathrm{sortvec}$ are both ill-defined. In what follows, we show that for certain inputs of interest both functions are indeed well-defined.
    \end{remark}

    Let $\Psi: \Phi(\mathbb{X}^{l}_{\mathbb{R}^D,N}) \rightarrow \mathbb{X}^{l}_{\mathbb{R}^D ,N}$ where $\Psi = \mathrm{sortvec} \circ \Phi_{\mathrm{deep}}^{-1}$. Then, we have
    \begin{align*}
        \forall X \in \mathbb{X}^{l}_{\mathbb{R}^D,N}: &\Psi \circ \Phi(X) = \mathrm{sortvec} \circ \Phi_{\mathrm{deep}}^{-1} \circ \Phi(X) \\
        &\stackrel{\text{(a)}}{=}  \mathrm{sortvec} \circ \Phi_{\mathrm{deep}}^{-1} \big( \sum_{x \in X} \begin{pmatrix}
            (x+1l(x) j ) &  (x+1l(x) j )^{\odot 2} & \cdots & (x+1l(x) j )^{\odot N} 
        \end{pmatrix} \big) \\
        &\stackrel{\text{(b)}}{=} \mathrm{sortvec}\circ \Phi_{\mathrm{deep}}^{-1} \big(
        \begin{pmatrix}
        \sum_{x \in X} e_{1}^{\top}x+l(x) j & \cdots & \sum_{x \in X} (e_{1}^{\top}x+l(x) j)^{N} \\
                    \vdots \\
        \sum_{x \in X} e_{D}^{\top}x+l(x) j & \cdots & \sum_{x \in X} (e_{D}^{\top}x+l(x) j)^{N}
        \end{pmatrix}   \big) \\
        &\stackrel{\text{(c)}}{=} \mathrm{sortvec}\circ \Phi_{\mathrm{deep}}^{-1} \big(
        \begin{pmatrix}
        \Phi_{\mathrm{deep}}  ( \{ \{ e_{1}^{\top}x+l(x) j: x\in X \} \} ) \\
                    \vdots \\
        \Phi_{\mathrm{deep}} ( \{ \{ e_{D}^{\top}x+l(x) j : x\in X  \} \} ) 
        \end{pmatrix}   \big) \\
        &\stackrel{\text{(d)}}{=} \mathrm{sortvec} \big( \begin{pmatrix}
        \{\{ e_1^{\top}x+l(x)j : x \in X \}\} \\
                    \vdots \\
        \{\{ e_D^{\top}x+l(x)j : x \in X \}\}
        \end{pmatrix}   \big).
    \end{align*}
    where $\text{(a)}$ is due to equations \eqref{eq:Phi_l} and \eqref{eq:phi_l}, $\text{(b)}$ follows from  explicitly writing the elements of $\Phi(X)$, $\text{(c)}$ follows from the definition of $\Phi_{\mathrm{deep}}$ (see \Cref{def:deep_phi}), and finally $\text{(d)}$ is due to the fact that we allow $\Phi_{\mathrm{deep}}^{-1}$ to operate elementwise.
    
    {\bf Case 1 (Distinct Identifiers).} Let $X = \{ \{ x_n : n \in [N] \} \}$. If all elements of $l(X) = \{ \{ l(x): x \in X \}\}$ are unique, then we have
    \[
        \forall d \in [D]: \mathrm{sort}(\{\{ e_d^{\top}x+l(x)j : x \in X \}\} ) =  \big(  e_d^{\top}x_{\pi(n)}  \big)_{n \in [N]} \in \R^N
    \]
    where $\pi: [N] \rightarrow [N]$ is the permutation operator such that $l(x_{\pi(1)}) < \cdots < l( x_{\pi(N)})$. Then, we have
    \[
        \Psi \circ \Phi(X) =  \{\{ \begin{pmatrix}
         e_1^{\top}x_{\pi(n)}  \\
                    \vdots \\
         e_D^{\top}x_{\pi(n)} 
        \end{pmatrix} : n \in [N]\}\} =  \{\{x_{\pi(n)}  : n \in [N]\}\}= X.
    \]
    {\bf Case 2 (Repeated Identifiers).}
    If $l(X)$ has repeated elements, then there exists at least two distinct permutation operators $\pi$ and $\pi^{\prime}$ ($\pi \neq \pi^{\prime}$) that sort the elements of $l(X)$, that is,
    \begin{align*}
        l(x_{\pi(1)}) &\leq l(x_{\pi(2)}) \leq  \cdots \leq l(x_{\pi(N)}) \\
        l(x_{\pi^\prime(1)}) &\leq l(x_{\pi^{\prime}(2)}) \leq  \cdots \leq l(x_{\pi^{\prime}(N)}).
    \end{align*}
    In this case, we have $l(x_{\pi(n)}) = l(x_{\pi^\prime(n)})$ for all $n \in [N]$  --- even though $\pi(n) \neq \pi^\prime(n)$ for some $n \in [N]$. From \Cref{def:identifiable_sets}, since $l(x_{\pi(n)}) = l(x_{\pi^\prime(n)})$, we have $x_{\pi^\prime(n)} = x_{\pi(n)}$ where $n \in [N]$. Consequently, we have
    \begin{align*}
        \forall d \in [D]: \mathrm{sort}(\{\{ e_d^{\top}x+l(x)j : x \in X \}\} ) &=  \big(  e_d^{\top}x_{\pi(n)}  \big)_{n \in [N]} \\
        &=  \big(  e_d^{\top}x_{\pi^{\prime}(n)}  \big)_{n \in [N]} \in \R^N.
    \end{align*}
    Therefore, even though there are multiple permutation operators that $\mathrm{sort}$s the elements of $\{\{ e_d^{\top}x+l(x)j : x \in X \}\}$, the output of the $\mathrm{sort}$ function remains unchanged, that is, $\mathrm{sort}$ is a well-defined function for any element of $X \in \mathbb{X}^{l}_{\mathbb{R}^D,N}$. Consequently, $\mathrm{sortvec}$ is well-defined on $\mathbb{X}^{l}_{\mathbb{R}^D,N}$ and we have
    \begin{align*}
        \forall X \in \mathbb{X}^{l}_{\mathbb{R}^D,N}: \Psi \circ \Phi(X) &=  \{\{ \begin{pmatrix}
         e_1^{\top}x_{\pi_1(n)}  \\
                    \vdots \\
         e_D^{\top}x_{\pi_D(n)} 
        \end{pmatrix} : n \in [N]\}\} \\
        &\stackrel{\text{(a)}}{=}  \{\{ \begin{pmatrix}
         e_1^{\top}x_{\pi_1(n)}  \\
                    \vdots \\
         e_D^{\top}x_{\pi_1(n)} 
        \end{pmatrix} : n \in [N]\}\} = X,
    \end{align*}
    where $\pi_d$ is a permutation operator that sorts the elements of $\{\{ e_d^{\top}x+l(x)j : x \in X \}\}$ --- for all $d \in [D]$ --- and $\text{(a)}$ is due to $x_{\pi_{i}(n)} = x_{\pi_{j}(n)}$ for all $i, j \in [D]$ and $n \in [N]$. Therefore, we have
    \[
        \forall X \in \mathbb{X}^{l}_{\mathbb{R}^D,N} = \Psi \circ \Phi(X)=\mathrm{sortvec}\circ \Phi_{\mathrm{deep}}^{-1} \circ \Phi(X) = X,
    \]
    that is, $\Psi = \mathrm{sortvec}\circ \Phi_{\mathrm{deep}}^{-1}$ is well-defined on $\Phi(\mathbb{X}^{l}_{\mathbb{R}^D,N})$ and $\Psi = \Phi^{-1}: \Phi(\mathbb{X}^{l}_{\mathbb{R}^D,N}) \rightarrow \mathbb{X}^{l}_{\mathbb{R}^D,N}$. This proves that $\Phi$ is an injective function on $\mathbb{X}^{l}_{\mathbb{R}^D,N}$.

\newpage
\section{Proof of \Cref{prop:perm_w_distinct_labels}}
    The proof is similar to that of \Cref{thm:continuous_decomp_4}. Let $\mathbb{D}$ be a compact subset of $\mathbb{R}^D$, that is, $\mathbb{D} \neq \mathbb{R}^D$. 
    
    The encoding function $\phi: \mathbb{D} \rightarrow \mathrm{codom}(\phi)$ (defined in the proof \Cref{thm:perm_w_distinct_labels}) such that $\Phi(X) = \sum_{x \in X} \phi(x)$ is an injective map over multisets with exactly $N$ elements, that is, $\Phi^{-1} \circ \Phi(X) = X$ where $X \in \mathbb{X}^{l}_{\mathbb{D},N}$ and $l: \mathbb{D} \rightarrow \mathbb{R}$ is the continuous identifier function. Let $x_{\circ} \in \mathbb{R}^{D} \setminus \mathbb{D}$. Then, we define $\phi^{\prime}(x) = \phi(x) - \phi(x_{\circ})$. For a multiset $X \in \mathbb{X}^{l}_{\mathbb{D},[N]}$ with $|X| \leq N$ elements, we have
    \begin{align*}
        \forall X \in \mathbb{X}^{l}_{\mathbb{D},[N]}: \Phi^{\prime}(X) &= \sum_{x \in X} \phi^{\prime}(x) = \sum_{x \in X}  \phi(x) - |X| \phi(x_{\circ} ) \\
        &= \Phi( X \cup \{ \{ \underbrace{x_{\circ} , \ldots, x_{\circ} }_{N-|X|} \} \} ) - N \phi(x_{\circ} ) \\
        &= \Phi( X \cup \{ \{ \underbrace{x_{\circ} , \ldots, x_{\circ} }_{N-|X|} \} \} )+\mathrm{const}
    \end{align*}
where $\mathrm{const} = -N \phi(x_\circ) $. Since $\Phi$ is injective over $\mathbb{X}_{\mathbb{D},N}$, $\Phi^{\prime}$ is an injective map. That is to say,
    \begin{align*}
         \forall X \in \mathbb{X}^{l}_{\mathbb{D},[N]}: \Big( \Phi^{-1} \circ (\Phi^{\prime} (X)  - \mathrm{const} \big) \Big) \cap \mathbb{D}  &= \Big( \Phi^{-1} \circ  \Phi ( X \cup \{ \{ \underbrace{x_{\circ} , \ldots, x_{\circ} }_{N-|X|} \} \} )  \Big) \cap \mathbb{D}   \\
         &= ( X \cup \{ \{ \underbrace{x_{\circ} , \ldots, x_{\circ} }_{N-|X|} \} \} )  \cap \mathbb{D} \\
         &= X.
    \end{align*}
Therefore, we have ${\Phi^{\prime}}^{-1}(U) = \Phi^{-1} \big( U  - \mathrm{const} \big) \cap \mathbb{D}$ for all $U \in \Phi^{\prime}(\mathbb{X}^{l}_{\mathbb{D},[N]}) = \{ \Phi^{\prime}(X): X \in \mathbb{X}^{l}_{\mathbb{D},[N]} \}$. If we define $\rho = f \circ (\Phi^{\prime})^{-1}$ where $\mathrm{dom}(\rho) = \Phi^{\prime}(\mathbb{X}^{l}_{\mathbb{D},[N]})$, then we have $f(X) = \rho \circ \Phi^{\prime}(X)$ for all $X \in  \mathbb{X}^{l}_{\mathbb{D},[N]}$. We arrive at the exact form of sum-decomposition by renaming $\Phi^{\prime}$ to $\Phi$.

\newpage
\section{Proof of \Cref{prop:rational_identifier}}\label{sec:rational_identifier}
Let $X \in \mathbb{X}_{\mathbb{Q}^D,N}$. For any rational-valued vectors $x, x^{\prime} \in X$ such that $l(x) = l(x^{\prime})$, we have
\begin{equation} \label{eq:zeta_bijection}
    \mathrm{const} \sum_{d \in [D]} ( x_d - x_d^{\prime}) \log{\zeta(d)} = 0, 
\end{equation}
where $ x_d$ and $x_d^{\prime}$ are $d$-th elements of $x$ and $x^{\prime}$, and $\mathrm{const} \in \mathbb{N}$ is such that $y_d  \stackrel{\mathrm{def}}{=} \mathrm{const} (x_d - x_d^{\prime}) \in \mathbb{Z}$  --- for all $d \in [D]$. From equation \eqref{eq:zeta_bijection}, we have
\[
    \sum_{d \in [D]} y_d \log{\zeta(d)} = 0 \ \rightarrow \ \prod_{d \in [D]} \zeta(d)^{y_d} = 1.
\]
Therefore, we have
\begin{equation}\label{eq:zeta_bijection_2}
    \prod_{\substack{d \in [D] \\ y_d > 0}} \zeta(d)^{y_d} = \prod_{\substack{d \in [D] \\ -y_d > 0}} \zeta(d)^{-y_d}  = n \in \mathbb{N}
\end{equation}
Both sides of equation \eqref{eq:zeta_bijection_2} are prime number decompositions of an integer $n \in \mathbb{N}$ with completely exclusive set of prime numbers. Therefore, we have  $n = 1$ which results in $y_d =  \mathrm{const} (x_d - x_d^{\prime}) = 0$ for all $d \in [D]$, that is, $x = x^{\prime}$. This proves the following result:
\[
    \forall x , x^{\prime} \in X \big(\in \mathbb{X}_{\mathbb{Q}^D,N} \big): \ l(x) = l(x^{\prime})  \longrightarrow  \ x = x^{\prime}.
\]
Finally, since $l$ is a continuous linear function on $\mathbb{R}^D$, $\mathbb{X}_{\mathbb{Q}^D,N}$ is an $l$-identifiable set. 

\newpage
\section{Proof of \Cref{lem:phi_dense}}\label{sec:phi_dense}
    We need to show for any $X \in \mathbb{X}_{\mathbb{D},N}$, there exists a sequence $\{ X_n \in \mathbb{X}_{Q(\mathbb{D}),N} : n \in \mathbb{N}\}$ such that $\lim_{n \rightarrow \infty} \Phi(X_n) = \Phi(X)$. From \Cref{lem:continuous_Phi2}, $\Phi$ is a continuous map. Therefore, we simply need to prove the following property:
    \[
        \forall X \in \mathbb{X}_{\mathbb{D},N} : \lim_{n \rightarrow \infty} X_n = X,
    \]
    where $X_n \in \mathbb{X}_{Q(\mathbb{D}),N}$ for all $n \in \mathbb{N}$. By definition, we want show $\forall \varepsilon >0, \exists N(\varepsilon) \in \mathbb{N}$ such that
    \[
        \forall n \geq N(\varepsilon): d_M ( X_n^m , X) < \varepsilon.
    \]
    Let $\mathcal{N}_n(x) = \{ y \in Q(\mathbb{D}): \| x - y\|_2 \leq \frac{1}{n} \}$ be a bounded set centered at $x \in \mathbb{D}$ and $n \in \mathbb{N}$. It is important to note that $\mathcal{N}_n(x) $ is a nonempty set for all $n \in \mathbb{N}$, that is, the intersection of $Q(\mathbb{D})$ with the nonempty interior of $\mathbb{D}$ is nonempty because $Q(\mathbb{D})$ is a dense subset of $\mathbb{D}$. We let $q_n(x)$ be {\bf any} random point in $\mathcal{N}_n(x) $. Then, for any $X \in \mathbb{X}_{\mathbb{D},N}$, we let $X_n = \{ \{ q_n(x): x \in X \} \} \in \mathbb{X}_{Q(\mathbb{D}),N}$. By construction, we have
    \[
        d_M (X_n, X) \leq \sqrt{N \max_{x \in \mathbb{D}} \| x - q_n(x) \|_2^2} \leq  \sqrt{N} n^{-1}.
    \]
    If we let $N(\varepsilon)  = \lfloor \frac{\sqrt{N}}{\varepsilon} \rfloor + 1$, then $d_M (X_n, X) \leq \varepsilon$ for all $n \geq N(\varepsilon)$. Therefore, we have $\lim_{n \rightarrow \infty}d_M (X_n, X) = 0 $, that is, $\lim_{n \rightarrow \infty} \mathcal{X}_n^m =  X$. Any realization of the random process $\{ X_n  \}_{n \in \mathbb{N}}$ forms a sequence in $\mathbb{X}_{Q(\mathbb{D}),N}$ that converges to $X$, that is, $\mathbb{X}_{Q(\mathbb{D}),N}$ is a dense subset of $\mathbb{X}_{\mathbb{D},N}$.

    The function $\Phi$ in \Cref{thm:perm_w_distinct_labels} is continuous; see Fact \ref{fact:phi_contu} and \Cref{lem:continuous_Phi2}. Furthermore, we showed that $\mathbb{X}_{Q(\mathbb{D}),N}$ is a dense subset of $\mathbb{X}_{\mathbb{D},N}$. Therefore, for any $U \in \Phi(\mathbb{X}_{\mathbb{D},N}) $ there exists (at least) a $X \in \mathbb{X}_{\mathbb{D},N} $ such that $U = \Phi(X)$. Let $\{ X_n \in \mathbb{X}_{Q(\mathbb{D}),N}: n \in \mathbb{N} \}$ be such that 
    $\lim_{n \rightarrow \infty} X_n = X$. Since $\Phi$ is a continuous map, we have $\lim_{n \rightarrow \infty} \Phi(X_n) = \Phi(X)$. That is, there exists a sequence $\{ U_n = \Phi(X_n) \in \Phi(\mathbb{X}_{Q(\mathbb{D}),N}): n \in \mathbb{N} \}$ such that $\lim_{n \rightarrow \infty} U_n = U$. This proves that $\Phi(\mathbb{X}_{Q(\mathbb{D}),N})$ is a dense subset of $\Phi(\mathbb{X}_{\mathbb{D},N}) $. This completes the proof.

\newpage
\section{Proof of \Cref{thm:continuous_extension_headache}}\label{sec:continuous_extension_headache}
\begin{fact}\label{fact:cont_Phi_}
    Let $\mathbb{D}$ be a compact subset of $\mathbb{R}^D$ with nonempty interior. The function $\phi$ in \Cref{prop:rational_identifier} is continuous. Its associated multiset function $\Phi: \mathbb{X}_{\mathbb{D},N} \rightarrow \codom(\Phi)$ is a continuous function (see \Cref{lem:continuous_Phi})
\end{fact}
From Fact \ref{fact:cont_Phi_} and \Cref{cor:sum_decomp_q}, there exist a continuous multiset function and $\Phi: \mathbb{X}_{\mathbb{D},N} \rightarrow \codom(\Phi) \subset \mathbb{C}^{D \times N}$ and $\rho: \Phi(\mathbb{X}_{ Q(\mathbb{D}),N}) \rightarrow \codom(\rho)$ such that 
    \[
        \forall X \in \mathbb{X}_{ Q(\mathbb{D}),N}: \ f(  X )  = \rho \big( \sum_{x \in X} \phi (x) \big) = \rho \circ \Phi(X).
    \]
In this proof, we want to define the function $\rho_e:  \Phi(\mathbb{X}_{\mathbb{D},N}) \rightarrow \mathrm{codom}(\rho_e)$ as follows:
    \begin{equation}\label{eq:q_e}
        \forall Z \in \Phi(\mathbb{X}_{\mathbb{D},N}) : \rho_e (Z) =  \lim_{ Z_n \rightarrow Z} \rho (Z_n),
    \end{equation}
    where $Z_n \in \Phi(\mathbb{X}_{Q(\mathbb{D}),N})$ for all $n \in \mathbb{N}$. The goal is to show that $\rho_e$ is (1) well-defined and (2) continuous over its compact domain $\Phi(\mathbb{X}_{\mathbb{D},N})$. If these two conditions are valid, we let $Z = \Phi(X) \in \Phi(\mathbb{X}_{Q(\mathbb{D}),N})$ where $X \in \mathbb{X}_{ Q(\mathbb{D}),N}$ and $\{ Z_n \in \Phi(\mathbb{X}_{Q(\mathbb{D}),N}): Z_n = Z, n \in \mathbb{N} \}$. Then, we have
    \[
        \forall X \in \mathbb{X}_{Q(\mathbb{D}),N}: \ f(  X )  = \rho_e  \circ \Phi(X) = \rho \circ \Phi(X).
    \]
    \begin{proposition}[{\bf Well-definedness}]\label{prop:well_defined_rho_e}
        Let $\mathcal{Z} \stackrel{\mathrm{def}}{=} \{ Z_n \in \Phi(\mathbb{X}_{ Q(\mathbb{D}),N}): n \in \mathbb{N} \}$ be the convergent sequence, that is, $ \lim_{n \rightarrow \infty} Z_n = Z$. Given a continuous multiset function $f: \mathbb{X}_{\mathbb{D}, N} \rightarrow \mathrm{codom}(f)$, let $\rho: \Phi(\mathbb{X}_{Q(\mathbb{D}), N}) \rightarrow \mathrm{codom}(\rho) \subset f(\mathbb{X}_{\mathbb{D}, N})$ be defined in \Cref{cor:sum_decomp_q}. Then, the sequence $\rho (\mathcal{Z}) \stackrel{\mathrm{def}}{=} \{ \rho (Z_n) : n \in \mathbb{N} \}$ is convergent to a unique point in $f(\mathbb{X}_{\mathbb{D},N})$. The term $\lim_{n \rightarrow \infty} \rho(Z_n)$ only depends on $Z$, and not specific choice of the sequence $\mathcal{Z}$.
    \end{proposition}
    As a result of \Cref{prop:well_defined_rho_e}, the function $\rho_e:  \Phi(\mathbb{X}_{\mathbb{D},N}) \rightarrow \mathrm{codom}(\rho_e) \subseteq f(\mathbb{X}_{\mathbb{D},N})$ is well-defined. That is, $\lim_{ Z_n \rightarrow Z} \rho (Z_n)$ does not depend on the specific convergent sequence $\mathcal{Z}$ so long as its limiting point --- $\lim_{n \rightarrow \mathbb{N}}Z_n = Z$ --- is fixed.
    \begin{proposition}[{\bf Continuity}]\label{prop:rho_e_cont}
        The function $\rho_e$ is continuous on the compact domain $\Phi(\mathbb{X}_{\mathbb{D},N})$.
    \end{proposition}
    In summary, we have
    \[
        \forall X \in \mathbb{X}_{Q(\mathbb{D}),N}: f(  X )  = \rho_e \circ \Phi(X).
    \]
    where $\rho_e: \Phi(\mathbb{X}_{\mathbb{D},N}) \rightarrow \mathrm{codom}(\rho_e)$ and $\Phi: \mathbb{X}_{\mathbb{D},N} \rightarrow \codom(\Phi)$ are continuous functions. Therefore, $\rho_e \circ \Phi$ is a continuous function on $\mathbb{X}_{\mathbb{D},N}$. Since $\mathbb{X}_{Q(\mathbb{D}),N}$ is a dense subset of $\mathbb{X}_{\mathbb{D},N}$ (see \Cref{lem:phi_dense}) and $f:\mathbb{X}_{\mathbb{D},N} \rightarrow \mathrm{codom}(f) $ is a continuous multiset function, we have
    \[
        \forall X\in \mathbb{X}_{\mathbb{D},N}: \ f(X) = \lim_{n \rightarrow \infty} f(X_n)
    \]
    for any sequence $\{X_n \in \mathbb{X}_{Q(\mathbb{D}),N}: n \in \mathbb{N} \}$ where $\lim_{n \rightarrow \infty} X_n = X$. Therefore, we have
    \[
        \forall X\in \mathbb{X}_{\mathbb{D},N}: \ f(X) = \lim_{n \rightarrow \infty} \rho_e \circ \Phi(X_n).
    \]
    Since $\rho_e \circ \Phi$ is a continuous function on $\mathbb{X}_{\mathbb{D},N}$, we have
    \[
        \forall X\in \mathbb{X}_{\mathbb{D},N}: \ f(X) = \lim_{n \rightarrow \infty} \rho_e \circ \Phi(X_n) =  \rho_e \circ \Phi(\lim_{n \rightarrow \infty} X_n) =  \rho_e \circ \Phi(X).
    \]

    We argue that $\rho_e$ has a continuous extension to $\mathbb{C}^{D \times N}$. The set $\mathbb{X}_{\mathbb{D},N}$ is a compact set. From \Cref{lem:continuous_Phi2}, $\Phi(\mathbb{X}_{\mathbb{D},N})$ is also a compact set. Finally, Fact \ref{fact:cont_compact} shows this continuous extension is admitted. After renaming $\rho_e$ to $\rho$, we arrive at the exact statement of the theorem.

\subsection{Proof of \Cref{prop:well_defined_rho_e}}
    \begin{lemma}\label{lemma:cauchy_rho_z}
        Let $\mathcal{Z} \stackrel{\mathrm{def}}{=} \{ Z_n \in \Phi(\mathbb{X}_{ Q(\mathbb{D}),N}): n \in \mathbb{N} \}$ be the convergent sequence, that is, $\lim_{n \rightarrow \infty} Z_n = Z \in \Phi(\mathbb{X}_{ \mathbb{D},N})$. Given a continuous multiset function $f: \mathbb{X}_{\mathbb{D}, N} \rightarrow \mathrm{codom}(f)$, let $\rho: \Phi(\mathbb{X}_{Q(\mathbb{D}), N}) \rightarrow \mathrm{codom}(\rho) \subset f(\mathbb{X}_{\mathbb{D}, N})$ be defined in \Cref{cor:sum_decomp_q}. The sequence $\rho (\mathcal{Z}) \stackrel{\mathrm{def}}{=} \{ \rho (Z_n) : n \in \mathbb{N} \}$ is Cauchy in compact metric space $(f(\mathbb{X}_{ \mathbb{D},N}), \| \cdot \|_2 )$.
    \end{lemma}
    \begin{theorem} \label{thm:compact_metric_cauchy}
    {\citep{attenborough2003mathematics}}
        A Cauchy sequence in a compact metric space is convergent to a point in the metric space.
    \end{theorem}
    \begin{lemma} \label{lem:unique_limit}
        Let $\mathcal{Z} \stackrel{\mathrm{def}}{=} \{ Z_n \in \Phi(\mathbb{X}_{ Q(\mathbb{D}),N}): n \in \mathbb{N} \}$ be the convergent sequence, that is, $\lim_{n \rightarrow \infty} Z_n = Z \in \Phi(\mathbb{X}_{ \mathbb{D},N})$. Given a continuous multiset function $f: \mathbb{X}_{\mathbb{D}, N} \rightarrow \mathrm{codom}(f)$, let $\rho: \Phi(\mathbb{X}_{Q(\mathbb{D}), N}) \rightarrow \mathrm{codom}(\rho) \subset f(\mathbb{X}_{\mathbb{D}, N})$ be defined in \Cref{cor:sum_decomp_q}. The sequence $\rho (\mathcal{Z}) \stackrel{\mathrm{def}}{=} \{ \rho (Z_n) : n \in \mathbb{N} \}$ is convergent to a unique point in $f(\mathbb{X}_{ \mathbb{D},N})$. The term $\lim_{n \rightarrow \infty} \rho(Z_n)$ only depends on $Z = \lim_{n \rightarrow} Z_n$.
    \end{lemma}
\subsubsection{Proof of \Cref{lemma:cauchy_rho_z}}
    \begin{fact}\label{fact:cauchy_z}
        Every convergent sequence is Cauchy. Hence, the convergent sequence $\mathcal{Z} \stackrel{\mathrm{def}}{=} \{ Z_n: n \in \mathbb{N} \}$ is Cauchy in $( \Phi(\mathbb{X}_{Q(\mathbb{D}),N}), \| \cdot \|_F )$.
    \end{fact}
        From Fact \ref{fact:cauchy_z}, for any $\delta > 0$, there exists $N(\delta) \in \mathbb{N}$ such that $\| Z_{n_1} - Z_{n_2} \|_F < \delta$ for all $ n_1, n_2 > N(\delta )$. Therefore, we have
        \begin{equation} \label{eq:cauchy_z_n}
            \forall n_1, n_2 > N(\delta ): \| \Phi( X_{n_1} ) - \Phi(X_{n_2}) \|_F < \delta.    
        \end{equation}
    where $X_n = \Phi^{-1}(Z_n)$ for all $n \in \mathbb{N}$.
    The set $\mathbb{X}_{Q(\mathbb{D}),N}$ is an $l$-identifiable subset of $\mathbb{X}_{\mathbb{D},N}$. 
    \begin{proposition}\label{prop:Phi_inv_cont}
        Let $\mathbb{X}_{\mathbb{R}^D / l, N}$ be an $l$-identifiable set, and $\Phi(\mathbb{X}_{\mathbb{R}^D / l, N}) = \{ \Phi(X): X \in\mathbb{X}_{\mathbb{R}^D / l, N} \}$ where $\Phi$ is defined in equations \eqref{eq:phi_l} and \eqref{eq:Phi_l}. The function $\Phi^{-1}: \Phi( \mathbb{X}_{\mathbb{R}^D / l, N}) \rightarrow \mathbb{X}_{\mathbb{R}^D / l, N}$ is defined in the proof of \Cref{lem:Phi_inj}. We claim that $\Phi^{-1}$ is a continuous function on its domain.
    \end{proposition}
    From \Cref{prop:Phi_inv_cont}, $\Phi^{-1}$ is a continuous function on $\Phi(\mathbb{X}_{Q(\mathbb{D}),N} )$. Since $f$ is a continuous multiset function on its domain, the function $\rho = f \circ \Phi^{-1}$ is continuous on $\Phi(\mathbb{X}_{Q(\mathbb{D}),N} )$. By definition of continuity, for any $\varepsilon >0$ and $\Phi(X)$ where $X \in \mathbb{X}_{Q(\mathbb{D}),N}$, there exists $\delta(\varepsilon) >0$ such that
    \begin{equation}\label{eq:cont_phi_epsilon}
        \forall X^{\prime} \in \mathbb{X}_{Q(\mathbb{D}),N}: \| \Phi(X) - \Phi(X^{\prime})\|_F < \delta(\varepsilon) \rightarrow \| f(X)- f(X^{\prime}) \|_2 < \varepsilon,
    \end{equation}
    where $X = \Phi^{-1} \circ \Phi(X)$ and $X^{\prime} = \Phi^{-1} \circ \Phi(X^{\prime})$. 
    
    Comparing equation \eqref{eq:cauchy_z_n} to the left-hand-side of equation \eqref{eq:cont_phi_epsilon} --- letting $Z_{n_1} =  \Phi(X) $ and $Z_{n_2} =  \Phi(X^{\prime})$ ---  we have
    \[
        \forall n_1, n_2 > N(\delta(\varepsilon)): \| f \circ \Phi^{-1}(Z_{n_1})-  f\circ \Phi^{-1}(Z_{n_2}) \|_2 < \varepsilon,
    \]
    that is, $f \circ \Phi^{-1}(\mathcal{Z}) = \rho (\mathcal{Z})$ is a Cauchy sequence; see \Cref{cor:sum_decomp_q} and proof of \Cref{thm:perm_w_distinct_labels}. Finally, \Cref{lem:compact_set_of_sets} and the following fact show that $f(\mathbb{X}_{\mathbb{D},N})$ is indeed a compact set, that is, $(f(\mathbb{X}_{\mathbb{D},N}), \| \cdot \|_2) $ is a sequentially compact metric space.
    \begin{fact}(\citealt{pugh2002real})
        The image of a compact set under continuous map is a compact set. 
    \end{fact}
    
{\bf Proof of \Cref{prop:Phi_inv_cont}}
    Let $\varepsilon >0$ and $Z \in \Phi( \mathbb{X}_{\mathbb{R}^D / l, N})$ --- that is, $Z = \Phi(X)  \in \mathbb{C}^{D\times N}$ for a unique $X \in \mathbb{X}_{\mathbb{R}^D / l, N}$. We define $\mathbb{D}_{\Phi}(\varepsilon, Z) = \{ Z^{\prime} \in \Phi( \mathbb{X}_{\mathbb{R}^D / l, N}): \| Z - Z^{\prime} \|_F < \varepsilon \}$. For any $Z^{\prime}  \in \mathbb{D}_{\Phi}(\varepsilon, Z)$, we have
    \begin{align*}
        d_{M}( \Phi^{-1}(Z)&, \Phi^{-1}(Z^{\prime} ))  \stackrel{\text{(a)}}{\leq} \sup_{\substack{\mathrm{dZ} \in \mathbb{C}^{D\times N}:  \\ Z+\mathrm{d} Z \in \mathbb{D}_{\Phi}(\varepsilon, Z) }}d_{M}( X,  \Phi^{-1} \circ (Z+\mathrm{d} Z)) \\
        &\stackrel{\text{(b)}}{=} \sup_{\substack{\mathrm{dZ} \in \mathbb{C}^{D\times N}:  \\ Z+\mathrm{d} Z \in \mathbb{D}_{\Phi}(\varepsilon, Z) }}d_{M}\Big( X,  \Phi^{-1} \big(
        \begin{pmatrix}
        \Phi_{\mathrm{deep}}  ( \{ \{ e_{1}^{\top}x+l(x) j: x\in X \} \} ) \\
                    \vdots \\
        \Phi_{\mathrm{deep}} ( \{ \{ e_{D}^{\top}x+l(x) j : x\in X  \} \})
        \end{pmatrix}  +\mathrm{dZ} \big) \Big) \\
        &\stackrel{\text{(c)}}{=} \sup_{\substack{\mathrm{dZ} \in \mathbb{C}^{D\times N}:  \\ Z+\mathrm{d} Z \in \mathbb{D}_{\Phi}(\varepsilon, Z) }} d_{M}\Big( X,  \mathrm{sortvec}  \big(
        \begin{pmatrix}
         \{ \{ e_{1}^{\top}x+l(x) j +r_{1}(x,\mathrm{dz}_1; X): x\in X \} \}  \\
                    \vdots \\
        \{ \{ e_{D}^{\top}x+l(x) j+ r_{D}(x,\mathrm{dz}_D; X): x\in X  \} \} 
        \end{pmatrix}   \big)  \Big) 
    \end{align*}
    where $\text{(a)}$ is due to the fact that we have $\|Z- Z^{\prime}\|_F < \varepsilon$ and $\Phi^{-1} = \mathrm{sortvec}\circ \Phi_{\mathrm{deep}}^{-1}$ is well-defined for $Z,Z^{\prime} \in \Phi( \mathbb{X}_{\mathbb{R}^D / l, N})$, $\text{(b)}$ is due to the definition of $\Phi$ --- see equations \eqref{eq:phi_l} and \eqref{eq:Phi_l} --- and the fact that $Z = \Phi(X)$, $\text{(c)}$ is due letting $\mathrm{dz}_d = e_D^{\top}\mathrm{d} Z \in \mathbb{C}^N$ where $e_d$ is the $d$-th standard basis of $\mathbb{R}^D$ --- for all $d \in [D]$ --- and the fact that $\Phi_{\mathrm{deep}}^{-1}$ is a continuous function \citep{zaheer2017deep}, that is, 
    \[
        \forall d \in [D], X \in \mathbb{X}_{\mathbb{R}^D / l, N}, x \in X: \lim_{\mathrm{dz} \in \mathbb{C}^N:\mathrm{dz}  \rightarrow 0 }r_{d}(x,\mathrm{dz}; X) = 0. 
    \]
    For any $\varepsilon >0$, $x \in X$, and $d \in [D]$, there exists a finite $\delta_d(\varepsilon, x;X) = \sup_{\mathrm{dz} \in \mathbb{D}(\varepsilon)} | r_{d}(x,\mathrm{dz};X) |$ where $\mathbb{D}(\varepsilon) = \{ z \in \mathbb{C}^N: \| z \|_2 < \varepsilon \}$ and $\lim_{\varepsilon \rightarrow 0 } \delta_d(\varepsilon, x;X) = 0$. For any $\varepsilon >0$ and $X \in \mathbb{X}_{\mathbb{R}^D / l, N}$, we have
    \begin{equation}\label{eq:delta_star}
        \delta^*(\varepsilon, X) \stackrel{\mathrm{def}}{=} \max_{d \in [D], x \in X} \delta_d(\varepsilon, x;X) = \sup_{d \in [D], x \in X,\mathrm{dz} \in \mathbb{D}(\varepsilon)} |r_{d}(x,\mathrm{dz};X)|,
    \end{equation}
    where $\lim_{\varepsilon \rightarrow 0 }\delta^*(\varepsilon, X) = 0$. 
    Let $X = \{ \{ x_n : n \in [N] \} \}$. Then, we have
    \[
        \forall d \in [D]: \mathrm{sort}(\{\{ e_d^{\top}x+l(x)j : x \in X \}\} ) =  \big(  e_d^{\top}x_{\pi(n)}  \big)_{n \in [N]} \in \R^N
    \]
    where $\pi: [N] \rightarrow [N]$ is a permutation operator such that $l(x_{\pi(1)}) \leq \cdots \leq l( x_{\pi(N)})$. Even though the permutation operator $\pi$ may not be unique, $\mathrm{sort}$ is a well-defined function; see the proof of \Cref{thm:continuously_sum_decomposable_1d}.  We let $S(X) \stackrel{\mathrm{def}}{=} \{ \varepsilon: \delta^*(\varepsilon, X) < \psi(X)\}$ and $\psi(X) \stackrel{\mathrm{def}}{=} \min_{\substack{x, x^{\prime} \in X \\ x \neq x^{\prime}}} \frac{1}{2}|l(x)-l(x^{\prime})| > 0$ where $l:\mathbb{R}^D \rightarrow \mathbb{R}$ is the continuous identifier function. From equation \eqref{eq:delta_star}, we have
    \[
        \forall d \in [D], \varepsilon \in S(X), x \in X, \mathrm{dz} \in \mathbb{D}(\varepsilon) : \mathrm{Im}( r_{d}(x,\mathrm{dz};X) ) < \delta^*(\varepsilon, X) < \psi(X),
    \]
    that is, 
    \[
        \forall d \in [D], \varepsilon \in S(X), x \in X, \mathrm{dz} \in \mathbb{D}(\varepsilon): \mathrm{Im}( r_{d}(x,\mathrm{dz};X) ) < \min_{\substack{x, x^{\prime} \in X \\ x \neq x^{\prime}}} \frac{1}{2}|l(x)-l(x^{\prime})|.
    \]
    Therefore, $\mathrm{dZ}$ perturbs the imaginary components of $\{ \{ e_{d}^{\top}x+l(x) j +r_{d}(x,\mathrm{dz}_d; X): x\in X \} \}$ (for any $d \in [D]$) by at most $\delta^*(\varepsilon, X) <  \min_{\substack{x, x^{\prime} \in X \\ x \neq x^{\prime}}} \frac{1}{2}|l(x)-l(x^{\prime})|$ and distinct elements do not switch place after adding the perturbation $\mathrm{dZ}$. More precisely, for all $d \in [D]$, $\varepsilon \in S(X)$, and $\mathrm{dz} \in \mathbb{D}(\varepsilon)$, we have
    \[
        \mathrm{sort}(\{\{ e_d^{\top}x+l(x)j+r_d(x,\mathrm{dz};X) : x \in X \}\} ) =  \big(  e_d^{\top}x_{\pi^{\prime}(n)}+ \mathrm{Re}( r_{d}(x_{\pi^{\prime}(n)},\mathrm{dz};X))  \big)_{n \in [N]} \in \R^N,
    \]
    where $\pi^{\prime}: [N] \rightarrow [N]$ is such that for all $\varepsilon \in S(X)$, we have
    \[
    \forall \mathrm{dz} \in \mathbb{D}(\varepsilon): l(x_{\pi^{\prime}(1)})+\mathrm{Im}(r_d(x_{\pi^{\prime}(1)},\mathrm{dz};X)) \leq \cdots \leq l( x_{\pi^{\prime}(N)})+\mathrm{Im}(r_d(x_{\pi^{\prime}(N)},\mathrm{dz};X)).
    \]
    Since $\mathrm{Im}( r_{d}(x,\mathrm{dz};X) ) < \min_{\substack{x, x^{\prime} \in X \\ x \neq x^{\prime}}} \frac{1}{2}|l(x)-l(x^{\prime})|$, we also have the following inequalities: 
    \begin{equation}\label{eq:l_x_pi}
        l(x_{\pi^{\prime}(1)}) \leq \cdots \leq l( x_{\pi^{\prime}(N)}).
    \end{equation}
    \begin{remark}
        The permutation operator $\pi^{\prime}$ may vary with $\mathrm{dz}$ and $x$, if two (or more) elements of $\{ \{ l(x): x \in X \} \}$ are identical. A proper notation should be $\pi^{\prime}(\mathrm{dz},x;X)$. For simplicity in notation, we avoid expressing this proper parameterization. 
    \end{remark}
    \begin{remark}
        The perturbation $\mathrm{dz}$ may switch the rank (or position) of two elements only if they are equal to each other, that is, if $l( x_{\pi^{\prime}(1)}) = l( x_{\pi^{\prime}(2)})$, then we may have $l(x_{\pi^{\prime}(1)})+\mathrm{Im}(r_d(x_{\pi^{\prime}(1)},\mathrm{dz};X)) < l(x_{\pi^{\prime}(2)})+\mathrm{Im}(r_d(x_{\pi^{\prime}(2)},\mathrm{dz};X))$. This does not provide any issue, since $l( x_{\pi^{\prime}(1)}) \leq l( x_{\pi^{\prime}(2)})$. In short, independent of $\mathrm{dz}$ and $x$, $\pi^{\prime}$ is such that $l(x_{\pi^{\prime}(1)}) \leq \cdots \leq l( x_{\pi^{\prime}(N)})$.
    \end{remark}
    
    From equation \eqref{eq:l_x_pi} and the definition of $\pi$, we have $l(x_{\pi(n)}) = l(x_{\pi^{\prime}(n)})$ for all $n \in [N]$ --- even though, we may have $\pi \neq \pi^{\prime}$. From \Cref{def:identifiable_sets}, since $l(x_{\pi(n)}) = l(x_{\pi^\prime(n)})$, we have $x_{\pi^\prime(n)} = x_{\pi(n)}$ for all $n \in [N]$. Consequently, for all $d \in [D]$, $\varepsilon \in S(X)$ and $\mathrm{dz} \in \mathbb{D}(\varepsilon)$, we have
    \begin{align*}
        \mathrm{sort}(\{\{ e_d^{\top}x+l(x)j+r_d(x,\mathrm{dz};X) : x \in X \}\} ) &=  \big(  e_d^{\top}x_{\pi^{\prime}(n)} + \mathrm{Re}(r_d(x_{\pi^{\prime}(n)},\mathrm{dz};X))  \big)_{n \in [N]} \\
        \mathrm{sort}(\{\{ e_d^{\top}x+l(x)j : x \in X \}\} ) &=\big(  e_d^{\top}x_{\pi(n)}  \big)_{n \in [N]}  = \big(  e_d^{\top}x_{\pi^{\prime}(n)}  \big)_{n \in [N]} .
    \end{align*}
    Therefore, even though there are multiple permutation operators that $\mathrm{sort}$s the elements of $\{\{ e_d^{\top}x+l(x)j+r_d(x,\mathrm{dz};X) : x \in X \}\}$, the output of the $\mathrm{sort}$ function gives an ordering that remains unchanged for distinct elements of $X$ for any $x \in X \in \mathbb{X}_{\mathbb{R}^D / l, N}$ and $\mathrm{dz} \in \mathbb{D}(\varepsilon)$ where $\varepsilon \in S(X)$. Consequently, we have
    \begin{align*}
        &d_{M}( \Phi^{-1}(Z), \Phi^{-1}(Z^{\prime} )) \leq  \sup_{\substack{\mathrm{dZ} \in \mathbb{C}^{D\times N}:  \\ Z+\mathrm{d} Z \in \mathbb{D}_{\Phi}(\varepsilon, Z) }}  d_{M}(X,  \mathrm{sortvec}  \big(
        \begin{pmatrix}
         \{ \{ e_{1}^{\top}x+l(x) j +r_{1}(x,\mathrm{dz}_1;X): x\in X \} \}  \\
                    \vdots \\
        \{ \{ e_{D}^{\top}x+l(x) j+ r_{D}(x,\mathrm{dz}_D;X): x\in X  \} \} 
        \end{pmatrix}   \big)  ) \\
        &\stackrel{\text{(a)}}{\leq}   \sup_{\substack{\mathrm{dZ} \in \mathbb{C}^{D\times N}:  \\ Z+\mathrm{d} Z \in \mathbb{D}_{\Phi}(\varepsilon, Z) }}  d_{M}(X, \{\{ \begin{pmatrix}
         e_1^{\top}x_{\pi_1(n)} + \mathrm{Re}( r_{1}(x_{\pi_1(n)},\mathrm{dz}_1;X) )  \\
                    \vdots \\
         e_D^{\top}x_{\pi_D(n)} +\mathrm{Re}(r_{D}(x_{\pi_D(n)},\mathrm{dz}_D;X) )
        \end{pmatrix}: n \in [N]  \} \})  \\
        &\stackrel{\text{(b)}}{\leq}   \sup_{\substack{\mathrm{dZ} \in \mathbb{C}^{D\times N}:  \\ Z+\mathrm{d} Z \in \mathbb{D}_{\Phi}(\varepsilon, Z) }}  d_{M}( \{\{ \begin{pmatrix}
         e_1^{\top}x_{\pi_1(n)}  )  \\
                    \vdots \\
         e_D^{\top}x_{\pi_D(n)}  )
        \end{pmatrix}: n \in [N]  \} \}, \{\{ \begin{pmatrix}
         e_1^{\top}x_{\pi_1(n)} + \mathrm{Re}( r_{1}(x_{\pi_1(n)},\mathrm{dz}_1;X) )  \\
                    \vdots \\
         e_D^{\top}x_{\pi_D(n)} +\mathrm{Re}(r_{D}(x_{\pi_D(n)},\mathrm{dz}_D;X) )
        \end{pmatrix}: n \in [N]  \} \})  \\
        &\stackrel{\text{(c)}}{\leq}  \sup_{\substack{\mathrm{dZ} \in \mathbb{C}^{D\times N}:  \\ Z+\mathrm{d} Z \in \mathbb{D}_{\Phi}(\varepsilon, Z) }}   \sqrt{\sum_{n \in [N]} \|\begin{pmatrix}
            \mathrm{Re}(r_{1}(x_{\pi_1(n)},\mathrm{dz}_1;X))  \\
                    \vdots \\
         \mathrm{Re}(r_{D}(x_{\pi_D(n)},\mathrm{dz}_D;X) )
        \end{pmatrix}\|_2^2} \\
        &\stackrel{\text{(d)}}{\leq} \sqrt{DN}  \sup_{d \in [D], x \in X, \mathrm{dz} \in \mathbb{D}(\varepsilon)}  |r_{d}(x,\mathrm{dz};X)| =  \sqrt{DN}\delta^*(\varepsilon, X)
    \end{align*}
    where $\text{(a)}$ uses permutation operators $\pi_d: [N] \rightarrow [N]$ and it depends on elements of $\{ \{ r_{d}(x,\mathrm{dz}_1;X): x \in X \}$, but $x_{\pi_d(n)} = x_{\pi(n)}$ for all $n \in [N]$ and $d \in [D]$, $\text{(b)}$ follows from the fact that $x_{\pi_d(n)} = x_{\pi(n)}$ for all $n \in [N]$ and $d \in [D]$, $\text{(c)}$ follows from the definition of the matching distance $d_M$, and $\text{(d)}$ follows from the fact that if $\mathrm{dZ} \in \mathbb{C}^{D \times N}$ is such that $Z+\mathrm{dZ}\in \mathbb{D}_{\Phi}(\varepsilon, Z)$, then its individual rows $\mathrm{dz}_1, \ldots, \mathrm{dz}_D \in \mathbb{R}^N$ have norms upper bounded by $\varepsilon$, that is, $\mathrm{dz}_d \in  \mathbb{D}(\varepsilon)$ and $|\mathrm{Re}(r_{d}(x,\mathrm{dz}_d;X))| \leq |r_{d}(x,\mathrm{dz}_d;X)|$ for all $d \in [D]$. 

    {\bf Continuity Statement. }For any $Z = \Phi(X) \in \Phi(\mathbb{X}_{\mathbb{R}^D,N})$ and $\delta >0$, there exists a positive $\epsilon(\delta) \in \{  \varepsilon^{\prime} \in S(X): \sqrt{DN}\delta^*(\varepsilon^{\prime}, X) < \delta \} $ where
    \[
        \forall Z^{\prime} \in \Phi(\mathbb{X}_{\mathbb{R}^D,N}): \| Z - Z^{\prime}\|_F < \epsilon(\delta) \rightarrow d_M(\Phi^{-1}(Z), \Phi^{-1}(Z^{\prime}) ) < \delta.
    \]

    \subsubsection{Proof of \Cref{lem:unique_limit}}
    Let $\mathcal{Z}_1 = \{ Z_{1,n} \in \Phi(\mathbb{X}_{Q(\mathbb{D}),N}): n \in \mathbb{N}\}$ and $\mathcal{Z}_2 = \{ Z_{2,n} \in \Phi(\mathbb{X}_{Q(\mathbb{D}),N}): n \in \mathbb{N}\}$ be two sequences such that $\lim_{n \rightarrow \infty} Z_{1,n} = \lim_{n \rightarrow \infty} Z_{2,n} = Z$. \\
    From \Cref{lemma:cauchy_rho_z}, the following limits are well-defined:
    \[
        \lim_{n \rightarrow \infty} \rho(Z_{1,n}) = f_1, \ \lim_{n \rightarrow \infty} \rho(Z_{2,n}) = f_2 \in f(\mathbb{X}_{\mathbb{D},N}).
    \]
    We construct $\mathcal{Z} = \{ Z_n: n \in \mathbb{N}\}$ in $\Phi(\mathbb{X}_{Q(\mathbb{D}),N})$ where $Z_{2n} = Z_{1,n}$ and $Z_{2n+1} = Z_{2,n}$ for all $n \in \mathbb{N}$. By construction, we have $\lim_{n \rightarrow \infty} Z_n = Z$. Since all convergent sequences are Cauchy, $\mathcal{Z}$ is a Cauchy sequence. Therefore, from our discussion the proof of \Cref{lemma:cauchy_rho_z}, the sequence $\rho(\mathcal{Z})$ must converge to $f^* \in f(\mathbb{X}_{\mathbb{D},N})$.
    \begin{fact}
        Every subsequence of a convergent sequence converges to the same limit as the original sequence.
    \end{fact}
    Both $\rho(\mathcal{Z}_1)$ and $\rho(\mathcal{Z}_2)$ are subsequences of the convergent sequence $\rho(\mathcal{Z})$. Therefore, we have
    \[
        \lim_{n \rightarrow \infty} \rho(Z_{1,n}) = \lim_{n \rightarrow \infty} \rho(Z_{2,n}) = \lim_{n \rightarrow \infty} \rho(Z_{n}) = f^*,
    \]
    that is, $f_1 = f_2$. Therefore, the limit of $\rho(\mathcal{Z})$ only depends on the limit of the sequence $\mathcal{Z}$.

\subsection{Proof of \Cref{prop:rho_e_cont}}
    We want to show that, for any $\Phi(X) \in \Phi(\mathbb{X}_{\mathbb{D},N})$ and $\varepsilon >0$, there is $\delta(\varepsilon)>0 $ such that
    \begin{equation} \label{eq:cont_rho_e}
        \forall X^{\prime}\in \mathbb{X}_{\mathbb{D},N}: \| \Phi(X) - \Phi(X^{\prime}) \|_F < \delta(\varepsilon) \rightarrow \| \rho_e \circ \Phi(X) - \rho_e \circ \Phi(X^{\prime}) \| < \varepsilon.
    \end{equation}
    We first use the definition of $\rho_e$ to reforulate the left-hand-side of equation \eqref{eq:cont_rho_e} in terms of convergent sequences in $\Phi(\mathbb{X}_{ Q(\mathbb{D}),N})$. This is formalized in \Cref{lem:cont_right_sequence}.
    \begin{lemma}\label{lem:cont_right_sequence}
        Let $X, X^{\prime}\in \mathbb{X}_{\mathbb{D},N} $. There exist convergent sequences $\mathcal{Z}_x \stackrel{\mathrm{def}}{=} \{ Z_{x,n} \in \Phi(\mathbb{X}_{Q(\mathbb{D}),N}): n \in \mathbb{N}\}$ and $\mathcal{Z}_y  \stackrel{\mathrm{def}}{=} \{ Z_{y,n} \in \Phi(\mathbb{X}_{Q(\mathbb{D}),N}): n \in \mathbb{N}\}$ and $N_x(\delta), N_y(\delta) \in \mathbb{N}$ such that 
        \begin{align*}
            \forall n > N_x(\delta )&: \ \| Z_{x,n} - \Phi(X)  \|_2 < \delta    \\
            \forall n > N_y(\delta )&: \ \| Z_{y,n} - \Phi(X^{\prime})  \|_2 < \delta.
        \end{align*}
        for any $\delta >0$. If $\| \Phi(X) -  \Phi(X^{\prime}) \|_2 < \delta$, then we have
        \[
            \forall n > N(\delta) \stackrel{\mathrm{def}}{=} \max \{N_x(\delta),N_y(\delta) \}: \ \| Z_{x,n} - Z_{y,n} \|_2 < 3\delta.
        \]
    \end{lemma}
    As the result of \Cref{lem:cont_right_sequence}, the left-hand-side of equation \eqref{eq:cont_rho_e} gives us the following inequality:
    \[
        \forall X^{\prime}\in \mathbb{X}_{\mathbb{D},N}: \| \Phi(X) - \Phi(X^{\prime}) \|_F < \delta , n > N(\delta ) \rightarrow  \| Z_{x,n} - Z_{y,n} \|_2 < 3\delta,
    \]
    where $N(\delta) \in \mathbb{N}$, $\mathcal{Z}_x = \{ Z_{x,n}: n \in \mathbb{N}\}$ and $\mathcal{Z}_y  = \{ Z_{y,n} : n \in \mathbb{N}\}$ are convergent sequences in $\Phi(\mathbb{X}_{Q(\mathbb{D}),N})$ (in \Cref{lem:cont_right_sequence}), that is,
    \[
        \lim_{n \rightarrow \infty} Z_{x,n} = \Phi(X), \ \lim_{n \rightarrow \infty} Z_{y,n} = \Phi(X^{\prime}) \in \Phi(\mathbb{X}_{\mathbb{D},N}).
    \]
    In \Cref{lem:cont_right_sequence} we prove that convergent sequences $\mathcal{Z}_x$ and $\mathcal{Z}_y$ become arbitrary close to each other as $\delta \rightarrow 0$. In \Cref{lem:cont_rho_z}, we use the fact that $\rho$ (not $\rho_e$) is a continuous function on noncompact domain $\Phi(\mathbb{X}_{Q(\mathbb{D}),N})$, and argue that $\| \rho( Z_{x,n}) - \rho( Z_{y,n}) \|_2$ converges to zero as $\delta \rightarrow 0$.
    \begin{lemma} \label{lem:cont_rho_z}
        For all $Z \in \Phi(\mathbb{X}_{Q(\mathbb{D}),N})$ and $\delta >0$, there exists a $\gamma(\delta) > 0$ such that
        \[
            \forall Z^{\prime} \in \Phi(\mathbb{X}_{Q(\mathbb{D}),N}): \| Z - Z^{\prime} \|_F < \gamma(\delta ) \rightarrow \| \rho(Z) - \rho(Z^{\prime}) \|_2 < \delta.
        \]
        For any $\delta >0$, we have
        \[
            \forall n > N^{'}(\delta ): \| \rho( Z_{x,n}) - \rho( Z_{y,n}) \|_2 < \delta .
        \]    
        where $N^{'}(\delta ) = N( \min \{ \frac{\delta}{3} ,  \frac{\gamma(\delta)}{3} \})$, and $N(\delta) \in \mathbb{N}$, convergent sequences $\mathcal{Z}_x = \{ Z_{x,n}: n \in \mathbb{N}\}$ and $\mathcal{Z}_y = \{ Z_{y,n}: n \in \mathbb{N}\}$ are defined in \Cref{lem:cont_right_sequence}. 
    \end{lemma}
    We now use \Cref{lem:cont_rho_z} to show that for all $\delta >0$, we have
    \[
        \forall X^{\prime}\in \mathbb{X}_{\mathbb{D},N}: \| \Phi(X) - \Phi(X^{\prime}) \|_F < \delta , n > N^{'}(\delta ) \rightarrow  \| \rho( Z_{x,n}) - \rho( Z_{y,n}) \|_2 < \delta,
    \]
    where $\mathcal{Z}_x = \{ Z_{x,n}: n \in \mathbb{N}\}$ and $\mathcal{Z}_y = \{ Z_{y,n}: n \in \mathbb{N}\}$ are the convergent sequences in $\Phi(\mathbb{X}_{Q(\mathbb{D}),N})$ and $N^{'}(\delta)$ is defined in \Cref{lem:cont_rho_z}.
    \begin{lemma} \label{lem:convergent_rho_e}
        Let $\mathcal{Z}_x = \{ Z_{x,n} \in \Phi(\mathbb{X}_{Q(\mathbb{D}),N}): n \in \mathbb{N}\}$ and $\mathcal{Z}_y = \{ Z_{y,n} \in \Phi(\mathbb{X}_{Q(\mathbb{D}),N}): n \in \mathbb{N}\}$ be the convergent sequences in \Cref{lem:cont_right_sequence}.
        For any $\delta >0$, there exists $N^{\prime}_x(\delta), N^{\prime}_y(\delta) \in \mathbb{N}$ such that 
        \begin{align*}
            \forall n > N^{\prime}_x(\delta )&: \ \| \rho \circ Z_{x,n} - \rho_e \circ \Phi(X)  \|_2 < \delta  \\
            \forall n > N^{\prime}_y(\delta )&: \ \| \rho \circ Z_{y,n} - \rho_e \circ \Phi(X^{\prime})  \|_2 < \delta  .
        \end{align*}
        Let $ \| \rho (Z_{x,n}) - \rho (Z_{y,n}) \|_2 < \delta$ for all $n > N^{''}(\delta ) \stackrel{\mathrm{def}}{=} \max \{N^{\prime}_x( \delta ),N^{\prime}_y( \delta  ),N^{'}(\delta ) \}$. Then, we have
        \[
            \forall n >  N^{''}(\delta): \ \|\rho_e \circ \Phi(X) - \rho_e \circ \Phi(X^{\prime}) \|_2 < 3\delta.
        \]
    \end{lemma}
    Combining the results of \Cref{lem:cont_right_sequence,lem:cont_rho_z,lem:convergent_rho_e} we arrive at the following result:
    \[
        \forall X^{\prime}\in \mathbb{X}_{\mathbb{D},N}: \| \Phi(X) - \Phi(X^{\prime}) \|_F < \delta   \rightarrow  \|\rho_e \circ \Phi(X) - \rho_e \circ \Phi(X^{\prime}) \|_2 < 3\delta,
    \]
    that is, $\delta(\varepsilon) = \frac{\varepsilon}{3}$ in equation \eqref{eq:cont_rho_e}, and $\rho_e$ is a continuous function on the compact domain $\Phi(\mathbb{X}_{\mathbb{D},N})$.
    \subsubsection{Proof of \Cref{lem:cont_right_sequence}}   
        Let $X, X^{\prime}\in \mathbb{X}_{\mathbb{D},N} $. Since $\Phi(\mathbb{X}_{Q(\mathbb{D}),N})$ is a dense subset of $\Phi(\mathbb{X}_{\mathbb{D},N})$ (see \Cref{lem:phi_dense}), there exists sequences $\mathcal{Z}_x \stackrel{\mathrm{def}}{=} \{ Z_{x,n} \in \Phi(\mathbb{X}_{Q(\mathbb{D}),N}): n \in \mathbb{N}\}$ and $\mathcal{Z}_y  \stackrel{\mathrm{def}}{=} \{ Z_{y,n} \in \Phi(\mathbb{X}_{Q(\mathbb{D}),N}): n \in \mathbb{N}\}$ such that 
        \[
            \lim_{n \rightarrow \infty} Z_{x,n} = \Phi(X), \ \lim_{n \rightarrow \infty} Z_{y,n} = \Phi(X^{\prime}) \in \Phi(\mathbb{X}_{\mathbb{D},N}),
        \]
        and 
        \[
            \rho_e \circ \Phi(X) = \lim_{n \rightarrow \infty} \rho (Z_{x,n}), \ \rho_e \circ \Phi(X^{\prime})  = \lim_{n \rightarrow \infty} \rho (Z_{y,n}) \in \mathrm{codom}(\rho_e) \subseteq f(\mathbb{X}_{\mathbb{D},N}).
        \]
        That is, there exists $N_x(\delta), N_y(\delta) \in \mathbb{N}$ such that 
        \begin{align*}
            \forall n > N_x(\delta )&: \ \| Z_{x,n} - \Phi(X)  \|_2 < \delta    \\
            \forall n > N_y(\delta )&: \ \| Z_{y,n} - \Phi(X^{\prime})  \|_2 < \delta.
        \end{align*}
        for any $\delta >0$. If $\| \Phi(X) -  \Phi(X^{\prime}) \|_2 < \delta$, then for all $ n > N(\delta)$, we have
        \begin{align*}
            \| Z_{x,n} - Z_{y,n} \|_2 &\stackrel{\text{(a)}}{\leq} \| Z_{x,n} -   \Phi(X)\|_2 + \| Z_{y,n} - \Phi(X^{\prime})\|_2 + \| \Phi(X) -  \Phi(X^{\prime}) \|_2 \\
            &< \delta +\delta +\delta  = 3 \delta ,
        \end{align*}
        where $N(\delta) \stackrel{\mathrm{def}}{=} \max \{N_x(\delta),N_y(\delta) \}$ and $\text{(a)}$ follows from the triangle inequality.
    \subsubsection{Proof of \Cref{lem:cont_rho_z}}    
        The function $\Phi^{-1}$ is continuous on its noncopact domain $\Phi(\mathbb{X}_{Q(\mathbb{D}),N})$; see \Cref{prop:Phi_inv_cont}. Therefore, $\rho = f \circ \Phi^{-1}$ is a continuous function on $\Phi(\mathbb{X}_{Q(\mathbb{D}),N})$. By definition of continuity, for all $Z \in \Phi(\mathbb{X}_{Q(\mathbb{D}),N})$ and $\delta >0$, there exists a $\gamma(\delta ) > 0$ such that
        \begin{equation}\label{eq:rho_is_cont}
            \forall Z^{\prime} \in \Phi(\mathbb{X}_{Q(\mathbb{D}),N}): \| Z - Z^{\prime} \|_F < \gamma(\delta ) \rightarrow \| \rho(Z) - \rho(Z^{\prime}) \|_2 < \delta .
        \end{equation}
        Let $\mathcal{Z}_x = \{ Z_{x,n}: n \in \mathbb{N}\}$ and $\mathcal{Z}_y = \{ Z_{y,n}: n \in \mathbb{N}\}$ be the convergent sequences in \Cref{lem:cont_right_sequence}. For all $\delta >0$, we have
        \[
            \forall X^{\prime}\in \mathbb{X}_{\mathbb{D},N}: \| \Phi(X) - \Phi(X^{\prime}) \|_F < \delta , n > N(\delta ) \rightarrow  \| Z_{x,n} - Z_{y,n} \|_2 < 3\delta.
        \]
        For any $\delta >0$, we let $N^{\prime}(\delta) =  N( \min \{ \frac{\delta}{3} ,  \frac{\gamma(\delta)}{3} \}) $ where $N(\delta) \in \mathbb{N}$ is defined in \Cref{lem:cont_right_sequence}.  By definition, we have
        \[
            \forall n > N^{\prime}(\delta) : \  \| Z_{x,n} - Z_{y,n} \|_2 <  \min\{ \delta ,  \gamma(\delta) \} \leq \gamma(\delta).
        \]
        Since $\rho$ is a continuous map, from equation \eqref{eq:rho_is_cont}, we arrive at the following inequality:
        \[
            \forall n > N^{\prime}(\delta ): \| \rho( Z_{x,n}) - \rho( Z_{y,n}) \|_2 < \delta.
        \]
    \subsubsection{Proof of \Cref{lem:convergent_rho_e}}  
        The sequences $\mathcal{Z}_x = \{ Z_{x,n}: n \in \mathbb{N}\}$ and $\mathcal{Z}_y  = \{ Z_{y,n} : n \in \mathbb{N}\}$ are convergent, that is,
        \[
            \lim_{n \rightarrow \infty} Z_{x,n} = \Phi(X), \ \lim_{n \rightarrow \infty} Z_{y,n} = \Phi(X^{\prime}) \in \Phi(\mathbb{X}_{\mathbb{D},N}).
        \]
        Since we have $\rho_e \circ \Phi(X) = \lim_{n \rightarrow \infty} \rho (Z_{x,b}), \  \rho_e \circ \Phi(X^{\prime}) = \lim_{n \rightarrow \infty} \rho (Z_{y,b})$, there exists $N^{\prime}_x(\delta), N^{\prime}_y(\delta) \in \mathbb{N}$ such that 
        \begin{align*}
            \forall n > N^{\prime}_x(\delta )&: \ \| \rho \circ Z_{x,n} - \rho_e \circ \Phi(X)  \|_2 < \delta  \\
            \forall n > N^{\prime}_y(\delta )&: \ \| \rho \circ Z_{y,n} - \rho_e \circ \Phi(X^{\prime})  \|_2 < \delta  .
        \end{align*}
        If $ \| \rho (Z_{x,n}) - \rho (Z_{y,n}) \|_2 < \delta$ for all $n > N^{''}(\delta) \stackrel{\mathrm{def}}{=} \max \{N^{\prime}_x( \delta ),N^{\prime}_y( \delta  ),N^{'}(\delta) \}$, then, from the triangle inequality and \Cref{lem:cont_rho_z}, we have
        \begin{align*}
            \|\rho_e \circ \Phi(X) - \rho_e \circ \Phi(X^{\prime}) \|_2 &\leq \| \rho (Z_{x,n}) - \rho (Z_{y,n}) \|_2 + \|\rho (Z_{x,n}) -  \rho_e \circ \Phi(X)\|_2 + \| \rho (Z_{y,n}) - \rho_e \circ \Phi(X^{\prime})\|_2 \\
            &<  \delta + \delta+ \delta = 3\delta.
        \end{align*}

\newpage
\section{Proof of \Cref{prop:hypersets}}
    For $K, N \in \mathbb{N}$, let $T,T^{\prime} \in \mathbb{T}^{l}_{N,K}$  be such that $S(T) = S(T^{\prime})$, that is,
    \[
         \{ (e_{n_1}^{\top}l(T), \alpha^1_{n_1} (T)) : n_1 \in [N] \} =  \{ (e_{n_1}^{\top}l(T^{\prime}), \alpha^{1}_{n_1}(T^{\prime})) : n_1 \in [N] \},
    \]
    where $e_n$ is the $n$-th standard basis vector of $\mathbb{R}^N$, for $n \in [N]$. By definition of $\ell$-identifiable tensors, all elements of $\{ \{e_{n}^{\top}l(T): n \in [N] \}\}$ are unique. Therefore, we we have
        \begin{align*}
            \forall n_1 \in [N]: \ e_{n_1}^{\top}l(T) = e_{\pi(n_1)}^{\top}l(T^{\prime}), \ \mbox{and} \ \alpha^{1}_{n_1}(T) = \alpha^{1}_{\pi(n_1)}(T^{\prime}),
        \end{align*}
        for a unique permutation operator $\pi: [N] \rightarrow [N]$. 
    \begin{lemma} \label{lem:S_t_to_sets}
        For all $k \in [K]$, we have
        \[
            \forall n_1,n_2, \ldots, n_k \in [N] : \alpha^{k}_{n_1,n_2, \ldots, n_k}(T) = \alpha^{k}_{\pi(n_1), \pi(n_2), \ldots, \pi(n_k)}(T^{\prime}),
        \]
        where $\pi: [N] \rightarrow [N]$ is a unique permutation operator.
    \end{lemma} 
    \begin{proof}
        The claim holds for $k = 1$. We prove this statement by induction. Let us assume the claim is true for $k \in [K-1]$. We want to show that it also holds for $k+1$, that is,
        \[
            \forall n_1,n_2, \ldots, n_{k+1} \in [N] : \alpha^{k+1}_{n_1,n_2, \ldots, n_{k+1}}(T) = \alpha^{k+1}_{\pi(n_1), \pi(n_2), \ldots, \pi(n_{k+1})}(T^{\prime}).
        \]
        From the definition of $\alpha^{k}$, we have
        \[
            \forall n_1, \ldots, n_{k+1} \in [N]: \ e_{n_{k+1}}^{\top}l(T) = e_{\pi(n_{k+1})}^{\top}l(T^{\prime}), \ \mbox{and} \ \alpha^{k+1}_{n_1,\ldots, n_{k+1}}(T) = \alpha^{k+1}_{\pi(n_1), \ldots, \pi(n_{k+1})}(T^{\prime})
        \]
        --- which follows from the fact that elements of $\{ \{e_{n}^{\top}l(T): n \in [N] \}\}$ are unique. This concludes the proof.
    \end{proof}
    From \Cref{lem:S_t_to_sets}, we have
    \[
        \forall n_1, \ldots, n_{K} \in [N]: \alpha^{k+1}_{n_1,\ldots, n_{K}}(T) = \alpha^{K}_{\pi(n_1), \ldots, \pi(n_{K})}(T^{\prime}),
    \]
    that is, $T = \pi(T^{\prime})$.

    Now let $T, T^{\prime}  \in \mathbb{T}^{l}_{N,K}$ be such that $T = \pi(T^{\prime})$ where $\pi: [N] \rightarrow [N]$ is a permutation operator. By definition, we have
    \[
        \forall n_1, \ldots, n_K \in [N]:  \alpha^{K}_{n_1, \ldots, n_K}(T) = \alpha^{K}_{\pi(n_1), \ldots, \pi(n_K)}(T^{\prime}),
    \]
    and
    \[
        \forall n \in [N]:  e_{n}^{\top}l(T) = e_n^{\top}l( \pi(T^{\prime})) = e_{\pi(n)}^{\top}l(T^{\prime}).
    \]
    For all $n_1,\ldots, n_{K-1} \in [N]$, we have  
    \begin{align*}
        \alpha^{K-1}_{n_1, \ldots, n_{K-1}}(T) &= \{ (e_{n_K}^{\top}l(T), \alpha^{K}_{n_1, \ldots, n_K}(T) ): n_K \in [N] \} \\
        &= \{ (e_{\pi(n_K)}^{\top}l(T^{\prime}), \alpha^{K}_{\pi(n_1), \ldots, \pi(n_K)}(T^{\prime}) ): n_K \in [N] \} \\
        &= \{ (l_{n_K}(T^{\prime}), \alpha^{K}_{\pi(n_1), \ldots, \pi(n_{K-1}),n_K}(T^{\prime}) ): n_K \in [N] \}  \\
        &= \alpha^{K-1}_{ \pi(n_1), \ldots, \pi(n_{K-1})}(T^{\prime})
    \end{align*}
    Using a simple argument by induction, we arrive at the statement in \Cref{lem:S_t_to_sets}. Therefore, we have
    \begin{align*}
        S(T) &= \{ (e_{n_1}^{\top}l(T), \alpha^{1}_{n_1} ): n_1 \in [N] \} = \{ (e_{\pi(n_1)}^{\top}l(T^{\prime}), \alpha^{1}_{\pi(n_1)}(T^{\prime}) ): n_1 \in [N] \} \\
        &= \{ (e_{n_1}^{\top}l(T^{\prime}), \alpha^{1}_{n_1}(T^{\prime}) ): n_1 \in [N] \}  \\
        &= S(T^{\prime})
    \end{align*}
\newpage
\section{Proof of \Cref{thm:tensor}}
\begin{definition}\label{def:phi_to_Phi}
    Let $K,N \in \mathbb{N}$. For all $k \in [K]$, let $\mathbb{D}_k$ be a domain and $\phi_k: \mathbb{D}_k \rightarrow \mathrm{codom}(\phi_k)$, we define the following  multiset function
    \[
        \Phi_k \big( \{ \{ x_n \in  \mathbb{D}_k: n \in [N] \} \}) = \sum_{n \in [N]} \phi_k(x_n),
    \]
    and $\mathrm{codom}(\Phi_k) = \{ \sum_{n \in [N]} \phi_k(x_n): x_n \in \mathbb{D}_k, \forall n \in [N] \}$.
\end{definition}
Let us first show that the proposed sum-decomposable model is injective on $ \mathbb{T}^{l}_{N,K}$. Let $K,N \in \mathbb{N}$ and $T, T^{\prime} \in \mathbb{T}^{l}_{N,K} $ where 
\[
    \sum_{n_1 \in [N]} \phi_{1} (e_{n_1}^{\top}l(T), \beta^1_{n_1}(T) ) = \sum_{n_1 \in [N]} \phi_{1} (e_{n_1}^{\top}l(T^{\prime}), \beta^1_{n_1}(T^{\prime}) ),
\]
that is,
\[
    \Phi_{1} ( \{\{ (e_{n_1}^{\top}l(T), \beta^1_{n_1}(T) ): n_1 \in [N] \}\} ) = \Phi_{1} ( \{\{ (e_{n_1}^{\top}l(T^{\prime}), \beta^1_{n_1}(T^{\prime}) ): n_1 \in [N] \}\} )
\]
Let us {\bf assume} that $\phi_1$ is such that the corresponding $\Phi_{1}$ is an injective multiset function (see \Cref{def:phi_to_Phi}) --- we shall discuss its sufficient condition later in the proof. Since $\{ \{ e_{n}^{\top}l(T) : n \in [N] \} \}$ has all distinct elements for all $T \in \mathbb{T}^{l}_{N,K} $, we have $e_{n_1}^{\top}(T) = e_{\pi(n_1)}^{\top}l(T^{\prime})$ for a unique permutation operator $\pi: [N] \rightarrow [N]$ and 
for all $n_1 \in [N]$. Therefore, we have
\[
    \forall n_1 \in [N]: \beta^1_{n_1}(T) = \beta^1_{\pi(n_1)}(T^{\prime})
\]
\begin{lemma}\label{lem:phi_beta_}
    Let $k \in [K]$ and {\bf assume} $\{ \Phi_k : k \in [K] \}$ are injective multiset functions over their domains, that is,
    \[
        \forall k \in [K]: \phi_k: \mathbb{D}_k \rightarrow \mathrm{codom}(\phi_k)
    \] 
    where $\mathbb{D}_k = \mathrm{codom}(l) \times \mathrm{codom}(\Phi_{k+1})$ and $\mathbb{D}_K = \mathrm{codom}(l) \times \mathbb{R}^{D}$. Then, for all $n_1, \ldots, n_k \in [N]$, we have $\beta^k_{n_1,\ldots, n_k}(T) = \beta^k_{\pi(n_1),\ldots, \pi(n_k)}(T^{\prime})$ where $\pi: [N] \rightarrow [N]$ is a unique permutation operator and $k \in [K]$.
\end{lemma}
\begin{proof}
     The claim holds for $k = 1$. We prove this statement by induction. Let us assume this claim is true for $k \in [K-1]$. We want to show that it also holds for $k+1$, that is,
        \[
            \forall n_1,n_2, \ldots, n_{k+1} \in [N] : \beta^{k+1}_{n_1,n_2, \ldots, n_{k+1}}(T) = \beta^{k+1}_{\pi(n_1), \pi(n_2), \ldots, \pi(n_{k+1})}(T^{\prime}).
        \]
        From the definition of $\beta^{k}$, we have
        \[
            \forall n_1,n_2, \ldots, n_{k} \in [N]: \sum_{n_{k+1} \in [N]} \phi_{k+1}(e_{n_{k+1}}^{\top}l(T), \beta^{k+1}_{n_1 \ldots n_{k+1}}(T) ) =  \sum_{n_{k+1} \in [N]} \phi_{k+1}(e_{n_{k+1}}^{\top}l(T^{\prime}), \beta^{k+1}_{n_1 \ldots n_{k+1}}(T^{\prime}) ),
        \]
        that is,
        \[
            \Phi_{k+1}( \{\{ (e_{n_{k+1}}^{\top}l(T), \beta^{k+1}_{n_1 \ldots n_{k+1}}(T) ): n_{k+1} \in [N] \} \}) =   \Phi_{k+1}( \{\{ (e_{n_{k+1}}^{\top}l(T^{\prime}), \beta^{k+1}_{n_1 \ldots  n_{k+1}}(T^{\prime}) ): n_{k+1} \in [N] \} \}).
        \]
        for all $n_1,n_2, \ldots, n_{k} \in [N]$. Since $\Phi_k$ is an injective multiset function, we have
        \[
            \forall n_1, \ldots, n_{k+1} \in [N]: \ e_{n_{k+1}}^{\top}l(T) = e_{\pi(n_{k+1})}^{\top}l(T^{\prime}), \ \mbox{and} \ \beta^{k+1}_{n_1,\ldots, n_{k+1}}(T) = \beta^{k+1}_{\pi(n_1), \ldots, \pi(n_{k+1})}(T^{\prime})
        \]
        --- which follows from the fact that elements of $\{ \{e_{n}^{\top}l(T): n \in [N] \}\}$ are unique. This concludes the proof.
\end{proof}
From \Cref{lem:phi_beta_}, we arrive at 
\[  
    \forall n_1, \ldots, n_K \in [N]: \beta^{K}_{n_1,n_2, \ldots, n_{K}}(T) = \beta^{K}_{\pi(n_1), \pi(n_2), \ldots, \pi(n_{K})}(T^{\prime}),
\]
for a unique permutation operator $\pi: [N] \rightarrow [N]$, that is, $T = \pi(T^{\prime})$ and $S(T) = S(T^{\prime})$; see \Cref{prop:hypersets}.

Using induction, one can easily verify that given $S(T)$, we can compute  $\sum_{n_1 \in [N]} \phi_{1} (e_{n_1}^{\top}l(T), \beta^1_{n_1}(T) )$. Therefore, the following function is well-defined and injective:
\[
    \forall T \in \mathbb{T}^{l}_{N,K}: m \circ S(T) = \sum_{n_1 \in [N]} \phi_{1} (l_{n_1}(T), \beta^1_{n_1}(T) ),
\]
that is, if $ m \circ S(T) =  m \circ S(T^{\prime})$ then we have $S(T) = S(T^{\prime})$ where $T,T^{\prime} \in \mathbb{T}^{l}_{N,K}$.

Now we define the function $f_s: S(\mathbb{T}^{l}_{N,K}) \rightarrow \mathrm{codom}(f)$ as follows:
\[
    \forall T \in \mathbb{T}^{l}_{N,K}: f_s \circ S(T) \stackrel{\mathrm{def}}{=} f(T).
\]
Since $f$ is a permutation-invariant, the function $f_s$ is well-defined, that is, $f_s \circ S(T) = f(T) = f(\pi(T)) = f_s \circ S(\pi(T)) = f_s \circ S(T) $  for any permutation operator $\pi: [N] \rightarrow [N]$.  Since $m$ is an injective function over its domain, it is invertible on it. Now we define the following function:
\[
    \forall u \in  m \circ S(\mathbb{T}^{l}_{N,K})  : \rho (u) \stackrel{\mathrm{def}}{=} f_s \circ m^{-1}(u).
\]
For any $u \in  m \circ S(\mathbb{T}^{l}_{N,K})$, we have $u = \sum_{n_1 \in [N]} \phi_{1} (e_{n_1}^{\top}l(T), \beta^1_{n_1}(T) )$ where $T \in \mathbb{T}^{l}_{N,K}$, that is, 
\[
     \forall T \in \mathbb{T}^{l}_{N,K}: \rho(u) = \rho \big( \sum_{n_1 \in [N]} \phi_{1} (e_{n_1}^{\top}l(T), \beta^1_{n_1}(T) ) \big) = f_s \circ m^{-1} \circ m \circ S(T) = f(T).
\]

{\bf  Sufficient conditions for injective multiset functions $\{ \Phi_k : k \in [K] \}$.}

(1) If $l(T) \in \mathbb{R}^{N \times M}$, then we use the result in \Cref{thm:multi_decomposition} to ensure the injectivity of $\Phi_k$, for all $k \in [K]$. The function $\phi_k$ is defined on domain $\mathbb{D}_k = \mathrm{codom}(l) \times \mathrm{codom}(\Phi_{k+1})$ where $\mathbb{D}_K = \mathrm{codom}(l) \times \mathbb{R}^{D}$, $\mathrm{codom}(l) \subset \mathbb{R}^M$, and $ \mathrm{codom}(\Phi_{k+1}) \subset \mathbb{R}^{D_{k+1}}$. From \Cref{thm:multi_decomposition}, $D_k = {N+ D_{k+1} \choose N  } -1 $ ensures the injectivity of $\Phi_k$, for all $k \in [K]$.

(2) If $l(T) \in \mathbb{Q}^{N \times M}$, then we use the result in \Cref{thm:perm_w_distinct_labels} to ensure the injectivity of $\Phi_k$, for all $k \in [K]$. This is due the fact that rational-valued vectors are identifiable (see \Cref{prop:rational_identifier}). From \Cref{thm:perm_w_distinct_labels}, $D_k = 2 N(M+ D_{k+1})$ ensures the injectivity of $\Phi_k$, for all $k \in [K]$.

\newpage
\section{Supplementary Discussion} \label{sec:ferey}
As discussed in the main text, the imporant step in showing the existence of sum-decomposable representation is proving that the multiset encoding function $\Phi: \mathrm{dom}(\Phi) \rightarrow \mathrm{codom}(\Phi)$ is an injective map, that is, $\rho = f \circ \Phi^{-1}$ is well-defined over its admissible inputs, that is, $\mathrm{codom}(\Phi)$.

\begin{proposition}[\citealt{fereydounian2022exact}]
    Consider the following continuous map \footnote{This is a trivially altered version of the function in \citep{fereydounian2022exact}.}:
    \[
        \forall x \in \mathbb{R^D}, d_1,d_2 \in [D], n \in [N]: \big( \phi(x) \big)_{d_1,d_2,n} =  \begin{cases}
            \mathrm{Re} \{ (x_{d_1} + x_{d_2} \sqrt{-1})^{n} \} & \ \mbox{if} \ d_2 > d_1\\
            \mathrm{Im} \{ (x_{d_1} + x_{d_2} \sqrt{-1})^{n} \} & \ \mbox{if} \ d_1 > d_2\\
            0 & \ \mbox{otherwise}.
        \end{cases}   
    \]
    The map $\phi: \mathbb{R}^D \rightarrow \mathrm{codom}(\phi) \subset \mathbb{R}^{D \times D \times N}$ defines the following injective multiset function:
    \[
        \forall X \in \mathbb{X}_{\mathbb{R}^D, N} : \Phi(X) = \sum_{x \in X} \phi(x).
    \]
    \label{prop:ferey}
\end{proposition}

We argue that the result in \Cref{prop:ferey} is not valid for all multisets, as the following example suggests.
\begin{example}
     Consider the following distinct sets:
    \[
        X = \{ \{ \begin{bmatrix}
            1 \\ 1 \\ 1
        \end{bmatrix}, \begin{bmatrix}
            3 \\ 2 \\ 1
        \end{bmatrix}, \begin{bmatrix}
            1 \\ 2 \\ 2
        \end{bmatrix}, \begin{bmatrix}
            3 \\ 1 \\ 2
        \end{bmatrix} \} \}, \  X^{\prime}= \{ \{ \begin{bmatrix}
            1 \\ 2 \\ 1
        \end{bmatrix}, \begin{bmatrix}
            3 \\ 1 \\ 1
        \end{bmatrix}, \begin{bmatrix}
            3 \\ 2 \\ 2
        \end{bmatrix}, \begin{bmatrix}
            1 \\ 1 \\ 2
        \end{bmatrix} \} \}.
    \]
    One can be readily verify that $\Phi(X) = \Phi(X^{\prime})$, where $\Phi$ is defined in \Cref{prop:ferey}. The main insight behind this example is the fact that
    $\{ \{ (e_d^{\top} x_n , e_{d^{\prime}}^{\top} x_n): n \in [N] \} \}  = \{ \{ (e_d^{\top} x_n , e_{d^{\prime}}^{\top} x^{\prime}_n): n \in [N] \} \} $ for all distinct $d, d^{\prime} \in [D]$, where $X = \{ \{ x_n: n \in [N] \} \}$, $X^{\prime} = \{ \{ x^{\prime}_n: n \in [N] \} \}$, $D = 3$, and $N = 4$. This later equality does indeed show $X = X^{\prime}$ {\bf if} both multisets contains distinct vectors with {\bf distinct} elements, namely, sets with distinct vectors.
\end{example}
The key elements in proving \Cref{thm:multi_decomposition} is to construct an injective  $\Phi$ which guarantees the existence of  $\rho = f \circ \Phi^{-1}$. Even assuming input multisets contain vectors with distinct elements, the above result doe not guarantee the continuity of $\rho$. Furthermore, one can easily show that $\mathrm{codom}(\Phi)$ is not a compact set. To show this, note that the domain of $\Phi$ does not include a single point $X$ in the example above. Now, one can construct a sequence of (multi)ests $X_n$ where $\lim_{n \rightarrow \infty} X_n = X$ such that, for all $n \in \mathbb{N}$, all elements of $X_n$ are distinct and have distinct values, that is, $X_n \in \mathrm{dom}(\Phi)$. Since $\Phi$ is a continuous map, $\{ \Phi(X_n): n \in \mathbb{N} \}$ is a Cauchy sequence in $\mathrm{codom}(\Phi)$ whose limit does not belong to $\mathrm{codom}(\Phi)$, that is, the co domain of $\Phi$ is not compact. 

\bibliography{references}

\begin{thebibliography}{68}
\providecommand{\natexlab}[1]{#1}
\providecommand{\url}[1]{\texttt{#1}}
\expandafter\ifx\csname urlstyle\endcsname\relax
  \providecommand{\doi}[1]{doi: #1}\else
  \providecommand{\doi}{doi: \begingroup \urlstyle{rm}\Url}\fi

\bibitem[Amir et~al.(2023)Amir, Gortler, Avni, Ravina, and Dym]{amir2023neural}
Tal Amir, Steven~J Gortler, Ilai Avni, Ravina Ravina, and Nadav Dym.
\newblock Neural injective functions for multisets, measures and graphs via a
  finite witness theorem.
\newblock \emph{arXiv preprint arXiv:2306.06529}, 2023.

\bibitem[Attenborough(2003)]{attenborough2003mathematics}
Mary~P Attenborough.
\newblock \emph{Mathematics for electrical engineering and computing}.
\newblock Elsevier, 2003.

\bibitem[Azizian and Lelarge(2020)]{azizian2020expressive}
Waiss Azizian and Marc Lelarge.
\newblock Expressive power of invariant and equivariant graph neural networks.
\newblock \emph{arXiv preprint arXiv:2006.15646}, 2020.

\bibitem[Beaini et~al.(2021)Beaini, Passaro, L{\'e}tourneau, Hamilton, Corso,
  and Li{\`o}]{beaini2021directional}
Dominique Beaini, Saro Passaro, Vincent L{\'e}tourneau, Will Hamilton, Gabriele
  Corso, and Pietro Li{\`o}.
\newblock Directional graph networks.
\newblock In \emph{International Conference on Machine Learning}, pages
  748--758. PMLR, 2021.

\bibitem[Bevilacqua et~al.(2021)Bevilacqua, Frasca, Lim, Srinivasan, Cai,
  Balamurugan, Bronstein, and Maron]{bevilacqua2021equivariant}
Beatrice Bevilacqua, Fabrizio Frasca, Derek Lim, Balasubramaniam Srinivasan,
  Chen Cai, Gopinath Balamurugan, Michael~M Bronstein, and Haggai Maron.
\newblock Equivariant subgraph aggregation networks.
\newblock \emph{arXiv preprint arXiv:2110.02910}, 2021.

\bibitem[Bro et~al.(2008)Bro, Acar, and Kolda]{bro2008resolving}
Rasmus Bro, Evrim Acar, and Tamara~G Kolda.
\newblock Resolving the sign ambiguity in the singular value decomposition.
\newblock \emph{Journal of Chemometrics: A Journal of the Chemometrics
  Society}, 22\penalty0 (2):\penalty0 135--140, 2008.

\bibitem[Chen et~al.(2019)Chen, Wu, and Zaki]{chen2019reinforcement}
Yu~Chen, Lingfei Wu, and Mohammed~J Zaki.
\newblock Reinforcement learning based graph-to-sequence model for natural
  question generation.
\newblock \emph{arXiv preprint arXiv:1908.04942}, 2019.

\bibitem[{\'C}urgus and Mascioni(2006)]{curgus2006roots}
Branko {\'C}urgus and Vania Mascioni.
\newblock Roots and polynomials as homeomorphic spaces.
\newblock \emph{Expositiones Mathematicae}, 24\penalty0 (1):\penalty0 81--95,
  2006.

\bibitem[Cvetkovic et~al.(1997)Cvetkovic, Cvetkovi{\'c}, Rowlinson, and
  Simic]{cvetkovic1997eigenspaces}
Dragos Cvetkovic, Drago{\v{s}}~M Cvetkovi{\'c}, Peter Rowlinson, and Slobodan
  Simic.
\newblock \emph{Eigenspaces of graphs}.
\newblock Number~66. Cambridge University Press, 1997.

\bibitem[Deimling(2010)]{deimling2010nonlinear}
Klaus Deimling.
\newblock \emph{Nonlinear functional analysis}.
\newblock Courier Corporation, 2010.

\bibitem[Dwivedi and Bresson(2020)]{dwivedi2020generalization}
Vijay~Prakash Dwivedi and Xavier Bresson.
\newblock A generalization of transformer networks to graphs.
\newblock \emph{arXiv preprint arXiv:2012.09699}, 2020.

\bibitem[Dwivedi et~al.(2020)Dwivedi, Joshi, Laurent, Bengio, and
  Bresson]{dwivedi2020benchmarking}
Vijay~Prakash Dwivedi, Chaitanya~K Joshi, Thomas Laurent, Yoshua Bengio, and
  Xavier Bresson.
\newblock Benchmarking graph neural networks.
\newblock 2020.

\bibitem[Dwivedi et~al.(2021)Dwivedi, Luu, Laurent, Bengio, and
  Bresson]{dwivedi2021graph}
Vijay~Prakash Dwivedi, Anh~Tuan Luu, Thomas Laurent, Yoshua Bengio, and Xavier
  Bresson.
\newblock Graph neural networks with learnable structural and positional
  representations.
\newblock \emph{arXiv preprint arXiv:2110.07875}, 2021.

\bibitem[Dym and Gortler(2022)]{dym2022low}
Nadav Dym and Steven~J Gortler.
\newblock Low dimensional invariant embeddings for universal geometric
  learning.
\newblock \emph{arXiv preprint arXiv:2205.02956}, 2022.

\bibitem[Eastment and Krzanowski(1982)]{eastment1982cross}
HT~Eastment and WJ~Krzanowski.
\newblock Cross-validatory choice of the number of components from a principal
  component analysis.
\newblock \emph{Technometrics}, 24\penalty0 (1):\penalty0 73--77, 1982.

\bibitem[Engelking(1989)]{engelking1989general}
Ryszard Engelking.
\newblock General topology.
\newblock \emph{Sigma series in pure mathematics}, 6, 1989.

\bibitem[Eslami et~al.(2016)Eslami, Heess, Weber, Tassa, Szepesvari, Hinton,
  et~al.]{eslami2016attend}
SM~Eslami, Nicolas Heess, Theophane Weber, Yuval Tassa, David Szepesvari,
  Geoffrey~E Hinton, et~al.
\newblock Attend, infer, repeat: Fast scene understanding with generative
  models.
\newblock \emph{Advances in neural information processing systems}, 29, 2016.

\bibitem[Feller(1967)]{feller1967introduction}
William Feller.
\newblock An introduction to probability theory and its applications.
\newblock Technical report, Wiley series in probability and mathematical
  statistics, 3rd edn.(Wiley, New~…, 1967.

\bibitem[Fereydounian et~al.(2022)Fereydounian, Hassani, Dadashkarimi, and
  Karbasi]{fereydounian2022exact}
Mohammad Fereydounian, Hamed Hassani, Javid Dadashkarimi, and Amin Karbasi.
\newblock The exact class of graph functions generated by graph neural
  networks.
\newblock \emph{arXiv preprint arXiv:2202.08833}, 2022.

\bibitem[Garg et~al.(2020)Garg, Jegelka, and Jaakkola]{garg2020generalization}
Vikas Garg, Stefanie Jegelka, and Tommi Jaakkola.
\newblock Generalization and representational limits of graph neural networks.
\newblock In \emph{International Conference on Machine Learning}, pages
  3419--3430. PMLR, 2020.

\bibitem[Hordan et~al.(2023)Hordan, Amir, Gortler, and Dym]{hordan2023complete}
Snir Hordan, Tal Amir, Steven~J Gortler, and Nadav Dym.
\newblock Complete neural networks for euclidean graphs.
\newblock \emph{arXiv preprint arXiv:2301.13821}, 2023.

\bibitem[Joshi et~al.(2023)Joshi, Bodnar, Mathis, Cohen, and
  Lio]{joshi2023expressive}
Chaitanya~K Joshi, Cristian Bodnar, Simon~V Mathis, Taco Cohen, and Pietro Lio.
\newblock On the expressive power of geometric graph neural networks.
\newblock \emph{arXiv preprint arXiv:2301.09308}, 2023.

\bibitem[Keriven and Peyr{\'e}(2019)]{keriven2019universal}
Nicolas Keriven and Gabriel Peyr{\'e}.
\newblock Universal invariant and equivariant graph neural networks.
\newblock \emph{Advances in Neural Information Processing Systems}, 32, 2019.

\bibitem[Kim et~al.(2022)Kim, Nguyen, Min, Cho, Lee, Lee, and
  Hong]{kim2022pure}
Jinwoo Kim, Dat Nguyen, Seonwoo Min, Sungjun Cho, Moontae Lee, Honglak Lee, and
  Seunghoon Hong.
\newblock Pure transformers are powerful graph learners.
\newblock \emph{Advances in Neural Information Processing Systems},
  35:\penalty0 14582--14595, 2022.

\bibitem[Kosiorek et~al.(2018)Kosiorek, Kim, Teh, and
  Posner]{kosiorek2018sequential}
Adam Kosiorek, Hyunjik Kim, Yee~Whye Teh, and Ingmar Posner.
\newblock Sequential attend, infer, repeat: Generative modelling of moving
  objects.
\newblock \emph{Advances in Neural Information Processing Systems}, 31, 2018.

\bibitem[Kreuzer et~al.(2021)Kreuzer, Beaini, Hamilton, L{\'e}tourneau, and
  Tossou]{kreuzer2021rethinking}
Devin Kreuzer, Dominique Beaini, Will Hamilton, Vincent L{\'e}tourneau, and
  Prudencio Tossou.
\newblock Rethinking graph transformers with spectral attention.
\newblock \emph{Advances in Neural Information Processing Systems},
  34:\penalty0 21618--21629, 2021.

\bibitem[Lee et~al.(2019)Lee, Lee, Kim, Kosiorek, Choi, and Teh]{lee2019set}
Juho Lee, Yoonho Lee, Jungtaek Kim, Adam Kosiorek, Seungjin Choi, and Yee~Whye
  Teh.
\newblock Set transformer: A framework for attention-based
  permutation-invariant neural networks.
\newblock In \emph{International conference on machine learning}, pages
  3744--3753. PMLR, 2019.

\bibitem[Li et~al.(2020)Li, Lin, Madhusudan, Sharma, Xu, Sapatnekar, Harjani,
  and Hu]{li2020customized}
Yaguang Li, Yishuang Lin, Meghna Madhusudan, Arvind Sharma, Wenbin Xu, Sachin~S
  Sapatnekar, Ramesh Harjani, and Jiang Hu.
\newblock A customized graph neural network model for guiding analog ic
  placement.
\newblock In \emph{Proceedings of the 39th International Conference on
  Computer-Aided Design}, pages 1--9, 2020.

\bibitem[Lim et~al.(2022)Lim, Robinson, Zhao, Smidt, Sra, Maron, and
  Jegelka]{lim2022sign}
Derek Lim, Joshua Robinson, Lingxiao Zhao, Tess Smidt, Suvrit Sra, Haggai
  Maron, and Stefanie Jegelka.
\newblock Sign and basis invariant networks for spectral graph representation
  learning.
\newblock \emph{arXiv preprint arXiv:2202.13013}, 2022.

\bibitem[Lu et~al.(2020)Lu, Pentapati, Zhu, Samadi, and Lim]{lu2020tp}
Yi-Chen Lu, Sai Surya~Kiran Pentapati, Lingjun Zhu, Kambiz Samadi, and Sung~Kyu
  Lim.
\newblock Tp-gnn: A graph neural network framework for tier partitioning in
  monolithic 3d ics.
\newblock In \emph{2020 57th ACM/IEEE Design Automation Conference (DAC)},
  pages 1--6. IEEE, 2020.

\bibitem[Ma et~al.(2018)Ma, Chen, and Xiao]{ma2018constrained}
Tengfei Ma, Jie Chen, and Cao Xiao.
\newblock Constrained generation of semantically valid graphs via regularizing
  variational autoencoders.
\newblock \emph{Advances in Neural Information Processing Systems}, 31, 2018.

\bibitem[Maron et~al.(2018)Maron, Ben-Hamu, Shamir, and
  Lipman]{maron2018invariant}
Haggai Maron, Heli Ben-Hamu, Nadav Shamir, and Yaron Lipman.
\newblock Invariant and equivariant graph networks.
\newblock \emph{arXiv preprint arXiv:1812.09902}, 2018.

\bibitem[Maron et~al.(2019{\natexlab{a}})Maron, Ben-Hamu, Serviansky, and
  Lipman]{maron2019provably}
Haggai Maron, Heli Ben-Hamu, Hadar Serviansky, and Yaron Lipman.
\newblock Provably powerful graph networks.
\newblock \emph{Advances in neural information processing systems}, 32,
  2019{\natexlab{a}}.

\bibitem[Maron et~al.(2019{\natexlab{b}})Maron, Fetaya, Segol, and
  Lipman]{maron2019universality}
Haggai Maron, Ethan Fetaya, Nimrod Segol, and Yaron Lipman.
\newblock On the universality of invariant networks.
\newblock In \emph{International conference on machine learning}, pages
  4363--4371. PMLR, 2019{\natexlab{b}}.

\bibitem[Mialon et~al.(2021)Mialon, Chen, Selosse, and
  Mairal]{mialon2021graphit}
Gr{\'e}goire Mialon, Dexiong Chen, Margot Selosse, and Julien Mairal.
\newblock Graphit: Encoding graph structure in transformers.
\newblock \emph{arXiv preprint arXiv:2106.05667}, 2021.

\bibitem[Morris et~al.(2019)Morris, Ritzert, Fey, Hamilton, Lenssen, Rattan,
  and Grohe]{morris2019weisfeiler}
Christopher Morris, Martin Ritzert, Matthias Fey, William~L Hamilton, Jan~Eric
  Lenssen, Gaurav Rattan, and Martin Grohe.
\newblock Weisfeiler and leman go neural: Higher-order graph neural networks.
\newblock In \emph{Proceedings of the AAAI conference on artificial
  intelligence}, volume~33, pages 4602--4609, 2019.

\bibitem[Muandet et~al.(2012)Muandet, Fukumizu, Dinuzzo, and
  Sch{\"o}lkopf]{muandet2012learning}
Krikamol Muandet, Kenji Fukumizu, Francesco Dinuzzo, and Bernhard
  Sch{\"o}lkopf.
\newblock Learning from distributions via support measure machines.
\newblock \emph{Advances in neural information processing systems}, 25, 2012.

\bibitem[Muandet et~al.(2013)Muandet, Balduzzi, and
  Sch{\"o}lkopf]{muandet2013domain}
Krikamol Muandet, David Balduzzi, and Bernhard Sch{\"o}lkopf.
\newblock Domain generalization via invariant feature representation.
\newblock In \emph{International conference on machine learning}, pages 10--18.
  PMLR, 2013.

\bibitem[Murphy et~al.(2018)Murphy, Srinivasan, Rao, and
  Ribeiro]{murphy2018janossy}
Ryan~L Murphy, Balasubramaniam Srinivasan, Vinayak Rao, and Bruno Ribeiro.
\newblock Janossy pooling: Learning deep permutation-invariant functions for
  variable-size inputs.
\newblock \emph{arXiv preprint arXiv:1811.01900}, 2018.

\bibitem[Ntampaka et~al.(2016)Ntampaka, Trac, Sutherland, Fromenteau,
  P{\'o}czos, and Schneider]{ntampaka2016dynamical}
Michelle Ntampaka, Hy~Trac, Dougal~J Sutherland, Sebastian Fromenteau,
  Barnab{\'a}s P{\'o}czos, and Jeff Schneider.
\newblock Dynamical mass measurements of contaminated galaxy clusters using
  machine learning.
\newblock \emph{The Astrophysical Journal}, 831\penalty0 (2):\penalty0 135,
  2016.

\bibitem[Oliva et~al.(2013)Oliva, P{\'o}czos, and
  Schneider]{oliva2013distribution}
Junier Oliva, Barnab{\'a}s P{\'o}czos, and Jeff Schneider.
\newblock Distribution to distribution regression.
\newblock In \emph{International Conference on Machine Learning}, pages
  1049--1057. PMLR, 2013.

\bibitem[Ovsjanikov et~al.(2008)Ovsjanikov, Sun, and
  Guibas]{ovsjanikov2008global}
Maks Ovsjanikov, Jian Sun, and Leonidas Guibas.
\newblock Global intrinsic symmetries of shapes.
\newblock In \emph{Computer graphics forum}, volume~27, pages 1341--1348. Wiley
  Online Library, 2008.

\bibitem[P{\'o}czos et~al.(2013)P{\'o}czos, Singh, Rinaldo, and
  Wasserman]{poczos2013distribution}
Barnab{\'a}s P{\'o}czos, Aarti Singh, Alessandro Rinaldo, and Larry Wasserman.
\newblock Distribution-free distribution regression.
\newblock In \emph{Artificial Intelligence and Statistics}, pages 507--515.
  PMLR, 2013.

\bibitem[Pozdnyakov and Ceriotti(2022)]{pozdnyakov2022incompleteness}
Sergey~N Pozdnyakov and Michele Ceriotti.
\newblock Incompleteness of graph neural networks for points clouds in three
  dimensions.
\newblock \emph{Machine Learning: Science and Technology}, 3\penalty0
  (4):\penalty0 045020, 2022.

\bibitem[Pugh and Pugh(2002)]{pugh2002real}
Charles~Chapman Pugh and CC~Pugh.
\newblock \emph{Real mathematical analysis}, volume 2011.
\newblock Springer, 2002.

\bibitem[Qi et~al.(2017{\natexlab{a}})Qi, Su, Mo, and Guibas]{qi2017pointnet}
Charles~R Qi, Hao Su, Kaichun Mo, and Leonidas~J Guibas.
\newblock Pointnet: Deep learning on point sets for 3d classification and
  segmentation.
\newblock In \emph{Proceedings of the IEEE conference on computer vision and
  pattern recognition}, pages 652--660, 2017{\natexlab{a}}.

\bibitem[Qi et~al.(2017{\natexlab{b}})Qi, Yi, Su, and
  Guibas]{qi2017pointnet_plus}
Charles~Ruizhongtai Qi, Li~Yi, Hao Su, and Leonidas~J Guibas.
\newblock Pointnet++: Deep hierarchical feature learning on point sets in a
  metric space.
\newblock \emph{Advances in neural information processing systems}, 30,
  2017{\natexlab{b}}.

\bibitem[Ravanbakhsh(2020)]{ravanbakhsh2020universal}
Siamak Ravanbakhsh.
\newblock Universal equivariant multilayer perceptrons.
\newblock In \emph{International Conference on Machine Learning}, pages
  7996--8006. PMLR, 2020.

\bibitem[Ravanbakhsh et~al.(2016{\natexlab{a}})Ravanbakhsh, Oliva, Fromenteau,
  Price, Ho, Schneider, and P{\'o}czos]{ravanbakhsh2016estimating}
Siamak Ravanbakhsh, Junier Oliva, Sebastian Fromenteau, Layne Price, Shirley
  Ho, Jeff Schneider, and Barnab{\'a}s P{\'o}czos.
\newblock Estimating cosmological parameters from the dark matter distribution.
\newblock In \emph{International Conference on Machine Learning}, pages
  2407--2416. PMLR, 2016{\natexlab{a}}.

\bibitem[Ravanbakhsh et~al.(2016{\natexlab{b}})Ravanbakhsh, Schneider, and
  Poczos]{ravanbakhsh2016deep}
Siamak Ravanbakhsh, Jeff Schneider, and Barnabas Poczos.
\newblock Deep learning with sets and point clouds.
\newblock \emph{arXiv preprint arXiv:1611.04500}, 2016{\natexlab{b}}.

\bibitem[Rustamov et~al.(2007)]{rustamov2007laplace}
Raif~M Rustamov et~al.
\newblock Laplace-beltrami eigenfunctions for deformation invariant shape
  representation.
\newblock In \emph{Symposium on geometry processing}, volume 257, pages
  225--233, 2007.

\bibitem[Rydh(2007)]{rydh2007minimal}
David Rydh.
\newblock A minimal set of generators for the ring of multisymmetric functions.
\newblock In \emph{Annales de l'institut Fourier}, volume~57, pages 1741--1769,
  2007.

\bibitem[Segol and Lipman(2019)]{segol2019universal}
Nimrod Segol and Yaron Lipman.
\newblock On universal equivariant set networks.
\newblock \emph{arXiv preprint arXiv:1910.02421}, 2019.

\bibitem[S{\'e}roul(2012)]{seroul2012programming}
Raymond S{\'e}roul.
\newblock \emph{Programming for mathematicians}.
\newblock Springer Science \& Business Media, 2012.

\bibitem[Shawe-Taylor(1993)]{shawe1993symmetries}
John Shawe-Taylor.
\newblock Symmetries and discriminability in feedforward network architectures.
\newblock \emph{IEEE Transactions on Neural Networks}, 4\penalty0 (5):\penalty0
  816--826, 1993.

\bibitem[Stein and Shakarchi(2010)]{stein2010complex}
Elias~M Stein and Rami Shakarchi.
\newblock \emph{Complex analysis}, volume~2.
\newblock Princeton University Press, 2010.

\bibitem[Sunehag et~al.(2017)Sunehag, Lever, Gruslys, Czarnecki, Zambaldi,
  Jaderberg, Lanctot, Sonnerat, Leibo, Tuyls, et~al.]{sunehag2017value}
Peter Sunehag, Guy Lever, Audrunas Gruslys, Wojciech~Marian Czarnecki, Vinicius
  Zambaldi, Max Jaderberg, Marc Lanctot, Nicolas Sonnerat, Joel~Z Leibo, Karl
  Tuyls, et~al.
\newblock Value-decomposition networks for cooperative multi-agent learning.
\newblock \emph{arXiv preprint arXiv:1706.05296}, 2017.

\bibitem[Sutherland(2009)]{sutherland2009introduction}
Wilson~A Sutherland.
\newblock \emph{Introduction to metric and topological spaces}.
\newblock Oxford University Press, 2009.

\bibitem[Szab{\'o} et~al.(2016)Szab{\'o}, Sriperumbudur, P{\'o}czos, and
  Gretton]{szabo2016learning}
Zolt{\'a}n Szab{\'o}, Bharath~K Sriperumbudur, Barnab{\'a}s P{\'o}czos, and
  Arthur Gretton.
\newblock Learning theory for distribution regression.
\newblock \emph{The Journal of Machine Learning Research}, 17\penalty0
  (1):\penalty0 5272--5311, 2016.

\bibitem[Von~Luxburg(2007)]{von2007tutorial}
Ulrike Von~Luxburg.
\newblock A tutorial on spectral clustering.
\newblock \emph{Statistics and computing}, 17:\penalty0 395--416, 2007.

\bibitem[Wagstaff et~al.(2019)Wagstaff, Fuchs, Engelcke, Posner, and
  Osborne]{wagstaff2019limitations}
Edward Wagstaff, Fabian Fuchs, Martin Engelcke, Ingmar Posner, and Michael~A
  Osborne.
\newblock On the limitations of representing functions on sets.
\newblock In \emph{International Conference on Machine Learning}, pages
  6487--6494. PMLR, 2019.

\bibitem[Wagstaff et~al.(2022)Wagstaff, Fuchs, Engelcke, Osborne, and
  Posner]{wagstaff2022universal}
Edward Wagstaff, Fabian~B Fuchs, Martin Engelcke, Michael~A Osborne, and Ingmar
  Posner.
\newblock Universal approximation of functions on sets.
\newblock \emph{Journal of Machine Learning Research}, 23\penalty0
  (151):\penalty0 1--56, 2022.

\bibitem[Wang et~al.(2019)Wang, He, Cao, Liu, and Chua]{wang2019kgat}
Xiang Wang, Xiangnan He, Yixin Cao, Meng Liu, and Tat-Seng Chua.
\newblock Kgat: Knowledge graph attention network for recommendation.
\newblock In \emph{Proceedings of the 25th ACM SIGKDD international conference
  on knowledge discovery \& data mining}, pages 950--958, 2019.

\bibitem[Xie et~al.(2021)Xie, Liang, Xu, Hu, Duan, and Chen]{xie2021net2}
Zhiyao Xie, Rongjian Liang, Xiaoqing Xu, Jiang Hu, Yixiao Duan, and Yiran Chen.
\newblock Net2: A graph attention network method customized for pre-placement
  net length estimation.
\newblock In \emph{Proceedings of the 26th Asia and South Pacific Design
  Automation Conference}, pages 671--677, 2021.

\bibitem[Xu et~al.(2018)Xu, Hu, Leskovec, and Jegelka]{xu2018powerful}
Keyulu Xu, Weihua Hu, Jure Leskovec, and Stefanie Jegelka.
\newblock How powerful are graph neural networks?
\newblock \emph{arXiv preprint arXiv:1810.00826}, 2018.

\bibitem[Yang et~al.(2019)Yang, Tang, Zhang, and Cai]{yang2019auto}
Xu~Yang, Kaihua Tang, Hanwang Zhang, and Jianfei Cai.
\newblock Auto-encoding scene graphs for image captioning.
\newblock In \emph{Proceedings of the IEEE/CVF conference on computer vision
  and pattern recognition}, pages 10685--10694, 2019.

\bibitem[Zaheer et~al.(2017)Zaheer, Kottur, Ravanbakhsh, Poczos, Salakhutdinov,
  and Smola]{zaheer2017deep}
Manzil Zaheer, Satwik Kottur, Siamak Ravanbakhsh, Barnabas Poczos, Russ~R
  Salakhutdinov, and Alexander~J Smola.
\newblock Deep sets.
\newblock \emph{Advances in neural information processing systems}, 30, 2017.

\bibitem[Zhu et~al.(2020)Zhu, Liu, Chen, Zhao, and Pan]{zhu2020exploring}
Keren Zhu, Mingjie Liu, Hao Chen, Zheng Zhao, and David~Z Pan.
\newblock Exploring logic optimizations with reinforcement learning and graph
  convolutional network.
\newblock In \emph{Proceedings of the 2020 ACM/IEEE Workshop on Machine
  Learning for CAD}, pages 145--150, 2020.

\end{thebibliography}
\end{document}